\documentclass[]{article}
\usepackage{songmei}
\usepackage[toc,page]{appendix}
\usepackage[mathscr]{euscript}

\DeclareSymbolFont{rsfs}{U}{rsfs}{m}{n}
\DeclareSymbolFontAlphabet{\mathscrsfs}{rsfs}

\newtheoremstyle{myremark} 
    {\topsep}                    
    {\topsep}                    
    {\rm}                        
    {}                           
    {\bf}                        
    {.}                          
    {.5em}                       
    {}  

\theoremstyle{myremark}
\newtheorem{remark}{Remark}[section]

\def\bbara{\bar {\boldsymbol a}} 
\def\bundera{\underline {\boldsymbol a}} 
\def\undera{{\underline a}}
\def\btilda{\Tilde {\boldsymbol a}} 
\def\bara{{\bar a}}
\def\tilda{{\tilde a}}
\def\bba{\bar {a}}
\def\ta{\Tilde {a}}
\def\ua{\underline{a}}

\def\bH{{\boldsymbol H}}

\def\bbw{{\bar \bw}}
\def\ubw{{\underline \bw}}

\def\urho{{\underline \rho}}
\def\trho{{\Tilde \rho}}
\def\urhoN{\underline{\rho}^{(N)}}
\def\trhoN{\Tilde{\rho}^{(N)}}

\def\hrho{\hat{\rho}}
\def\orho{\overline{\rho}}
\def\rad{{\bar \rho}}

\def\oW{\overline{W}}
\def\obW{\overline{\boldsymbol W}}
\def\btheta{{\boldsymbol \theta}}
\def\bbtheta{\bar {\boldsymbol \theta}}
\def\ubtheta{\underline{\boldsymbol \theta}}
\def\tbtheta{\Tilde {\boldsymbol \theta}}

\newcommand{\lp}{\left(}
\newcommand{\rp}{\right)}
\newcommand{\lb}{\left[}
\newcommand{\rb}{\right]}
\newcommand{\na}{\nabla}

\def\Ltwo{{L^2}}

\def\oR{{\overline R}}
\def\reals{{\mathbb R}}
\def\de{{\rm d}}
\def\good{{\rm good}}
\def\speed{{\alpha}}
\def\Ent{{\rm Ent}}
\def\id{{\boldsymbol I}}

\def\sBayes{\mbox{\tiny \rm Bayes}}
\def\sop{\mbox{\tiny \rm op}}
\def\prob{{\mathbb P}}
\def\normal{{\mathsf N}}
\def\rad{{\bar \rho}}
\def\supp{{\rm supp}}

\def\proj{{\boldsymbol P}}

\def\tbx{\tilde{\boldsymbol x}}
\def\cuP{\mathscrsfs{P}}
\def\cuD{\mathscrsfs{D}}
\def\cuG{\mathscrsfs{G}}
\def\bbA{{\bar {\boldsymbol A}}}
\def\hf{\hat{f}}

\def\cF{{\mathcal F}}
\def\cH{{\mathcal H}}
\def\cV{{\mathcal V}}

\def\cL{{\mathcal L}}
\def\cK{\mathcal{K}}

\def\blambda{{\boldsymbol \lambda}}
\def\bxi{\boldsymbol{\xi}}
\def\bomega{{\boldsymbol \omega}}

\def\bfzero{{\boldsymbol 0}}
\def\bzero{{\boldsymbol 0}}
\def\bsigma{{\boldsymbol \sigma}}
\def\bSigma{{\boldsymbol \Sigma}}

\def\bA{{\boldsymbol A}}
\def\ba{{\boldsymbol a}}

\def\bh{{\boldsymbol h}}
\def\bz{{\boldsymbol z}}
\def\by{{\boldsymbol y}}
\def\bU{{\boldsymbol U}}
\def\bG{{\boldsymbol G}}
\def\bF{{\boldsymbol F}}
\def\bw{{\boldsymbol w}}
\def\bW{{\boldsymbol W}}
\def\bz{{\boldsymbol z}}
\def\bZ{{\boldsymbol Z}}
\def\bw{{\boldsymbol w}}
\def\bg{{\boldsymbol g}}
\def\bI{{\boldsymbol I}}

\def\bX{{\boldsymbol X}}
\def\bx{{\boldsymbol x}}
\def\bu{{\boldsymbol u}}
\def\bh{{\boldsymbol h}}
\def\bH{{\boldsymbol H}}

\usepackage{hyperref}
\hypersetup{
    colorlinks,
    linkcolor={blue!80!black},
    citecolor={green!50!black},
}
\colorlet{linkequation}{blue}

\title{Mean-field theory of two-layers neural networks:\\
dimension-free bounds and kernel limit}

\author{Song Mei\thanks{Institute for Computational and Mathematical Engineering, Stanford University},\;\; Theodor Misiakiewicz\thanks{Department of Statistics, Stanford University}, \;\; Andrea Montanari\thanks{Department of Electrical Engineering and Department of Statistics, Stanford University}}

\begin{document}

\maketitle

\begin{abstract}
We consider learning two layer neural networks using stochastic gradient descent. The mean-field description of this learning dynamics 
approximates the evolution of the network weights by an evolution in the space of probability distributions in $\reals^D$ (where $D$
is the number of parameters associated to each neuron). This evolution can be defined through a partial differential equation or, equivalently, as the gradient flow in 
the Wasserstein space of probability distributions. 
Earlier work shows that (under some regularity assumptions), the mean field description is accurate
as soon as the number of hidden units is much larger than the dimension $D$. In this paper we establish stronger and more general approximation guarantees.
First of all, we show that the number of hidden units only needs to be larger than a quantity dependent on the regularity properties of the data, and independent of the 
dimensions. 
Next, we generalize this analysis to the case of unbounded activation functions, which was not covered by earlier bounds. 
We extend our results to noisy stochastic gradient descent. 

Finally, we show that kernel ridge regression can be recovered as a special limit of  the mean field analysis. 
\end{abstract}

\tableofcontents

 \section{Introduction}

Multi-layer neural networks, and in particular multi-layer perceptrons, present a number of remarkable features.
They are effectively trained using stochastic-gradient descent (SGD) \cite{lecun1998gradient}; their behavior is fairly insensitive to the number of hidden units or to 
the input dimensions \cite{srivastava2014dropout}; 
their number of parameters is often larger than the number of samples.

In this paper consider simple neural networks with one layer of $N$ hidden units:
\begin{align}
\hf_N(\bx;\btheta) = \frac{1}{N}\sum_{i=1}^N\sigma_\star(\bx;\btheta_i)\, ,\;\;\;\; \sigma_\star(\bx;\btheta_i) = a_i\sigma(\bx;\bw_i)\, ,\label{eq:NNET}
\end{align}
Here $\bx\in\reals^d$ is a feature vector, $\btheta = (\btheta_1,\dots,\btheta_N)$ comprises the network parameters, $\btheta_i = (a_i,\bw_i)\in\reals^{D}$, and
$\sigma:\reals^d\times\reals^{D-1}\to\reals$ is a bounded activation function. 
The most classical example is $\sigma(\bx;\bw) = \sigma(\<\bw,\bx\>)$, where $\sigma:\reals\to\reals$  is a scalar function (and of course $D=d+1$),
but our theory covers a broader set of examples. 
We assume to be given data $(y_i,\bx_i) \sim \prob$, with $\prob\in\cuP(\reals\times\reals^d)$ a probability distribution over $\reals\times\reals^d$,
and attempt at minimizing the square loss risk: 
\begin{align}
R_N(\btheta) = \E\big\{(y - \hf_N(\bx;\btheta))^2\}. 
\end{align}
The risk function $R_N$ can be either understood as population risk or empirical risk, depending on viewing $\prob$ as a population distribution or assuming $\prob = n^{-1} \sum_{k=1}^n \delta_{(y_k, \bx_k)}$ is supported on $n$ data points. If $R_N$ is understood as the population risk, we can rewrite 
\begin{align}
R_N(\btheta) = R_{\sBayes} + \E\big\{(f(\bx) - \hf_N(\bx;\btheta))^2\}  \,,
\end{align}
where $f(\bx) = \E\{y|\bx\}$ and $R_{\sBayes}$ is the Bayes error. 

Classical theory of universal approximation provides useful insights into the way two-layers networks capture
arbitrary input-output relations \cite{cybenko1989approximation,barron1993universal}. In particular, Barron's theorem \cite{barron1993universal} guarantees
\begin{align}
\inf_{\btheta} R_N(\btheta)  \le  R_{\sBayes} + \frac{1}{N}\left(2r\int \|\bomega\|_2|F(\bomega)| \de\bomega\right)^2\, \, ,\label{eq:Barron}
\end{align}
where $F$ is the Fourier transform of $f$, and $r$ is the supremum of $\|\bx\|_2$ in the support of $\prob$. 
This result is remarkable in that the minimum number of neurons needed to achieve a certain accuracy depends only on intrinsic
regularity properties of $f$ and not on the dimension $d$.  The proof of this and similar results shows that it is more insightful to think of 
the representation \eqref{eq:NNET} in terms of the empirical distribution of the neurons $\hrho^{(N)} \equiv N^{-1} \sum_{i\le N}\delta_{\btheta_i}$. With a slight abuse of notation,
we have $\hf_N(\bx;\btheta) = \hf(\bx;\hrho^{(N)})$, where, for a general distribution $\rho \in\cuP(\reals^D)$, we define
\begin{align}
\hf(\bx;\rho) = \int\! \sigma_\star(\bx;\btheta) \, \rho(\de\btheta)\, . \label{eq:hf-rho}
\end{align}
The universal approximation property is then related to the fact that an arbitrary distribution $\rho$  can be approximated by one supported on $N$ points\footnote{Of course,
here we are hiding some important technical issues.}.

Approximation theory provides some insight into the peculiar properties of neural networks. Small population risk is achieved by many networks,
since what matters is the distribution $\rho$, not the parameters $\btheta_1,\dots,\btheta_N$. The behavior is insensitive to the number of neurons $N$,
as long as this is large enough for $\hrho^{(N)}$ to approximate $\rho$. Finally, the bound \eqref{eq:Barron} is dimension-free. 

Of course these insights concern ideal representations, and not necessarily the networks generated by SGD.  Recently,
an analysis of SGD dynamics has been developed that connects naturally to the theory of universal approximation 
\cite{mei2018mean,sirignano2018mean,rotskoff2018neural,chizat2018global}.
 The main object of study is the empirical distribution $\hrho^{(N)}_k$ after $k$ SGD steps. For large $N$, small step size $\eps$ and setting $k=t/\eps$, 
$\hrho^{(N)}_k$  turns out to be well approximated by a probability distribution $\rho_t\in\cuP(\reals^D)$. The latter evolves according to the following partial 
differential equation 
\begin{align}
\partial_t\rho_t &= 2\xi(t)\nabla_\btheta \cdot \big(\rho_t \nabla_\btheta \Psi(\btheta;\rho_t)\big)\, ,\;\;\;\; \Psi(\btheta;\rho_t)\equiv V(\btheta)+\int U(\btheta,\tbtheta) \, \rho_t(\de\tbtheta)\, ,\label{eq:PDE} \tag{DD}\\
&V(\btheta) = -\E\{y\sigma_\star(\bx;\btheta)\}\, ,\;\;\;\;\; U(\btheta_1,\btheta_2) = \E\{\sigma_\star(\bx;\btheta_1)\sigma_\star(\bx;\btheta_2)\}\, . 
\end{align}
(Here $\xi(t)$ is a function that gauges the evolution of step size and will be defined below. In fact, there is little loss to the following discussion 
in setting $\xi(t) = 1$.)
We will refer to this as the \emph{mean field} description, or \emph{distributional dynamics}. This description has the advantage of being explicitly 
independent of the number of hidden units $N$ and hence accounts for one of the empirical findings described above
(the insensitivity to the number of neurons). Further, it allows to focus on some key elements of the dynamics
(global convergence, typical behavior) neglecting others (local minima, statistical noise).

Several papers used this approach over the last year to analyze learning in two-layers networks:
this work will be succinctly reviewed in Section \ref{sec:Related}.

Of course, a crucial question needs to be answered for this approach to be meaningful:  In what regime is the distributional dynamics
a good approximation to SGD?
Quantitative approximation guarantees were established in \cite{mei2018mean}, under certain regularity conditions on the data distribution $\prob$, and
for activation functions $\sigma_\star(\bx;\btheta)$ bounded. Under these conditions, and for time $t\in [0,T]$ bounded,  \cite{mei2018mean} proves that the
distributional dynamics solution $\rho_t$ approximates well the actual empirical distribution $\hrho^{(N)}_{k=t/\eps}$, when the number of neurons is much larger than
the problem dimensions $N\gg D$.

The results of \cite{mei2018mean} present several limitations, that we overcome in the present paper. We briefly summarize our contributions.
\begin{description}
\item[Dimension-free approximation.] As mentioned above, both classical approximation theory and the mean-field analysis of SGD 
approximate a certain target distribution $\rho$ by the empirical distributions of the network parameters $\hrho^{(N)}$. However, while 
the approximation bound (\ref{eq:Barron}) is dimension-free, the approximation guarantees of  \cite{mei2018mean} are explicitly dimension-dependent.
Even for very smooth functions $f(\bx)$, and well behaved data distributions, the results of  \cite{mei2018mean} require $N\gg D$. 

Here we prove a new bound that is dimension independent and therefore more natural. The proof follows a coupling argument which is different and more powerful
than the one of \cite{mei2018mean}. A key improvement consists in isolating different error terms, and developing a more delicate
concentration-of-measure argument which controls the dependence of the error on $N$.

Let us emphasize that capturing the correct dimension-dependence is an important test of the mean-field theory, and it
is crucial in order to compare neural networks  to other learning techniques (see Section \ref{sec:Connection}).
\item[Unbounded activations.] The approximation guarantee of \cite{mei2018mean} only applies to activation functions $\sigma_\star(\bx;\btheta_i)$
that are bounded. This excludes the important case of unbounded second-layer coefficients as in Eq.~\eqref{eq:NNET}. We extend our analysis 
to that case. This requires to develop an \emph{a priori} bound on the growth of the coefficients $a_i$. As in the previous point, our approximation guarantee
is dimension-free.
\item[Noisy SGD.] Finally, in some cases it is useful to inject noise into SGD. From a practical perspective this can help avoiding local minima.
From an analytical perspective, it corresponds to a modified PDE, which contains an additional Laplacian term 
$\Delta_{\btheta}\rho_t$. This PDE  has  smoother solutions $\rho_t$ 
that are supported everywhere and  converge globally to a unique fixed point \cite{mei2018mean}. 

In this setting, we prove a dimension-free approximation guarantee for the case of bounded activations. We also obtain a guarantee for
noisy SGD unbounded activations, but the latter is not dimension-free. 
\item[Kernel limit.] We analyze the PDE (\ref{eq:PDE}) in a specific short-time limit and show that it is well approximated
by a linearized dynamics. This dynamics can be thought as fitting a kernel ridge regression\footnote{`Kernel ridge regression' and `kernel regression' are used with somewhat different meanings in the literature. Kernel ridge regression uses global information and
can be defined as ridge regression in reproducing kernel Hilbert space (RKHS), while kernel regression uses local averages. See Remark \ref{rmk:kernel_ridge_regression} for a definition. } model with respect to  a kernel corresponding to the 
initial weight distribution $\rho_0$. We thus recover --from a different viewpoint-- a connection with kernel methods that
has been investigated in several recent papers \cite{jacot2018neural,du2018gradient,du2018gradient2,allen2018convergence}. Beyond the short time scale, the dynamics is analogous to kernel boosting dynamics with a time-varying data-dependent kernel
(a point that already appears in \cite{rotskoff2018neural}). 
\end{description}
Mean-field theory allowed to prove global convergence guarantees for SGD in two-layers neural networks \cite{mei2018mean,chizat2018global}.
Unfortunately, these results do not provide (in general) useful bounds on the network size $N$. We believe that the results in this paper
are a required step in that direction.

The rest of this paper is organized as follows. The next section overviews related work, focusing in particular on 
the distributional dynamics  (\ref{eq:PDE}), its variants and applications. In Section \ref{sec:Main} we present formal statements of our results. 
Section \ref{sec:Connection} develops the connection with kernel methods. Proofs are mostly deferred to the appendices. 

\section{Related work}
\label{sec:Related}

As mentioned above, classical approximation theory already uses (either implicitly or explicitly) the idea of lifting the class of
$N$-neurons neural networks, cf. Eq.~\eqref{eq:NNET}, to the infinite-dimensional space \eqref{eq:hf-rho} parametrized by probability distributions $\rho$,
see e.g. \cite{cybenko1989approximation,barron1993universal,bartlett1998sample,anthony2009neural}. This idea was exploited algorithmically, e.g. in 
\cite{bengio2006convex,nitanda2017stochastic}. 

Only very recently (stochastic) gradient descent  was proved to converge (for large enough number of neurons) to the infinite-dimensional 
evolution \eqref{eq:PDE} \cite{mei2018mean,rotskoff2018neural,sirignano2018mean,chizat2018global}. In particular, \cite{mei2018mean}
proves quantitative bounds to approximate SGD by the mean-field dynamics. Our work is mainly motivated by the objective to obtain a better
scaling with dimension and to allow for unbounded second-layer coefficients. 

The mean-field description was exploited in several papers to establish global convergence results. 
In \cite{mei2018mean} global convergence was proved in special examples, and in a general setting for noisy SGD.  The papers
\cite{rotskoff2018neural,chizat2018global} studied global convergence by exploiting the homogeneity properties of Eq.~\eqref{eq:NNET}.
In particular, \cite{chizat2018global} proves a general global convergence result. For initial conditions $\rho_0$ with full support, the PDE \eqref{eq:PDE} converges
to a global minimum provided activations are homogeneous in the parameters. Notice that the presence of unbounded second layer
coefficients is crucial in order to achieve homogeneity. Unfortunately, the results of \cite{chizat2018global}  do not provide quantitative approximation
bounds relating the PDE \eqref{eq:PDE} to finite-$N$ SGD. The present paper fills this gap by establishing approximation bounds that apply to 
the setting of \cite{chizat2018global}.

A different optimization algorithm was studied in \cite{wei2018margin} using the mean-field description. The 
algorithm resamples a positive fraction of the neurons uniformly at random at a constant rate. This allows the authors
to establish a global convergence result (under certain assumed smoothness properties on the PDE solution). Again, this paper does not provide
quantitative bounds on the difference between PDE and finite-$N$ SGD. While our theorems do not cover the algorithm of 
\cite{wei2018margin}, we believe that their algorithm could be analyzed
using the approach developed here. Exponentially fast convergence to a global optimum was proven in \cite{javanmard2019analysis}
for certain radial-basis-function networks, using again the mean-field approach. While the setting of  \cite{javanmard2019analysis}
is somewhat different (weights are constrained to a convex compact domain), the technique presented here could be applicable to that problem as well.

Finally, a recent stream of works \cite{jacot2018neural,geiger2019scaling,du2018gradient,du2018gradient2,allen2018convergence} argues that, as $N\to\infty$ two-layers networks are actually performing a type of
kernel ridge regression. 
As shown in  \cite{chizat2018note}, this phenomenon is not limited to neural network, but generic for a broad class of models.
As expected, the kernel regime can indeed be recovered as a special limit of the mean-field dynamics \eqref{eq:PDE}, cf. Section \ref{sec:Connection}. Let us emphasize that here we focus on the population rather than the empirical risk.

A discussion of the difference between the kernel and mean-field regimes was recently presented in \cite{dou2019training}.
However, \cite{dou2019training} argues that the difference between  kernel and mean-field behaviors is due to different initializations
of the coefficients $a_i$'s. We show instead that, for a suitable scaling of the initialization, kernel and mean field regimes
appear at different time scales. Namely, the kernel behavior 
arises at the beginning of the dynamics, and mean field 
characterizes longer time scales.
It is also worth mentioning that the connection between mean field dynamics and kernel boosting  with a time-varying data-dependent kernel was already present (somewhat implicitly) in \cite{rotskoff2018neural}.

\section{Dimension-free mean field approximation}
\label{sec:Main}

\subsection{General results}

As mentioned above, we assume to be given data $\{(y_k,\bx_k)\}_{k\ge 1}\sim_{i.i.d.}\prob\in\cuP(\reals\times\reals^d)$, and we run SGD with 
step size $s_k$:
\begin{align}\label{eq:SGD}\tag{SGD}
\btheta_i^{k+1} = \btheta^k_i + 2 s_k (y_k - \hf_N(\bx_k;\btheta^k)) \na_{\btheta } \sigma_\star ( \bx_k ; \btheta^k_i ).
\end{align}
We will work under a one-pass model, that is, each data point is visited once.

We also consider a noisy version of SGD, with a regularization term:
\begin{align}
\label{eq:NoisySGD}\tag{noisy-SGD}
\btheta_i^{k+1} = (1 - 2 \lambda s_k ) \btheta^k_i + 2 s_k (y_k - \hf_N(\bx_k;\btheta^k) ) \na_{\btheta} \sigma_\star ( \bx_k ; \btheta^k_i ) + \sqrt{2 s_k \tau/D} \; {\bg}^k_i  \, ,
\end{align}
where $\bg^k_i \sim \cN ( 0 , \id_{D} )$. The noiseless version is recovered by setting $\tau=0$ and $\lambda = 0$.
The step size is chosen according to : $s_k = \eps \xi (k\eps ) $, for a positive function $\xi : \R_{\geq 0} \to \R_{>0}$.

The infinite-dimensional evolution corresponding to noisy SGD is given by
\begin{align}
    &\partial_t \rho_t =
 2 \xi (t) \nabla_{\btheta} \cdot \big(\rho_t(\btheta) \nabla_\btheta \Psi_{\lambda} (\btheta; \rho_t)\big)+ 2 \xi (t)\tau D^{-1} \Delta_{\btheta} \rho_t \, ,\label{eq:PDEnoisy}\tag{diffusion-DD} \\
    &\Psi_\lambda (\btheta; \rho) = \Psi(\btheta; \rho)  + \frac{\lambda }{2} \norm{\btheta}_2^2\, .
\end{align}
The function $\Psi$ is defined as in \eqref{eq:PDE}. At this point it is important to note that the PDE \eqref{eq:PDE} has to be interpreted in weak sense, while,
for $\tau>0$, Eq.~\eqref{eq:PDEnoisy} has strong solutions i.e. solutions $\rho:(t,\btheta)\mapsto \rho_t(\btheta)$ that are $C^{1,2}(\reals\times\reals^D)$
(once continuous differentiable in time and twice in space, see \cite{mei2018mean} and Appendix  \ref{sec:existence_uniqueness}).

It is useful to lift the population risk in the space of distributions $\rho\in\cuP(\reals^D)$
\begin{align}
R(\rho) = \E(y^2) + 2 \int V(\btheta) \rho(\de \btheta) + \int U(\btheta, \btheta') \rho(\de \btheta) \rho(\de \btheta')\, .
\end{align}
We also note that, given the structure of the activation function in Eq.~\eqref{eq:NNET}, for $\btheta = (a,\bw)$, 
$\btheta_i = (a_i,\bw_i)$, we can write $V(\btheta) = a\, v(\bw)$, $U(\btheta_1,\btheta_2) = a_1a_2\, u(\bw_1,\bw_2)$, where $v(\bw) = - \E\{ y \sigma(\bx; \bw)\}$ and $u(\bw_1, \bw_2) = \E\{ \sigma(\bx; \bw_1) \sigma(\bx; \bw_2)\}$. 

In order to establish a non-asymptotic guarantee, we will make the following assumptions: 
\begin{itemize}
\item[A1.] $t \mapsto \xi (t)$ is bounded Lipschitz: $\norm{\xi}_\infty $, $\norm{\xi}_{\text{Lip}} \leq K_1$. 
\item[A2.] The activation function $\sigma: \R^d\times\R^{D-1} \to \R$ and the response variables are bounded: $\norm{\sigma}_\infty, |y_k|\leq K_2
  $. Furthermore, its gradient $\nabla_{\bw}\sigma(\bx;\bw)$ is $K_2$-sub-Gaussian (when  $\bx\sim\prob$).
\item[A3.] The functions $\bw \mapsto v(\bw)$ and $(\bw_1,\bw_2)\mapsto u(\bw_1,\bw_2)$ are differentiable, with bounded and Lipschitz continuous gradient:
$\| \nabla v(\bw) \|_2 \le K_3$, $\| \nabla u(\bw_1, \bw_2) \|_2 \le
K_3$, $\|\nabla v(\bw) -\nabla v(\bw')\|_2\le K_3\|\bw-\bw'\|_2$,
$\|\nabla u(\bw_1,\bw_2) -\nabla u(\bw'_1,\bw'_2)\|_2\le
K_3\|(\bw_1,\bw_2)-(\bw'_1,\bw'_2)\|_2$. 
\item[A4.] The initial condition $\rho_0 \in \cuP(\reals^D)$ is supported on $\vert a_i\vert \le K_4$ for a constant $K_4$. 
\end{itemize}

We will consider two different cases for the SGD dynamics:
\begin{description}
\item[General coefficients.] We initialize the parameters $\btheta^0_i = (a^0_i,\bw^0_i)$ as $(\btheta^0_i )_{i \leq N} \sim_{iid} \rho_0$. Both the $a_i^0$ and $\bw^0_i$ are updated during the dynamics.
\item[Fixed coefficients.] We use the same initialization as described above, but the coefficients $a_i$ are not updated by SGD.
The corresponding PDE  is given by Eq.~\eqref{eq:PDE} (or \eqref{eq:PDEnoisy}), except that the space derivatives are to be interpreted only with respect to $\bw$, i.e. replace $\na_\btheta$ by $(0,\na_\bw)$, and $\Delta_\btheta$ by $\Delta_\bw$.
\end{description}
While the second setting is less relevant in practice, it is at least as interesting  from  a theoretical point of view, and some of our guarantees are stronger in that case.

\begin{theorem}\label{thm:bound_approximating_noiseless}
Assume that conditions {\rm A1}-{\rm A4} hold, and let $T \ge 1$. Let $(\rho_t)_{t\ge 0}$ be the solution of the PDE (\ref{eq:PDE}) with initialization $\rho_0$, and let $(\btheta^k)_{k \in \N}$ to be the trajectory of SGD (\ref{eq:SGD}) with initialization $\btheta_i^0 \sim \rho_0$ independently. 
\begin{itemize}
    \item[(A)] Consider noiseless SGD with \emph{fixed coefficients}. 
Then there exists a constant $K$ (depending uniquely on the constants $K_i$ of assumptions {\rm A1}-{\rm A4}) such that
    \begin{align}\label{eq:FixedThm}
    \sup_{k \in [0, T/\eps] \cap \N} \abs{R_N(\btheta^k) - R(\rho_{k\eps})} \le K e^{KT} \frac{1}{\sqrt{N}}[ \sqrt{\log N} + z]   + K e^{KT} [ \sqrt{D + \log(N)} + z] \sqrt{\eps} 
    \end{align}
    with probability at least $1 - e^{- z^2}$.
    \item[(B)] Consider noiseless SGD with \emph{general coefficients}. Then there exists constants $K$ and $K_0$ (depending uniquely on the constants $K_i$ of assumptions {\rm A1}-{\rm A4})
 such that if $\eps \le 1/ [K_0 (D + \log N + z^2) e^{K_0 T^3}]$, we have
    \begin{align}\label{eq:GeneralThm}
    \sup_{k \in [0, T/\eps] \cap \N} \abs{R_N(\btheta^k) - R(\rho_{k\eps})} \le K e^{K T^3}\frac{1}{\sqrt{N}}[ \sqrt{\log N} + z]  + Ke^{K T^3}  [ \sqrt{D + \log N} + z] \sqrt{\eps} 
    \end{align}
    with probability at least $1 - e^{- z^2}$.
\end{itemize}
\end{theorem}

\begin{remark}
As anticipated in the introduction, provided  $T,K=O(1)$, the error terms in Eqs.~\eqref{eq:FixedThm},  \eqref{eq:GeneralThm}, are small as soon as $N\gg 1$.
In other words, the minimum number of neurons needed for the mean-field approximation to be accurate is independent of the dimension $D$, and only depends on intrinsic features
of the activation and data distribution. 

On the other hand, the dimension $D$ appears explicitly in conjunction with the step size $\eps$. We need $\eps\ll 1/D$ in order for mean field to be accurate.
This is the same trade-off between step size and dimension that was already achieved in \cite{mei2018mean}.
\end{remark}

We next consider noisy SGD, cf. Eq.~\eqref{eq:NoisySGD}, and the corresponding PDE in Eq.~\eqref{eq:PDEnoisy}. We need to make additional assumptions on the initialization
in this case.
\begin{itemize}
\item[A5.] The initial condition $\rho_0$ is such that, for $\btheta^0_i = (a^0_i,\bw^0_i)\sim \rho_0$, we have that $\bw^0_i$ is $K_5^2/D$-sub-Gaussian. 
\item[A6.] $V \in C^4 (\R^D)$, $U \in C^4 (\R^D \times \R^D )$, and $\na^k_1 u (\btheta_1 , \btheta_2 )$ is uniformly bounded for $0 \le k \le 4$.
\end{itemize}
\begin{remark}
The last condition ensures the existence of strong solutions for Eq.~\eqref{eq:PDEnoisy}. The existence and uniqueness of solution of the PDE \eqref{eq:PDE} and the PDE \eqref{eq:PDEnoisy} are discussed in Appendix \ref{sec:existence_uniqueness}.
\end{remark}
\begin{theorem}\label{thm:bound_approximating_noisy}
Assume that conditions {\rm A1} - {\rm A6} hold. Let $(\rho_t)_{t\ge 0}$ be the solution of the PDE (\ref{eq:PDEnoisy}) with initialization $\rho_0$, and let $(\btheta^k)_{k \in \N}$ to be the trajectory of noisy SGD (\ref{eq:NoisySGD}) with initialization $\btheta_i^0 \sim \rho_0$ independently. Finally assume that $\lambda \leq K_6$, $\tau \leq K_6$, $T\ge 1$.
\begin{itemize}
\item[(A)] Consider noisy SGD with \emph{fixed coefficients}.
Then there exists a constant $K$ (depending uniquely on the constants $K_i$ of assumptions {\rm A1}-{\rm A5} and $K_6$) such that
\begin{align}
    \sup_{k \in [0, T/\eps] \cap \N} \abs{R_N(\btheta^k) - R(\rho_{k\eps})} \le K e^{KT} \frac{1}{\sqrt{N}}[ \sqrt{\log N} + z]   + K e^{KT} [ \sqrt{D + \log(N / \eps)} + z] \sqrt{\eps}
\end{align}
    with probability at least $1 - e^{- z^2}$.
\item[(B)] Consider noisy SGD with \emph{general coefficients}.  Then there exists a constant $K$ (depending uniquely on the constants $K_i$ of assumptions {\rm A1}-{\rm A5} and $K_6$) such that
    \begin{align}
    \sup_{k \in [0, T/\eps] \cap \N} \abs{R_N(\btheta^k) - R(\rho_{k\eps})} \le & K e^{e^{KT} [ \sqrt{\log N } + z^2 ]} [ \sqrt{D\log N} + \log^{3/2} (NT) + z^5 ]/\sqrt{N}  \\
    & +Ke^{e^{KT}[\sqrt{\log N} +z^2]} [ \sqrt{D}\log (N (T / \eps \vee 1)) + \log^{3/2} N + z^6 ] \sqrt{\eps } \nonumber
    \end{align}
    with probability at least $1 - e^{- z^2}$.
\end{itemize}
\end{theorem}
\begin{remark}
Unlike the other results in this paper, part $(B)$ of Theorem \ref{thm:bound_approximating_noisy} does not establish a dimension-free bound. Further,
while previous bounds allow to control the approximation error for any $T= o(\log N)$,  Theorem \ref{thm:bound_approximating_noisy}.$(B)$ requires
$T=o(\log\log N)$ . The main difficulty in part $(B)$ is to control the growth of the coefficients $a_i$. This is more challenging than in the noiseless case, 
since we cannot give a deterministic bound on $|a_i|$.

Despite these drawbacks, Theorem \ref{thm:bound_approximating_noisy} $(B)$ is the first quantitative bound 
approximating noisy SGD by the distributional dynamics, for the case of unbounded coefficients. It implies that the mean field theory
is accurate when $N\gg D$. 
\end{remark}

\subsection{Example: Centered anisotropic Gaussians}

To illustrate an application of the theorems, we consider the problem
of classifying two Gaussians with the same mean and different
covariance. This example was studied in \cite{mei2018mean}, but we restate it here for the reader's convenience. 

Consider the joint distribution of data $(y,\bx)$ given by the following:
\begin{itemize}
\item[] With probability $1/2$: $y=+1$, $\bx\sim\normal(0,\bSigma_+)$,
\item[] With probability $1/2$: $y=-1$, $\bx\sim\normal(0,\bSigma_-)$,
\end{itemize}
where $\bSigma_{\pm} = \bU^\sT \diag((1 \pm \Delta)^2 \id_{s_0}, \id_{d - s_0}) \bU$ for $\bU$ to be an unknown orthogonal matrix. 
In other words, there exists a subspace $\cV$ of dimension $s_0$, such that the projection of $\bx$ on the subspace $\cV$ is distributed according to an isotropic Gaussian with variance $\tau_+^2 = (1 + \Delta)^2$ (if
$y=+1$) or $\tau_-^2 = (1 - \Delta^2)$ (if $y=-1$). The projection orthogonal to $\cV$ has instead the same variance in the two classes.

We  choose an activation function without offset or output weights, namely $\sigma_*(\bx;\btheta_i) = \sigma(\<\bw_i,\bx\>)$. 
While qualitatively similar results are obtained for other choices of $\sigma$, 
we will use a simple piecewise linear function (truncated ReLU) as a running example: take $t_1 < t_2$,
\[
\sigma(t) = \left\{\begin{aligned} &s_1, &~~& \text{ if } t\le t_1, \\
&s_2, &~~&  \text{ if } t\ge t_2, \\
&s_1 + (s_2 - s_1) (t - t_1) / (t_2 - t_1), &~~& \text{ if } t\in (t_1,t_2).
\end{aligned} \right.
\]
We introduce a class of good uninformative initializations $\cuP_{\good}\subseteq\cuP(\reals_{\ge 0})$ for which convergence to the optimum takes place. For $\rad\in\cuP(\reals_{\ge 0})$, we let 
\[
\oR_d(\rad) \equiv R(\rad \times {\rm Unif}(\mathbb S^{d - 1})), ~~~~~ \oR_\infty(\rad) \equiv \lim_{d \to \infty} \oR_d(\rad).
\]
We say that $\rad\in\cuP_{\good}$ if: $(i)$ $\rad$ is absolutely continuous with respect to Lebesgue measure, with bounded density; $(ii)$ $\oR_{\infty}(\rad) < 1$.

The following theorem is an improvement of \cite[Theorem 2]{mei2018mean} using Theorem \ref{thm:bound_approximating_noiseless}, whose proof is just by replacing the last step of proof of \cite[Theorem 2]{mei2018mean} using the new bounds developed in \ref{thm:bound_approximating_noiseless} (A). 
\begin{theorem}\label{thm:ConvergenceAnisotropic}
For any $\eta, \Delta, \delta > 0$, and $\rad_0 \in \cuP_{\good}$, there exists $d_0 = d_0(\eta, \rad_0, \Delta,\gamma)$, $T = T(\eta, \rad_0, \Delta,\gamma)$, and $C_0 = C_0(\eta, \rad_0, \Delta, \delta,\gamma)$, such that the following holds for the problem of 
classifying anisotropic Gaussians with $s_0=\gamma d$, $\gamma\in (0,1)$ fixed.
For any dimension parameters  $s_0 = \gamma d \ge d_0$, number of neurons $N \ge C_0$, consider SGD initialized with initialization $(\bw_i^0)_{i \le N} \sim_{iid}  \rad_0 \times {\rm Unif}(\mathbb S^{d - 1})$ and step size $\eps \le 1/(C_0 d)$. Then we have 
$R_{N}(\btheta^{k}) \le \inf_{\btheta\in\reals^{N\times d}} R_{N}(\btheta) + \eta$ for any $k \in [T/\eps, 10 T/\eps]$ with probability at least $1 - \delta$.
\end{theorem}
Comparing to \cite[Theorem 2]{mei2018mean}, here we require $N = O(1)$ neuron rather than previously $N = O(d)$ neurons. The number of data used $k = O(d)$ is still on the optimal order.

\section{Connection with kernel methods}
\label{sec:Connection}

As discussed above, mean-field theory captures the SGD dynamics of two layers neural networks when the number of hidden units $N$ is large.
Several recent papers studied a different description, that approximates the neural network as performing a form of kernel ridge regression 
\cite{jacot2018neural,du2018gradient}. This behavior also arises for large $N$: we 
will refer to this as to the `kernel regime', or `kernel limit'. As shown in \cite{chizat2018note} the existence of a kernel regime
is not specific to neural networks but it is a generic feature of overparameterized models, under certain differentiability
assumptions.

\subsection{A coupled dynamics}

 We will focus on noiseless gradient flow, and assume $y = f(\bx)$ (a general joint distribution over $(y,\bx)$ is recovered by setting
$f(\bx) = \E\{y|\bx\}$).  As in \cite{chizat2018note}, we modify the model \eqref{eq:NNET} by introducing an additional scale parameter $\alpha$:
\begin{align}
\hf_{\alpha, N}(\bx;\btheta) = \frac{\alpha}{N}\sum_{i=1}^N\bsigma_{\star}(\bx;\btheta_i)\, ,
\end{align}
In the case of general coefficients $a_i$, this amounts to rescaling the coefficients $a_i\to a_i/\alpha$. Equivalently, this corresponds  to a different initialization
for the $a_i$'s (larger by a factor $\alpha$). 

We first note that the theorems of the previous section obviously 
hold for the modified dynamics, with the PDE \eqref{eq:PDE} generalized to
\begin{align}
\partial_t \rho_t =& \alpha \nabla_\btheta \cdot \big(\rho_t\nabla_\btheta \Psi_\alpha(\btheta,\rho_t)\big) \, ,\;\; \\
\Psi_\alpha(\btheta,\rho) =& \E_\bx\big\{\sigma_{\star}(\bx;\btheta)\, ( \hf_\alpha(\bx;\rho)-f(\bx))\big\} = V(\btheta) + \alpha \int U(\btheta, \btheta') \rho(\de \btheta')\,,
\end{align}
where $\hf_\alpha(\bx;\rho)  =\alpha\int \sigma_{\star}(\bx;\btheta)\, \rho(\de\btheta)$. It is convenient to redefine time units by letting
$\rho^{\alpha}_{t}\equiv \rho_{\alpha^{-2} t}$. This satisfies the \emph{rescaled distributional dynamics}
\begin{align}\tag{Rescaled-DD}
\partial_t \rho^{\alpha}_t =  \frac{1}{\alpha}\nabla_\btheta \cdot \big(\rho^{\alpha}_t\nabla_\btheta \Psi_\alpha(\btheta,\rho^{\alpha}_t)\big)\, .  \label{eq:RescaledPDE}
\end{align}
We next consider the residuals $u_t^{\alpha}(\bx) = f(\bx)-f(\bx;\rho^{\alpha}_t)$ which we view as an element of $\Ltwo = \Ltwo(\reals^d;\prob)$. As 
first shown in \cite{rotskoff2018neural}, this satisfies the following \emph{mean field residual dynamics} (for further background, we refer to Appendix  \ref{sec:mean_field_kernel_appendix}):
\begin{align}
\partial_t u_t^{\alpha}(\bx) &= -\int \cH_{\rho^{\alpha}_t}(\bx,\tbx) u_t^\alpha(\tbx) \, \prob(\de\tbx) \equiv -(\cH_{\rho_t^\alpha}u^{\alpha}_t)(\bx)\, ,\label{eq:Residual}\tag{RD}\\
\cH_{\rho}(\bx,\tbx) & \equiv \int \<\nabla_{\btheta}\sigma_{\star}(\bx;\btheta),\nabla_{\btheta}\sigma_{\star}(\tbx;\btheta)\>\, \rho(\de\btheta)\, .
\end{align}
Coupling the dynamics (\ref{eq:RescaledPDE}) and (\ref{eq:Residual}) suggests the following point of description. Gradient flow dynamics of two-layers neural network is a kernel boosting dynamics with a time-varying kernel. The scaling parameter $\alpha$ controls the speed that the kernel evolves. 

The mean field residual dynamics (\ref{eq:Residual}) implies that
\[
\partial_t R_\alpha(\rho_t^\alpha) = \partial_t (\| u_t^\alpha \|_{\Ltwo}^2) = - 2 \< u_t^\alpha, \cH_{\rho_t^\alpha} u_t^\alpha\>_{\Ltwo},
\]
so that the risk will be non-increasing along the gradient flow dynamics. However, since the kernel $\cH_{\rho_t^\alpha}$ is not fixed, it is hard to analyze when the risk converges to $0$
(see \cite[Theorem 4]{mei2018mean}, \cite[Theorem 3.3 and 3.5]{chizat2018global} for general convergence results).

\subsection{Kernel limit of residual dynamics}

The kernel regime corresponds to large $\alpha$ and allows for a simpler treatment of the dynamics. Heuristically,
the reason for such a simplification is that the time derivative of $\rho_t^\alpha$ is of order $1/\alpha$, cf. \eqref{eq:RescaledPDE}. We are therefore tempted to replace $\cH_{\rho^{\alpha}_t}$ in Eq.~\eqref{eq:Residual} by $\cH_{\rho_0}$. Formally, we define the following \emph{linearized residual dynamics} 
\begin{align}
\partial_t u^*_t = -\cH_{\rho_0}u^*_t\, .\label{eq:Linearized}
\end{align}
We can also define the corresponding predictors by $f^*_t = f - u_t^*$.
The operator $\cH_{\rho_0}$ is bounded and standard semigroup theory \cite{evans2009partial} implies the following.
\begin{lemma}\label{lem:convergence_linear_residual_dynamics}
We have $\lim_{t\to\infty}u^*_t= u^{*}_{\infty} = \proj_{\rho_0}u^*_{0}$, where $\proj_{\rho_0}$ is the orthogonal projector onto the null space of
$\cH_{\rho_0}$. In particular, if the null space of $\cH_{\rho_0}$ is empty, then $\lim_{t\to\infty}\| u^{*}_t\|_{\Ltwo}\to 0$. Correspondingly $f^*_{\infty} = \proj^{\perp}_{\rho_0}f+  \proj_{\rho_0}f_0^*$ (where $\proj^{\perp}_{\rho_0}=\id-\proj_{\rho_0}$). 
\end{lemma}

The next theorem shows that the above intuition is correct. For $\alpha \ge t^2 D^{3/2}$, the linearized dynamics is a good approximation to
the mean field dynamics. 
Below, we denote the population risk by $R_\alpha(\rho)$: $R_\alpha(\rho) \equiv \E_{\bx}[(f(\bx) - \hf_\alpha(\bx; \rho))^2]$.
\begin{theorem}\label{thm:Kernel}
Let $u^{\alpha}_t$ and $u^*_t$ be the residues in the mean-field dynamics (\ref{eq:Residual}) and linearized dynamics (\ref{eq:Linearized}), respectively. Let assumptions {\rm A1}, {\rm A3}, {\rm A4} hold, and additionally  assume the following
\begin{itemize}
    \item $|y_i|$, $\|\sigma\|_{\infty}\le K_2$, and $\btheta\mapsto \sigma_{\star}(\bx;\btheta)$ is differentiable.
    \item $\| \nabla^3 u(\bw, \bw') \|_{\op}, \| \nabla^4 u(\bw, \bw') \|_{\op} \le \kappa$.
    \item $R_\alpha(\rho_0) \le B$.
\end{itemize}
Then there exists a constant $K$ depending on $\{ K_i \}_{i=1}^4$, such that
\begin{itemize}
\item[(A)] For SGD with fixed coefficients, we have 
\begin{align}
\|u^{\alpha}_t-u^{*}_t\|_{L_2} \le& K \kappa^{1/2}  B \frac{D^{3/2}t^2}{\alpha} \, ,\label{eq:residual_difference_fixed}\\
R_\alpha(\rho^{\alpha}_t) \le& \Big(\|u^{*}_t\|_{\Ltwo}+ K \kappa^{1/2}  B \frac{D^{3/2}t^2}{\alpha}\Big)^{2}\,. \label{eq:risk_bound_fixed}
\end{align}
\item[(B)] For SGD with general coefficients, we have 
\begin{align}
\|u^{\alpha}_t-u^{*}_t\|_{L_2} \le&K \kappa^{1/2} (1+ B^{1/2} t / \alpha)^3    B \frac{D^{3/2}t^2}{\alpha}\, , \label{eq:residual_difference_general}\\
R_\alpha(\rho^{\alpha}_t) \le& \Big(\|u^{*}_t\|_{\Ltwo} + K \kappa^{1/2} (1+ B^{1/2} t / \alpha)^3    B \frac{D^{3/2}t^2}{\alpha}\Big)^2\,.\label{eq:risk_bound_general}
\end{align}
\item[(C)] In particular, if under the law $(a, \bw) \sim \rho_0$, $a$ is independent of $\bw$ and $\vert \E(a) \vert \le K_5/\alpha$. Then
 $B\le K$ is independent of $\alpha$. If the null space of $\cH_{\rho_0}$ is empty, then under both settings (fixed and variable coefficients)
\begin{align}\label{eqn:risk_goes_to_0}
\lim_{\alpha \to \infty} \sup_{t \in [0, T]} \| u^\alpha_t - u^*_t \|_{L_2} =& 0,\\
\lim_{t \to \infty}\lim_{\alpha \to \infty} R_\alpha(\rho^\alpha_t) =& 0\, . 
\end{align}
\end{itemize}
\end{theorem}

\begin{remark}
Unlike in similar results in the literature, we focus here on the 
population risk rather than the empirical risk.
The recent paper \cite{chizat2018note} addresses both the 
overparametrized and the underparametrized regime.
The latter result (namely \cite[Theorem 3.4]{chizat2018note}) is of course relevant for the population risk. However, while \cite{chizat2018note}
proves convergence to a local minimum, here we show
that the population risk becomes close to $0$. 
\end{remark}

\begin{remark}
As stated above, the linearized residual dynamics can be interpreted as performing kernel ridge regression with respect to the kernel $\cH_{\rho_0}$, see e.g. \cite{jacot2018neural}. A way to clarify the connection is to consider the case in which $\prob = n^{-1}\sum_{i\le n}\delta_{\bx_i}$ is the empirical data distribution. In this case the linearized dynamics converges to
\[
\lim_{t \to \infty} f^*_t(\bz) = f^*_{\infty}(\bz) = \bh(\bz)^\sT \bH^{-1}\by
\]
where 
\[
\begin{aligned}
\bh(\bz) =& [\cH_{\rho_0}(\bz, \bx_1), \ldots, \cH_{\rho_0}(\bz, \bx_n)]^\sT,\\
\bH =& (\cH_{\rho_0}(\bx_i, \bx_j))_{i j = 1}^n,\\
\by =& [f(\bx_1), \ldots, f(\bx_n)]^\sT. \\
\end{aligned}
\]
For the sake of completeness, we review the connection in Appendix \ref{sec:KernelRidgApp}. 
\end{remark}

\bibliographystyle{amsalpha}\bibliography{main.bbl}
\addcontentsline{toc}{section}{References}

\clearpage

\appendix

\section{Notations}

\begin{itemize}
\item For future reference, we copy the key definitions from the main text: 
\[
\begin{aligned}
R_N  (\btheta) & =  \E \{ y^2 \} + \frac{2}{N} \sum_{i = 1}^N V ( \btheta_i ) + \frac{1}{N^2} \sum_{i,j=1}^N U ( \btheta_i , \btheta_j ), \\
R  (\rho) & =  \E \{ y^2 \} + 2 \int V ( \btheta ) \rho (\de \btheta) + \int U ( \btheta_1 , \btheta_2 ) \rho (\de \btheta_1 ) \rho (\de \btheta_2 ), \\
V ( \btheta ) & =  - \E \{ y \sigma_\star ( \bx ; \btheta ) \}, \qquad U ( \btheta_1 , \btheta_2 ) = \E \{ \sigma_\star ( \bx ; \btheta_1 ) \sigma_\star ( \bx ; \btheta_2 ) \},\\
\Psi ( \btheta ; \rho ) & = V (\btheta ) + \int U ( \btheta , \btheta ') \rho ( \de \btheta ' ), \\
\Psi_\lambda ( \btheta ; \rho ) & = \Psi ( \btheta ; \rho ) + \frac{\lambda}{2} \| \btheta \|_2^2,
\end{aligned}
\]
where $\btheta = ( \btheta_i )_{i \le N} \in \R^{D \times N}$ or $\btheta \in \R^D$ depending on the context. Further, we will denote for $\btheta = (a , \bw)$ and $\btheta ' = (a' , \bw')$:
\[
V ( \btheta ) = a v(\bw ), \qquad U ( \btheta , \btheta' ) = aa' u (\bw , \bw '). 
\]
In particular,
\[
\nabla_{\btheta} V ( \btheta ) = ( v(\bw ), a \na_\bw v (\bw) ) , \qquad \na_{\btheta} U ( \btheta , \btheta' ) = (a' u (\bw , \bw '), aa' \na_\bw u (\bw , \bw' ) ). 
\]
In the case of fixed coefficients, without loss of generality, we will fix in the proof $a_i = 1$ for notational simplicity and freely denote $(\btheta_i)_{i=1}^N = (\bw_i)_{i = 1}^N$,
\[
\begin{aligned}
V ( \btheta ) & = v(\bw ), \qquad& U ( \btheta , \btheta' ) & = u (\bw , \bw '), \\
\nabla_{\btheta} V ( \btheta ) & = \na_\bw v (\bw)  , \qquad & \na_{\btheta} U ( \btheta , \btheta' ) & = \na_\bw u (\bw , \bw' ) .
\end{aligned}
\]

\item $W_2 ( \cdot , \cdot )$ is the Wasserstein distance between probability measures
\[
W_2 (\mu , \nu ) = \Big( \inf \Big\{ \int_{\R^D \times \R^D} \| \btheta_1 - \btheta_2 \|_2^2 \gamma (\de \btheta_1 , \de \btheta_2 ): \text{ $\gamma$ is a coupling of $\mu,\nu$} \Big\} \Big)^{1/2}.
\]

\item For $N \in \N$, we will denote $[N] = \{ 1,2,\ldots , N \}$. With a little abuse of notation, for $s \in \R$, we will denote $[s] = \eps \lfloor s/ \eps \rfloor$, with $\eps$ the time discretization parameter.

\item K will denote a generic constant depending  on $K_i$ for $i= 1,2,3,4,5,6$, where the $K_i$'s are constants that will be specified from the context.

\item In the proof and the statements of the theorems, we will only consider the leading order in $T$. In particular, we freely use that $K T^k \log^l T e^{KT} \le K' e^{K' T}$ for a constant $K' \ge K$.

\item For readers convenience, we copy here the two simplified versions of Gronwall's lemma that will be used extensively in the proof. 
\begin{enumerate}[label=(\roman*)]
    \item Consider an interval $I = [0,t]$ and $\phi$ a real-valued function defined on $I$, assume there exists positive constants $\alpha,\beta$ such that $\phi$ satisfies the integral inequality
\[
\phi (t) \le \alpha + \beta \int_0^t \phi (s) \de s, \qquad \forall t\in I,
\]
then $\phi (t) \le \alpha e^{\beta t}$ for all $t\in I$.
\item Consider a non-negative sequence $\{\phi_k \}_{k = 0}^n$ and assume there exists positive constants $\alpha,\beta$ such that $\{\phi_k \}_{k = 0}^n$ satisfies the summation inequality
\[
\phi_k \le \alpha + \beta \sum_{0 \le l < k} \phi_l , \qquad \forall k \in \{ 0,1 , \ldots, n \},
\]
then $\phi_k  \le \alpha+ \alpha \beta k e^{\beta k}$ for all $ k \in \{ 0,1 , \ldots, n \}$.
\end{enumerate}

\end{itemize}

\clearpage

\section{Proof of Theorem \ref{thm:bound_approximating_noiseless} part (A)}\label{sec:proof_1A}

Throughout this section, the assumptions of Theorem \ref{thm:bound_approximating_noiseless} (A) are understood to hold. These are assumptions {\rm A1}-{\rm A4} in Section \ref{sec:Main}. 
%
In writing the proofs, for notational simplicity, we consider the following special setting: 
\begin{itemize}
\item[R1.] The coefficients $a_i \equiv 1$. 
\item[R2.] The step size function $\xi(t) \equiv 1/2$. 
\end{itemize}
The proof can be easily generalized to the case of general bounded coefficient $\vert a_i \vert \le K$, and non-constant function $\xi(t)$.

In the proof of this theorem, we have $(\btheta_i)_{i=1}^N = (\bw_i)_{i = 1}^N$, and 
\[
\begin{aligned}
V(\btheta_i) =& a_i v(\bw_i) = v(\bw_i),\\
U(\btheta_i, \btheta_j) =& a_i a_j u(\bw_i, \bw_j) = u(\bw_i, \bw_j).
\end{aligned}
\]
We will consider four dynamics (note we choose $\xi(t) = 1/2$ in these equations):
\begin{itemize}
\item The \textit{nonlinear dynamics (ND)}: we introduce $( \bbtheta^t_i )_{i \in [N], t \ge 0}$ with initialization $\bbtheta^0_i \sim \rho_0$ i.i.d.: 
\[
\begin{aligned}
\frac{\de}{\de t} \bbtheta_i^t =& - 2 \xi(t) \Big[ \nabla V ( \bbtheta_i^t) + \int \nabla_1 U( \bbtheta_i^t, \btheta) \rho_t(\de \btheta) \Big]. 
\end{aligned}
\]
Equivalently, we have the integral equation
\begin{equation}
    \bbtheta^t_i = \bbtheta^0_i + 2 \int_0^t \xi(s) \bG ( \bbtheta^s_i ; \rho_s ) \de s, \label{eq:traj_ND_noiseless_A}
    \end{equation}
    where we denoted $\bG (\btheta ; \rho ) = - \na \Psi (\btheta ; \rho ) = - \na V ( \btheta ) - \int \na_1 U (\btheta , \btheta') \rho (\de \btheta')$. Note that $\bbtheta^t_i$ is random because of its random initialization, and its law is $\rho_t$.

\item The \textit{particle dynamics (PD)}: we introduce $( \ubtheta^t_i )_{i \in [N], t \ge 0}$ with initialization $\ubtheta_i^0 = \bbtheta_i^0$:
    \[
    \begin{aligned}
    \frac{\de}{\de t} \ubtheta_i^t =& - 2 \xi(t) \Big[ \nabla V ( \ubtheta_i^t) + \frac{1}{N} \sum_{j=1}^N \nabla_1 U( \ubtheta_i^t, \ubtheta_j^t) \Big]. 
    \end{aligned}
    \]
We introduce the particle distribution $\urho^{(N)}_t =(1/N) \sum_{i = 1}^N \delta_{\ubtheta_i^t}$. In integration form, we get:
\begin{equation}
\ubtheta^t_i = \ubtheta^0_i + 2 \int_0^t \xi(s) \bG ( \ubtheta^s_i ; \urho^{(N)}_s ) \de s. 
\label{eq:traj_PD_noiseless_A}
\end{equation}
\item The \textit{gradient descent (GD)}: we introduce $(\tbtheta^k_i)_{i \in [N], k \in \N}$ with initialization $\tbtheta^0_i = \bbtheta_i^0$:
\[
\begin{aligned}
\tbtheta^{k+1}_i =& \tbtheta^{k}_i - 2 s_{k} \Big[ \nabla V ( \tbtheta_i^{k}) +  \frac{1}{N} \sum_{j=1}^N \nabla_1 U( \tbtheta_i^{k}, \tbtheta_j^{k})  \Big],
\end{aligned}
\]
where $s_k = \eps \xi(k \eps)$. We introduce the particle distribution $\trho^{(N)}_k = (1/N)\sum_{i=1}^N \delta_{\tbtheta^k_i}$. In summation form, we get: 
\begin{equation}\label{eq:traj_GD_noiseless_A}
\tbtheta^{k}_i = \tbtheta^{0}_i + 2\eps \sum_{l = 0}^{k-1} \xi (l \eps ) \bG ( \tbtheta^l_i ; {\trho}^{(N)}_l ). 
\end{equation}
The GD dynamic corresponds to the discretized particle dynamic (\ref{eq:traj_PD_noiseless_A}).    
\item The \textit{stochastic gradient descent (SGD)}: we introduce $(\btheta_i^k)_{i \in [N], k \in \N}$ with initialization $\btheta^0_i = \bbtheta_i^0$: 
\[
\begin{aligned}
\btheta^{k + 1}_i =& \btheta^{k}_i - 2 s_{k}  \bF_i (\btheta^{k}; \bz_{k+1}),
\end{aligned}
\]
where $\bF_i (\btheta^k; \bz_{k+1}) =  (y_{k+1} - \hat{y}_{k+1} ) \na_{\btheta} \sigma_\star ( \bx_{k+1}; \btheta^k_i ) $, with $\bz_k \equiv ( \bx_k , y_k)$ and $\hat{y}_{k+1} = (1/N) \sum_{j=1}^N \sigma_\star(\bx_{k+1}; \btheta^k_j)$. In summation form, we have
\begin{equation}\label{eq:traj_SGD_noiseless_A}
\btheta^k_i = \btheta^0_i + 2 \eps \sum_{l = 0}^{k-1}\xi (l \eps ) \bF_i ( \btheta^l ; \bz_{l+1} ). 
\end{equation}
\end{itemize}

Denote $\btheta^t = (\btheta_1^t, \ldots, \btheta_N^t)$, $\bbtheta^t = (\bbtheta_1^t, \ldots, \bbtheta_N^t)$, $\tbtheta^t = (\tbtheta_1^t, \ldots, \tbtheta_N^t)$, and $\ubtheta^t = (\ubtheta_1^t, \ldots, \ubtheta_N^t)$. For $t \in \R_{\ge 0}$, define $[t] = \eps \lfloor t/\eps\rfloor $. We will use the nonlinear dynamics, particle dynamics, gradient descent dynamics as interpolation dynamics
\[
\begin{aligned}
& \abs{R(\rho_{k\eps}) - R_N(\btheta^k) } \\
\leq & \underbrace{\abs{R(\rho_{k\eps}) - R_N(\bbtheta^{k\eps})}}_{\rm PDE - ND}  + \underbrace{\abs{R_N(\bbtheta^{k\eps}) - R_N(\ubtheta^{k\eps})}}_{\rm ND - PD} + \underbrace{\abs{R_N(\btheta^{k\eps}) - R_N(\tbtheta^k) } }_{\rm PD - GD} + \underbrace{\abs{R_N(\tbtheta^{k}) - R_N(\btheta^k) } }_{\rm GD - SGD}. 
\end{aligned}
\]
By Proposition \ref{prop:PDE_ND_A}, \ref{prop:ND_PD_A}, \ref{prop:PD_GD_A}, \ref{prop:GD_SGD_A} proved below, we have with probability at least $1 - e^{-z^2}$, 
\[
\begin{aligned}
\sup_{t \in [0, T]} \vert R_N(\bbtheta^t) - R(\rho_t) \vert \le& K \frac{1}{\sqrt{N}}[ \sqrt{\log(NT)} + z], \\
\sup_{t \in [0, T]} \vert R_N(\ubtheta^t) - R_N(\bbtheta^t) \vert \le& K e^{KT} \frac{1}{\sqrt N} [\sqrt{\log(NT)} + z], \\
\sup_{k \in [0, T/\eps] \cap \N} \vert R_N(\tbtheta^k) - R_N(\ubtheta^{k\eps}) \vert \le& K e^{KT} \eps, \\
\sup_{k \in [0, T/\eps] \cap \N} \vert R_N(\btheta^k) - R_N(\tbtheta^k) \vert \le& K e^{KT} \sqrt{T\eps } [ \sqrt{D + \log N} + z ]. \\
\end{aligned}
\]
Combining these inequalities gives the conclusion of Theorem \ref{thm:bound_approximating_noiseless} (A). In the following subsections, we prove all the above interpolation bounds, under the setting of Theorem \ref{thm:bound_approximating_noiseless} (A).

\subsection{Technical lemmas}

Assumptions A1 - A3 immediately implies that
\begin{lemma}\label{lem:Lip_U_V_A}
There exists a constant $K$ depending on $K_1, K_2, K_3$, such that
\[
\vert V \vert, \vert U \vert, \| \nabla V \|_2, \| \nabla U \|_2, \| \nabla^2 V \|_{\op}, \| \nabla^2 U \|_{\op} \le K.
\]
For any $\btheta = (\btheta_i)_{i = 1}^N$ and $\btheta' = (\btheta_i')_{i=1}^N$, we have
\begin{equation}\label{eqn:risk_difference_bound_by_parameter_difference_A}
\vert R(\btheta) - R(\btheta') \vert \le K \max_{i \le N} \| \btheta_i - \btheta_i' \|_2. 
\end{equation}
\end{lemma}

\begin{proof}[Proof of Lemma \ref{lem:Lip_U_V_A}]
Note we have 
\[
\begin{aligned}
V(\btheta) =& - \E_{y, \bx}[y \sigma(\bx; \btheta)], \\
U(\btheta_1, \btheta_2) =& \E_{\bx}[\sigma(\bx; \btheta_1) \sigma(\bx; \btheta_2)]. 
\end{aligned}
\]
The boundedness of $V$ and $U$ are implied by the boundedness of $\| \sigma \|_\infty$ and $\vert y \vert$ in Assumption A1. The boundedness of $ \| \nabla V \|_2, \| \nabla U \|_2, \| \nabla^2 V \|_{\op}, \| \nabla^2 U \|_{\op}$ are implied by Assumption A3. 

Finally, Eq. (\ref{eqn:risk_difference_bound_by_parameter_difference_A}) holds by noting that 
\[
\vert R_N(\btheta) - R_N(\btheta') \vert \le \frac{1}{N}\sum_{i=1}^N \vert V(\btheta_i) - V(\btheta_i') \vert + \frac{1}{N^2} \sum_{i, j = 1}^N \vert U(\btheta_i, \btheta_j) - U(\btheta_i', \btheta_j')\vert, 
\]
and by the Lipschitz property of $V$ and $U$. 
\end{proof}

Using Eq. (\ref{eq:traj_ND_noiseless_A}) and (\ref{eq:traj_PD_noiseless_A}), we immediately have 
\begin{lemma}\label{lem:Lip_rho_A}
There exists a constant $K$ such that for any time $s, t$
\[
\begin{aligned}
\| \ubtheta_i^t - \ubtheta_i^s \|_2 \le & K \vert t - s \vert, \\
\| \bbtheta_i^t - \bbtheta_i^s \|_2 \le & K \vert t - s \vert, \\
W_2(\rho_t, \rho_s) \le & K \vert t - s \vert. \\
\end{aligned}
\]
\end{lemma}

\begin{proof}[Proof of Lemma \ref{lem:Lip_rho_A}] 
The first two inequalities are simply implied by the boundedness of $\nabla V$ and $\nabla_1 U$, and Eq. (\ref{eq:traj_ND_noiseless_A}) and (\ref{eq:traj_PD_noiseless_A}). The third inequality is simply implied by
\[
W_2 ( \rho_t , \rho_s ) \leq ( \E[ \| \bbtheta_i^t- \bbtheta_i^s\|_2^2])^{1/2}.
\]
\end{proof}

\subsection{Bound between PDE and nonlinear dynamics \label{sec:PDE_ND_A}}

\begin{proposition}[PDE-ND]\label{prop:PDE_ND_A}
There exists a constant $K$ depending only on the $K_i$, $i=1,2,3$, such that with probability at least $1 - e^{-z^2}$, we have 
\[
\sup_{t \in [0, T]} \vert R_N(\bbtheta^t) - R(\rho_t) \vert \le K\frac{1}{\sqrt{N}}[ \sqrt{\log(NT)} + z].
\]
\end{proposition}

\begin{proof}[Proof of Proposition \ref{prop:PDE_ND_A}]
We decompose the difference into the following two terms 
\[
\begin{aligned}
\vert R_N(\bbtheta^t) - R(\rho_t) \vert \le  \underbrace{\vert R_N(\bbtheta^t) - \E R_N( \bbtheta^t) \vert}_{\rm I} + \underbrace{\vert \E R_N(\bbtheta^t) - R(\rho_t) \vert}_{\rm II}. 
\end{aligned}
\]
where the expectation is taken with respect to $\bbtheta_i^0 \sim \rho_0$. The result holds simply by combining Lemma \ref{lem:error_II_bound_A} and Lemma \ref{lem:error_I_bound_A}.
\end{proof}

\begin{lemma}[Term ${\rm II}$ bound]\label{lem:error_II_bound_A}
We have 
\[
\vert \E R_N( \bbtheta^t) - R(\rho_t) \vert \le K/N.
\]
\end{lemma}
\begin{proof}[Proof of Lemma \ref{lem:error_II_bound_A}]
The bound holds simply by observing that
\[
\vert \E R_N( \bbtheta^t) - R(\rho_t) \vert = \frac{1}{N} \Big \vert \int U(\btheta, \btheta) \rho_t (\de \btheta) - \int U(\btheta_1, \btheta_2) \rho_t(\de \btheta_1) \rho_t(\de \btheta_2) \Big\vert  \le K/N.
\]
\end{proof}

\begin{lemma}[Term ${\rm I}$ bound]\label{lem:error_I_bound_A}
There exists a constant $K$, such that 
\[
\P \Big( \sup_{t \in [0, T]} \vert R_N( \bbtheta^t ) - \E R_N( \bbtheta^t ) \vert \le K [\sqrt{\log(NT) } + z] / \sqrt{N} \Big) \ge 1 - e^{- z^2}. 
\]
\end{lemma}

\begin{proof}[Proof of Lemma \ref{lem:error_I_bound_A}]
Let $\btheta = (\btheta_1, \ldots, \btheta_i, \ldots, \btheta_N)$ and $\btheta' = (\btheta_1, \ldots, \btheta_i', \ldots \btheta_N)$ be two configurations that differ only in the $i$'th variable. Then 
\begin{equation}\label{eqn:RN_bounded_Lip_difference_A}
\begin{aligned}
&\vert R_N(\btheta) - R_N(\btheta') \vert\\
\le& \frac{2}{N} \vert V(\btheta_i) - V(\btheta_i') \vert  + \frac{1}{N^2} \vert U(\btheta_i, \btheta_i) - U(\btheta_i', \btheta_i') \vert + \frac{2}{N^2} \sum_{j \in [N], j \neq i}\vert U(\btheta_i, \btheta_j) - U(\btheta_i', \btheta_j) \vert \\
\le& \frac{K}{N}. 
\end{aligned}
\end{equation}
Applying McDiarmid's inequality, we have 
\[
\P\Big( \vert R_N( \bbtheta^t) - \E R_N( \bbtheta^t ) \vert \ge \delta \Big) \le \exp\{ - N \delta^2 / K \}. 
\]
By Lemma \ref{lem:Lip_rho_A} and \ref{lem:Lip_U_V_A}, we have 
\[
\Big \vert \vert R_N( \bbtheta^t ) - \E R_N( \bbtheta^t ) \vert - \vert R_N( \bbtheta^s ) - \E R_N( \bbtheta^s ) \vert \Big \vert \le K \vert s - t \vert. 
\]
Hence taking the union bound over $s \in \eta \{0, 1, \ldots, \lfloor T / \eta \rfloor\}$ and bounding the difference between time in the interval and grid, we have
\[
\P\Big( \sup_{t \in [0, T]} \vert R_N( \bbtheta^t) - \E R_N( \bbtheta^t) \vert \ge  \delta + K\eta \Big) \le (T / \eta) \exp\{ - N \delta^2 / K \}. 
\]
Now taking $\eta = 1/\sqrt N$ and $\delta = K [\sqrt{\log(NT) } + z]/\sqrt{N}$, we get the desired result. 
\end{proof}

\subsection{Bound between nonlinear dynamics and particle dynamics}

\begin{proposition}[ND-PD]\label{prop:ND_PD_A}
There exists a constant $K$, such that with probability at least $1 - e^{-z^2}$, we have 
\begin{align}
 \sup_{t \in [0, T]} \max_{i \in [N]} \| \ubtheta_i^t - \bbtheta_i^t \|_2 \le& K e^{KT} \frac{1}{\sqrt N} [\sqrt{\log(NT)} + z], \label{eqn:particle_population_perturbation_bound_A}\\
\sup_{t \in [0, T]} \vert R_N(\btheta^t) - R_N( \bbtheta^t ) \vert \le& K e^{KT} \frac{1}{\sqrt N}  [ \sqrt{\log(NT)} + z].\label{eqn:risk_particle_population_perturbation_bound_A}
\end{align}
\end{proposition}

\begin{proof}[Proof of Proposition \ref{prop:ND_PD_A}]
Note we have
\begin{equation}\label{eqn:difference_dynamics_A}
\begin{aligned}
\frac{1}{2}\frac{\de}{\de t} \| \ubtheta_i^t - \bbtheta_i^t \|_2^2 =& \<\ubtheta_i^t - \bbtheta_i^t, \nabla V (\bbtheta_i^t) - \nabla V(\ubtheta_i^t) \> + \Big\<\ubtheta_i^t - \bbtheta_i^t,  \frac{1}{N}\sum_{j=1}^N \nabla_1 U( \bbtheta_i^t, \bbtheta_j^t) - \nabla_1 U(\ubtheta_i^t, \ubtheta_j^t) \Big\>\\ 
& - \frac{1}{N}\<\ubtheta_i^t - \bbtheta_i^t, \nabla_1 U( \bbtheta_i^t, \bbtheta_i^t) - \int \nabla_1 U(\bbtheta_i^t, \btheta) \rho_t(\de \btheta) \>\\
&- \Big \<\ubtheta_i^t - \bbtheta_i^t, \frac{1}{N}\sum_{j \neq i}  \nabla_1 U( \bbtheta_i^t, \bbtheta_j^t) - \int \nabla_1 U(\bbtheta_i^t, \btheta) \rho_t(\de \btheta)\Big\>\\
\le& K \| \ubtheta_i^t - \bbtheta_i^t\|_2\cdot \max_{j \in [N]} \| \ubtheta_j^t - \bbtheta_j^t\|_2 + \| \ubtheta_i^t - \bbtheta_i^t\|_2 (K / N + I_i^t),
\end{aligned}
\end{equation}
where
\[
I_i^t \equiv \Big\| \frac{1}{N}\sum_{j \neq i} \Big[ \nabla_1 U( \bbtheta_i^t, \bbtheta_j^t) - \int \nabla_1 U(\bbtheta_i^t, \btheta) \rho_t(\de \btheta) \Big] \Big\|_2. 
\]
We would like to prove a uniform bound for $I_i^t$ for $i \in [N]$ and $t \in [0, T]$. 
\begin{lemma}\label{lem:concentration_of_I_A}
There exists a constant $K$, such that
\[
\P\Big( \sup_{t \in [0, T]} \max_{i \in[N]} I_i^t  \le K [\sqrt{\log (N T )} + z] / \sqrt{N} \Big)  \ge 1 - e^{-z^2}. 
\]
\end{lemma}

\begin{proof}[Proof of Lemma \ref{lem:concentration_of_I_A}]
Define $\bX_i^t = \nabla_1 U( \bbtheta_i^t, \bbtheta_j^t) - \int \nabla_1 U(\bbtheta_i^t, \btheta) \rho_t(\de \btheta)$.  Note we have $\E[\bX_i^t \vert \bbtheta_i^t] = 0$ (where expectation is taken with respect to $\bbtheta_j^0 \sim \rho_0$ for $j \neq i$), and $\| \bX_i^t \|_2 \le 2 K$ (by assumption that $\| \nabla U \|_2 \le K$). By Lemma \ref{lem:bounded_difference_martingale}, we have for any fixed $i \in [N]$ and $t \in [0, T]$, 
\[
\P\Big( I_i^t \ge K (\sqrt {1/N} + \delta) \Big) = \E\Big[ \P\Big( I_i^t \ge K (\sqrt {1/N} + \delta) \vert \bbtheta_i^t \Big)\Big] \le \exp\{ - N \delta^2  \}. 
\]
By Lemma \ref{lem:Lip_rho_A}, there exists $K$ such that, for any $0\le t, s \le T$ and $i \in [N]$, we have  
\[
\begin{aligned}
\vert I_i^t - I_i^s \vert \le K \vert t - s \vert. 
\end{aligned}
\]
Taking the union bound over $i \in [N]$ and $s \in \eta \{ 0, 1, \ldots, \lfloor T / \eta \rfloor \}$ and bounding time in the interval and the grid, we have
\[
\P\Big( \sup_{t \in [0, T]} \max_{i \in[N]} I_i^t \ge K (\sqrt {1/N} + \delta) + K \eta \Big)  \le (N T / \eta) \exp\{ - N \delta^2 \}. 
\]
Taking $\eta = \sqrt{1 / N}$, and $\delta = K [\sqrt{\log(NT) } + z]/ \sqrt{N}$, we get the desired result. 
\end{proof}

Let $\delta(N, T, z) = K [\sqrt{\log(NT) } + z]/\sqrt{N}$, and define
\[
\Delta(t) = \sup_{s \in [t]} \max_{i \in [N]} \| \ubtheta_i^s - \bbtheta_i^s \|_2. 
\]
We condition on the good event in Lemma \ref{lem:concentration_of_I_A} to happen. By Eq. (\ref{eqn:difference_dynamics_A}), we have 
\[
\begin{aligned}
\frac{\de\Delta}{\de t}(t) \le K \cdot \Delta(t)+ \delta(N, T, z),
\end{aligned}
\]
and by Gronwall's inequality, we obtain
\[
\Delta(T) \le K e^{KT} \delta(N, T, z). 
\]
By Eq. (\ref{eqn:risk_difference_bound_by_parameter_difference_A}), this proves Eq. (\ref{eqn:particle_population_perturbation_bound_A}) and (\ref{eqn:risk_particle_population_perturbation_bound_A}) hold with probability at least $1 - e^{- z^2}$. 
\end{proof}

\subsection{Bound between particle dynamics and GD \label{sec:PD_GD_A}}

\begin{proposition}[PD-GD]\label{prop:PD_GD_A}
There exists a constant K such that:
\[
\begin{aligned}
\sup_{k \in [0,t/\eps] \cap \N } \max_{i\leq N} \| \ubtheta^{k\eps}_i - \tbtheta^{k}_i \|_2 \le& K e^{KT} \eps,\\
\sup_{k \in [0,T/\eps] \cap \N } \vert R_N (\ubtheta^{k\eps}) - R_N ( \tbtheta^k ) \vert \leq& K e^{KT} \eps. 
\end{aligned}
\]
\end{proposition}
\begin{proof}[Proof of Proposition \ref{prop:PD_GD_A}]
By Lemma \ref{lem:Lip_rho_A}, we have
\[
\begin{aligned}
\| \ubtheta_i^t - \ubtheta_i^s \|_2 \le& K \vert t - s \vert,\\
W_2(\urho^{(N)}_{t}, \urho^{(N)}_{s}) \le& K \vert t - s \vert.
\end{aligned}
\]
For $k \in \N$ and $t = k \eps$, we have
\[
\begin{aligned}
\| \ubtheta_i^{t} - \tbtheta_i^k \|_2 \leq &  \int_0^t  \| \bG ( \ubtheta^s_i ; \urho^{(N)}_s )  - \bG ( \tbtheta^{[s]/\eps}_i ; \Tilde \rho^{(N)}_{[s]/\eps} ) \|_2 \de s \\
 \leq &  \int_0^t \| \bG ( \ubtheta^s_i ; \urho^{(N)}_s )  - \bG ( \ubtheta^{[s]/\eps}_i ; \urho^{(N)}_{[s]/\eps} ) \|_2 \de s +  \int_0^t \| \bG ( \ubtheta^{[s]}_i ; \urho^{(N)}_{[s]} )  - \bG ( \tbtheta^{[s]/\eps}_i ; \trho^{(N)}_{[s]/\eps} ) \|_2 \de s \\
 \leq & K t \eps  +  K \int_0^t \max_{i \in [N]} \| \ubtheta^{[s]}_i - \tbtheta^{[s]/\eps}_i \|_2 \de s. 
\end{aligned}
\]
Denoting $\Delta (t ) \equiv \sup_{k \in [0,t/\eps] \cap \N } \max_{i\leq N} \| \ubtheta^{k\eps}_i - \tbtheta^{k}_i \|_2$. We get the equation
\[
\Delta (t ) \leq K \int_0^t \Delta (s) \de s + K t \eps = K \int_0^t [\Delta(s) + \eps] \de s. 
\]
Applying Gronwall's lemma, we get:
\[
 \Delta (T) \leq K e^{KT} \eps.
\]
Using Eq. (\ref{eqn:risk_difference_bound_by_parameter_difference_A}) concludes the proof.
\end{proof}

\subsection{Bound between GD and SGD \label{sec:GD_SGD_A}}

\begin{proposition}[GD-SGD]\label{prop:GD_SGD_A}
There exists a constant $K$, such that with probability at least $1 - e^{-z^2}$, we have 
\begin{align}
\sup_{k \in [0, T/\eps] \cap \N} \max_{i \in [N]} \| \tbtheta_i^k - \btheta_i^k \|_2 \le& K e^{KT} \sqrt{T\eps } [\sqrt{D + \log N} + z ],\\
\sup_{k \in [0, T/\eps] \cap \N}\vert R_N(\tbtheta^k) - R_N(\btheta^k) \vert \le& K e^{KT} \sqrt{T\eps } [\sqrt{D + \log N} + z ].
\end{align}
\end{proposition}

\begin{proof}[Proof of Proposition \ref{prop:GD_SGD_A}]
Denoting $\cF_k = \sigma ( ( \btheta^0_i )_{i \in [N]}, \bz_1 , \ldots , \bz_{k} )$ the $\sigma$-algebra generated by observations $\bz_\ell = (y_\ell,\bx_\ell)$ 
up to step $k$, we get:
\[
\E [ \bF_i ( \btheta^k ; \bz_{k+1} ) \vert \cF_k ] = - \na V (\btheta^k_i ) - \frac{1}{N} \sum_{j = 1}^N \na_1 U ( \btheta_i^k , \btheta_j^k ) = \bG ( \btheta^k_i , \rho^{(N)}_k ),
\]
where $\rho^{(N)}_k \equiv (1/N) \sum_{i \in [N]} \delta_{\btheta^k_i}$ is the empirical distribution of the SGD iterates. Hence we get:
\[
\begin{aligned}
\| \btheta^k_i - \tbtheta^k_i \|_2 =& \Big\|  \eps \sum_{l = 0}^{k-1} \bF_i ( \btheta^l ; \bz_{l+1} ) - \eps \sum_{l = 0}^{k-1} \bG ( \tbtheta^l_i ; {\Tilde \rho}^{(N)}_l ) \Big \|_2 \\
\leq & \Big \| \eps \sum_{l = 0}^{k-1} \bZ_i^l \Big\|_2 +   \eps \sum_{l = 0}^{k-1} \Big\| \bG ( \btheta^l_i; \rho^{(N)}_l ) -\bG ( \tbtheta^l_i; {\Tilde \rho}^{(N)}_l ) \Big\|_2 \\
\equiv & A_i^k + B_i^k, 
\end{aligned}
\]
where we denoted $\bZ_i^l \equiv \bF_i ( \btheta^l ; \bz_{l+1} ) - \E [ \bF_i ( \btheta^l ; \bz_{l+1}) \vert \cF_{l}]$ and $A_i^k = \| \eps \sum_{l = 0}^{k-1} \bZ_i^l\|_2$. 

Note $\bF_i ( \btheta^l ; \bz_{l+1}) = (y_{l + 1} - \hat y_{l+1} ) \na_{\bw}  \sigma (\bx_{l+1}; \bw_i^l)$ for $\bz_{l+1} = (y_{l+1}, \bx_{l+1})$. Since we assumed in A2 that $\nabla_\bw \sigma(\bx; \bw)$ is $K$-sub-Gaussian, and since $y_{l+1}$ and $\hat y_{l+1}$ are  $K$ bounded, we have that $\bZ_i^l $ is $K$-sub-Gaussian (the product of a bounded random variable and a sub-Gaussian random variable is sub-Gaussian). We can therefore apply Azuma-Hoeffding inequality (Lemma \ref{lem:Azuma}) and get:
\[
\P \Big( \max_{k \in [0,T/\eps ] \cap \N} A_i^k \geq K \sqrt{T\eps } ( \sqrt{D} + z )  \Big) \leq e^{- z^2}. 
\]
Taking the union bound over $i \in [N]$, we get:
\begin{equation}\label{eqn:GD_SGD_bad_event_A}
\P \Big( \max_{i \in [N]} \max_{k \in [0,T/\eps ] \cap \N} A_i^k \geq K \sqrt{T\eps } ( \sqrt{D + \log N} + z )  \Big) \leq e^{- z^2}. 
\end{equation}
Introducing $\Delta (t ) \equiv \sup_{k \in [0,t/\eps] \cap \N } \max_{i \in [N]} \| \btheta^{k}_i - \tbtheta^{k}_i \|_2$, the $B_i^k$ terms can be bounded by:
\[
\begin{aligned}
B^k_i \le K \int_0^{k \eps} \| \bG ( \btheta^{[s] / \eps}_i; \rho^{(N)}_{[s] / \eps} ) -\bG ( \tbtheta^{[s] / \eps}_i ; {\Tilde \rho}^{(N)}_{[s] / \eps} ) \|_2 \de s \leq  K \int_0^{k \eps} \Delta (s) \de s. 
\end{aligned}
\]
Assuming the bad events in Eq. (\ref{eqn:GD_SGD_bad_event_A}) does not happen, we have 
\[
\Delta(t) \le K \int_0^t \Delta (s) \de s + K \sqrt{T\eps } ( \sqrt{D + \log N} + z ). 
\]
Applying Gronwall's inequality and applying Eq. (\ref{eqn:risk_difference_bound_by_parameter_difference_A}) concludes the proof.
\end{proof}

\section{Proof of Theorem \ref{thm:bound_approximating_noiseless} part (B)}\label{sec:proof_1B}

The difference in the proof of part (B) with the proof of part (A) comes from the fact that the functions $V$ and $U$ are not bounded and Lipschitz anymore, and that $\hf ( \bx ; \btheta )$ is not bounded by a constant. However, we show that when starting from an initial distribution $\rho_0$ with compact support in the variable $a$, the support of $\rho_t$ in the variable $a$ remains bounded uniformly on the interval $[0,T]$ by a constant that only depends on the $K_i$, $i=1, 2, 3, 4$, and $T$. 

For $\btheta = (a, \bw)$ and $\btheta' = (a', \bw')$, remember we have
\[
\begin{aligned}
\sigma_\star(\bx; \btheta) =& a \sigma(\bx; \bw),\\
v(\bw) =& - \E_{y, \bx}[y \sigma (\bx; \bw)],\\
u(\bw, \bw') =& \E_{\bx} [\sigma(\bx; \bw) \sigma(\bx; \bw')],\\
V(\btheta) =& a \cdot v (\bw),\\
U (\btheta , \btheta ' ) =& a a' \cdot u ( \bw , \bw' ),
\end{aligned}
\]
hence we have 
\[
\begin{aligned}
\na_\btheta V(\btheta) =& (v(\bw), a \na_\bw v(\bw)), \\
\na_\btheta U(\btheta, \btheta') =& (a' \cdot u(\bw, \bw'), a a' \cdot \na_\bw u(\bw, \bw')).
\end{aligned}
\]
Throughout this section, the assumptions A1 - A4 are understood to hold. For the sake of simplicity we will write the proof under the following
restriction: 
\begin{itemize}
\item[R1.] The step size function $\xi(t) \equiv 1/2$. 
\end{itemize}
The proof for a general function $\xi(t)$ is obtained by a straightforward adaptation.

We define the four dynamics with the same definitions as at the beginning of Section \ref{sec:proof_1A}. We copy them here for reader's convenience. 
\begin{itemize}
\item The \textit{nonlinear dynamics (ND)}: $( \bbtheta^t_i )_{i \in [N], t \ge 0}$ with initialization $\bbtheta^0_i \sim \rho_0$ i.i.d.: 
\begin{equation}
    \bbtheta^t_i = \bbtheta^0_i + 2 \int_0^t \xi(s) \bG ( \bbtheta^s_i ; \rho_s ) \de s, \label{eq:traj_ND_noiseless_B}
    \end{equation}
    where we denoted $\bG (\btheta ; \rho ) = - \na \Psi (\btheta ; \rho ) = - \na V ( \btheta ) - \int \na_1 U (\btheta , \btheta') \rho (\de \btheta')$. 

\item The \textit{particle dynamics (PD)}: $( \ubtheta^t_i )_{i \in [N], t \ge 0}$ with initialization $\ubtheta_i^0 = \bbtheta_i^0$:
\begin{equation}
\ubtheta^t_i = \ubtheta^0_i + 2 \int_0^t \xi(s) \bG ( \ubtheta^s_i ; \urho^{(N)}_s ) \de s, 
\label{eq:traj_PD_noiseless_B}
\end{equation}
where  $\urho^{(N)}_t =(1/N) \sum_{i = 1}^N \delta_{\ubtheta_i^t}$. 
\item The \textit{gradient descent (GD)}: $(\tbtheta^k_i)_{i \in [N], k \in \N}$ with initialization $\tbtheta^0_i = \bbtheta_i^0$:
\begin{equation}\label{eq:traj_GD_noiseless_B}
\tbtheta^{k}_i = \tbtheta^{0}_i + 2\eps \sum_{l = 0}^{k-1} \xi (l \eps ) \bG ( \tbtheta^l_i ; {\trho}^{(N)}_l ). 
\end{equation}
where $s_k = \eps \xi(k \eps)$ and $\trho^{(N)}_k = (1/N)\sum_{i=1}^N \delta_{\tbtheta^k_i}$. 
\item The \textit{stochastic gradient descent (SGD)}: $(\btheta_i^k)_{i \in [N], k \in \N}$ with initialization $\btheta^0_i = \bbtheta_i^0$: 
\begin{equation}\label{eq:traj_SGD_noiseless_B}
\btheta^k_i = \btheta^0_i + 2 \eps \sum_{l = 0}^{k-1}\xi (l \eps ) \bF_i ( \btheta^l ; \bz_{l+1} ),
\end{equation}
where $\bF_i (\btheta^k; \bz_{k+1}) =  (y_{k+1} - \hat{y}_{k+1} ) \na_{\btheta} \sigma_\star ( \bx_{k+1}; \btheta^k_i ) $, with $\bz_k \equiv ( \bx_k , y_k)$ and $\hat{y}_{k+1} = (1/N) \sum_{j=1}^N a_j^k \sigma(\bx_{k+1}; \bw^k_j)$. 
\end{itemize}

We have the decomposition 
\[
\begin{aligned}
& \abs{R(\rho_{k\eps}) - R_N(\btheta^k) } \\
\leq & \underbrace{\abs{R(\rho_{k\eps}) - R_N(\bbtheta^{k\eps})}}_{\rm PDE - ND}  + \underbrace{\abs{R_N(\bbtheta^{k\eps}) - R_N(\ubtheta^{k\eps})}}_{\rm ND - PD} + \underbrace{\abs{R_N(\btheta^{k\eps}) - R_N(\tbtheta^k) } }_{\rm PD - GD} + \underbrace{\abs{R_N(\tbtheta^{k}) - R_N(\btheta^k) } }_{\rm GD - SGD}. 
\end{aligned}
\]
By Proposition \ref{prop:PDE_ND_B}, \ref{prop:ND_PD_B}, \ref{prop:PD_GD_B}, \ref{prop:GD_SGD_B}, there exists constants $K$ and $K_0$, such that if we take $\eps \le 1/ [K_0 (D + \log N + z^2) e^{K_0 (1 + T)^3}]$, with probability at least $1 - e^{-z^2}$, we have
\[
\begin{aligned}
\sup_{t \in [0, T]} \vert R_N(\bbtheta^t) - R(\rho_t) \vert \le&  K(1 + T)^4 \frac{1}{\sqrt{N}}[ \sqrt{\log(N T)} + z], \\
\sup_{t \in [0, T]} \vert R_N(\ubtheta^t) - R_N(\bbtheta^t) \vert \le& K e^{K(1 + T)^3} \frac{1}{\sqrt N} [\sqrt{\log(NT)} + z], \\
\sup_{k \in [0, T/\eps] \cap \N} \vert R_N(\tbtheta^k) - R_N(\ubtheta^{k\eps}) \vert \le& K e^{K (1 + T)^3} \eps, \\
\sup_{k \in [0, T/\eps] \cap \N} \vert R_N(\btheta^k) - R_N(\tbtheta^k) \vert \le& K e^{K (1 + T)^3} \sqrt{\eps } [\sqrt{D + \log N} + z ]. \\
\end{aligned}
\]
Combining these inequalities, and noting that $K e^{K(1 + T)^3} \le K' e^{K'T^3}$ for some $K' \ge K$, give the conclusion of Theorem \ref{thm:bound_approximating_noiseless} (B). In the following subsections, we prove all the above interpolation bounds, under the setting of Theorem \ref{thm:bound_approximating_noiseless} (B).

\subsection{Technical lemmas}

\begin{lemma}\label{lem:bounded_support_of_rhot_B}
There exists a constant K depending only on the $K_i$, $i=1,2,3,4$, such that
\[
\begin{aligned}
\supp(\rho_t) \subseteq& [-K(1 + t), K(1 + t)] \times \R^{D - 1},\\
\vert \bara_i^t\vert \le& K(1 + t), \\
\vert \undera_i^t\vert \le& K(1 + t). 
\end{aligned}
\]
\end{lemma}

\begin{proof}[Proof of Lemma \ref{lem:bounded_support_of_rhot_B}]~

\noindent {\bf Step 1. }
Let $\bbtheta_i^t = (\bara_i^t, \bbw_i^t)$, and $\hat y(\bx; \rho_t) = \int a \sigma(\bx; \bw) \rho_t( \de \btheta)$. Note that along the PDE, we have 
\[
\frac{\de}{\de t} R(\rho_t) = - \int \| \nabla \Psi(\btheta; \rho_t) \|_2^2 \rho_t(\de \btheta) \le 0. 
\]
Hence we have (note $\vert y\vert \le K$, $\vert \sigma \vert \le K$, and $\supp(\rho_0) \subseteq [-K, K] \times \R^{D - 1}$)
\[
R(\rho_t) = \E_{y, \bx}[(y - \hat y(\bx; \rho_t))^2] \le R(\rho_0) = \E_{y, \bx}\Big[ \Big(y - \int a \sigma(\bx; \bw) \rho_0(\de \btheta)\Big)^2\Big] \le K. 
\]

The nonlinear dynamics for $\bara_i^t$ gives 
\[
\frac{\de }{\de t}\bara_i^t =  \E_{y, \bx}[(y - \hat y(\bx; \rho_t)) \sigma(\bx; \bbw_i^t)],
\]
which gives 
\[
\begin{aligned}
\Big\vert \frac{\de }{\de t}  \bara_i^t \Big\vert \le  \{ \E_{y, \bx}[(y - \hat y(\bx; \rho_t))^2] \E_{y, \bx}[\sigma(\bx; \bbw_i^t)^2] \}^{1/2} \le K. 
\end{aligned}
\]
Hence, we have 
\[
\vert \bara_i^t\vert \le \vert \bara_i^0 \vert + K t \le K(1 + t). 
\]
Note $(\bara_i^t, \bbw_i^t) \sim \rho_t$, hence we have $\supp(\rho_t) \subseteq [-K(1 + t), K(1 + t)] \times \R^{D - 1}$. 

\noindent
{\bf Step 2. } Denote $\ubtheta_i^t = (\undera_i^t, \ubw_i^t)$, $\urho_t^{(N)} = (1/N) \sum_{i=1}^N \delta_{\ubtheta_i^t}$, and denote $\underline y(\bx; \ubtheta^t) = (1/N)\sum_{i \in [N]} \undera_i^t \sigma(\bx; \ubw_i^t)$. Note along the PDE, we have 
\[
\frac{\de}{\de t} R_N(\ubtheta^t) = - \int \| \nabla \Psi(\btheta; \urho_t^{(N)}) \|_2^2 \urho_t^{(N)}(\de \btheta) \le 0. 
\]
Hence we have (note $\vert y\vert \le K$, $\vert \sigma \vert \le K$, and $\vert \undera_i^0 \vert \le K$)
\[
R_N(\ubtheta^t) = \E_{y, \bx}[(y - \underline y(\bx; \urho_t^{(N)}))^2] \le R_N(\ubtheta^0) = \E_{y, \bx}\Big[ \Big(y - \int a \sigma(\bx; \bw) \urho_0^{(N)}(\de \btheta)\Big)^2\Big] \le K. 
\]

The nonlinear dynamics for $\undera_i^t$ gives 
\[
\frac{\de }{\de t}\undera_i^t =  \E_{y, \bx}[(y - \underline y(\bx; \ubtheta^t)) \sigma(\bx; \ubw_i^t)],
\]
which gives 
\[
\begin{aligned}
\Big\vert \frac{\de }{\de t} \undera_i^t \Big\vert \le  \{ \E_{y, \bx}[(y - \underline y(\bx; \ubtheta^t))^2] \E_{y, \bx}[\sigma(\bx; \ubw_i^t)^2] \}^{1/2} \le K. 
\end{aligned}
\]
Hence, we have 
\[
\vert \undera_i^t\vert \le \vert \undera_i^0 \vert + K t \le K(1 + t). 
\]
This proves the lemma. 
\end{proof}

\begin{lemma}[Boundness and Lipschitzness]\label{lem:bound_lip_traject_B}
Denoting $\btheta = (a, \bw)$, $\btheta_1 = (a_1, \bw_1)$ and $\btheta_2 = (a_2, \bw_2)$. We have
\[
\begin{aligned}
\vert V(\btheta) \vert, \| \na V (\btheta ) \|_2 \leq& K( 1 + \vert a \vert),\\
\vert V(\btheta_1) - V(\btheta_2) \vert, \| \na V (\btheta_1) - \na V (\btheta_2) \|_2 \leq& K \cdot [ 1 + \vert a_1\vert \wedge \vert a_2\vert ] \cdot  \| \btheta_1 - \btheta_2 \|_2, \\
\vert U(\btheta, \btheta') \vert, \| \na_1 U(\btheta, \btheta' ) \|_2 \le& K  (1 + \vert a \vert) (1 + \vert a' \vert), \\
\vert U(\btheta_1, \btheta) - U(\btheta_2, \btheta) \vert, \| \na_{(1, 2)} U(\btheta_1, \btheta ) - \na_{(1, 2)} U(\btheta_2, \btheta ) \|_2 \leq& K  (1 + \vert a\vert ) \cdot [ 1 + \vert a_1\vert \wedge \vert a_2\vert ]\cdot  \| \btheta_1 - \btheta_2 \|_2, \\
\vert R_N(\btheta) - R_N(\btheta') \vert \le& K \max_{i \in [N]} (1 + \vert a_i\vert \vee \vert a_i' \vert )^2 \cdot \max_{j \in [N]} \| \btheta_j - \btheta_j' \|_2.
\end{aligned}
\]
\end{lemma}

\begin{proof}[Proof of Lemma \ref{lem:bound_lip_traject_B}]
We have 
\[
\begin{aligned}
\vert V(\btheta) \vert =& \vert a v(\bw) \vert \le K \vert a \vert,\\
\| \na V (\btheta ) \|_2 =& \|(v(\bw), a \nabla_\bw v(\bw)) \|_2 \le K(1 + \vert a \vert), \\
\end{aligned}
\]
and (assuming $\vert a_1 \vert \ge \vert a_2 \vert$)
\[
\begin{aligned}
\vert V(\btheta_1) - V(\btheta_2) \vert = \vert a_1 v(\bw_1) - a_2 v (\bw_2) \vert \le K [\vert a_1 - a_2 \vert + \vert a_2 \vert \| \bw_1 - \bw_2 \|_2] \le K [1 + \vert a_2 \vert ] \| \btheta_1 - \btheta_2 \|_2, 
\end{aligned}
\]
and
\[
\begin{aligned}
\| \na V (\btheta_1) - \na V (\btheta_2) \|_2=& \|(v(\bw_1) - v(\bw_2), a_1 \nabla v(\bw_1) - a_2 \nabla v(\bw_2)) \|_2\\
\le &K \| \bw_1 - \bw_2 \|_2 +K \| a_1 \nabla v(\bw_1) - a_2 \nabla v(\bw_1) \|_2 + \| a_2 [\nabla v(\bw_1) - \nabla v(\bw_2)] \|_2\\
\le& K [\| \bw_1 - \bw_2 \|_2 + \vert a_1 - a_2\vert] + K \vert a_2 \vert \| \bw_1 - \bw_2 \|_2 \\
\le& K (1 + \vert a_2 \vert) \| \btheta_1 - \btheta_2 \|_2,
\end{aligned}
\]
and
\[
\begin{aligned}
\vert U(\btheta, \btheta') \vert =& \vert a a' u(\bw, \bw') \vert \le K \vert a \vert \vert a'\vert,
\end{aligned}
\]
and
\[
\begin{aligned}
 \| \na U(\btheta, \btheta' ) \|_2 =& \| (a' u(\bw, \bw'), aa' \cdot \nabla_1 u(\bw, \bw')) \|_2 \le K \vert a' \vert (1 + \vert a \vert),
\end{aligned}
\]
and (assuming $\vert a_1 \vert \ge \vert a_2 \vert$)
\[
\begin{aligned}
\vert U(\btheta_1, \btheta) - U(\btheta_2, \btheta) \vert =& \vert a_1 a u(\bw_1, \bw) - a_2 a u(\bw_2, \bw) \vert \\
\le& K [\vert a_1 - a_2 \vert \vert a \vert + \vert a_2\vert \vert a \vert \| \bw_1 - \bw_2 \|_2 ] \\
\le& K (1 + \vert a_2 \vert ) \vert a \vert \| \btheta_1 - \btheta_2 \|_2,
\end{aligned}
\]
and
\[
\begin{aligned}
\| \na_1 U(\btheta_1, \btheta ) - \na_1 U(\btheta_2, \btheta ) \|_2 =& \| (a u(\bw_1, \bw) - a u(\bw_2, \bw), a_1 a \na_1 u(\bw_1, \bw) - a_2 a \na_1 u(\bw_2, \bw)) \|_2\\
\le& \vert a \vert \| \bw_1 - \bw_2 \|_2 + K \vert a \vert \vert a_1 - a_2 \vert + K \vert a \vert \vert a_2 \vert \| \bw_1 - \bw_2 \|_2 \\
\le& K \vert a \vert (1 + \vert a_2 \vert) \| \btheta_1 - \btheta_2 \|_2, 
\end{aligned}
\]
and 
\[
\begin{aligned}
\| \na_2 U(\btheta_1, \btheta ) - \na_2 U(\btheta_2, \btheta ) \|_2 =& \| (a_1 u(\bw_1, \bw) - a_2 u(\bw_2, \bw), a_1 a \na_2 u(\bw_1, \bw) - a_2 a \na_2 u(\bw_2, \bw)) \|_2\\
\le& K \vert a_1 - a_2 \vert + K  \vert a_2 \vert \| \bw_1 - \bw_2 \|_2 + K \vert a \vert \vert a_1 - a_2\vert + K \vert a \vert \vert a_2 \vert \| \bw_1 - \bw_2 \|_2\\
\le& K (1 + \vert a \vert) (1 + \vert a_2 \vert) \| \btheta_1 - \btheta_2 \|_2. 
\end{aligned}
\]
Finally, we have
\[
\begin{aligned}
&\vert R(\btheta) - R(\btheta') \vert \\
\le& 2 \max_{i \in [N]} \vert V(\btheta_i)  - V(\btheta_i')\vert + \max_{i, j \in [N]} \vert U(\btheta_i, \btheta_j) - U(\btheta_i', \btheta_j') \vert\\
\le& K \Big[ \max_{i \in [N]}(1 +  \vert a_i \vert \wedge \vert a_i' \vert) \| \btheta_i - \btheta_i' \|_2 +  \max_{i,j \in [N]} (1 + \vert a_i \vert \wedge \vert a_i' \vert)( \vert a_j \vert \vee \vert a_j ' \vert ) \| \btheta_i - \btheta_i' \|_2 \Big] \\
\le& K \max_{i \in [N]} (1 + \vert a_i\vert \vee \vert a_i' \vert )^2 \cdot \max_{j \in [N]} \| \btheta_j - \btheta_j' \|_2.  
\end{aligned}
\]
This concludes the proof.
\end{proof}

\begin{lemma}\label{lem:Lip_rho_B}
There exists a constant $K$ such that for any time $0 \le s < t$
\[
\begin{aligned}
\| \ubtheta_i^t - \ubtheta_i^s \|_2 \le & K (1 + s)^2 \vert t - s \vert, \\
\| \bbtheta_i^t - \bbtheta_i^s \|_2 \le & K (1 + s)^2 \vert t - s \vert, \\
W_2(\rho_t, \rho_s) \le & K (1 + s)^2 \vert t - s \vert. \\
\end{aligned}
\]
\end{lemma}

\begin{proof}[Proof of Lemma \ref{lem:Lip_rho_B}] This lemma holds by the bounds of $\nabla V$ and $\nabla_1 U$ in Lemma \ref{lem:bound_lip_traject_B} and the bounds for $\vert \bara_i^t\vert, \vert \undera_i^t \vert$ in Lemma \ref{lem:bounded_support_of_rhot_B}, and by the inequality 
\[
W_2 ( \rho_t , \rho_s ) \leq ( \E[ \| \bbtheta_i^t- \bbtheta_i^s\|_2^2])^{1/2}.
\]
\end{proof}

\subsection{Bound between PDE and nonlinear dynamics \label{sec:PDE_ND_B}}

\begin{proposition}[PDE-ND]\label{prop:PDE_ND_B}
There exists a constant $K$, such that with probability at least $1 - e^{-z^2}$, we have 
\[
\sup_{ t \in [0, T] \cap \N} \vert R_N(\bbtheta^{t}) - R(\rho_{t}) \vert \le K(1 + T)^4 \frac{1}{\sqrt{N}}[ \sqrt{\log(N T)} + z]
\]
\end{proposition}

\begin{proof}[Proof of Proposition \ref{prop:PDE_ND_B}]
We decompose the difference into the following two terms 
\[
\begin{aligned}
\vert R_N(\bbtheta^t) - R(\rho_t) \vert \le  \underbrace{\vert R_N(\bbtheta^t) - \E R_N( \bbtheta^t) \vert}_{\rm I} + \underbrace{\vert \E R_N(\bbtheta^t) - R(\rho_t) \vert}_{\rm II}. 
\end{aligned}
\]
where the expectation is taken with respect to $\bbtheta_i^0 \sim \rho_0$. The result holds simply by combining Lemma \ref{lem:error_II_bound_B} and Lemma \ref{lem:error_I_bound_B}.
\end{proof}

\begin{lemma}[Term ${\rm II}$ bound]\label{lem:error_II_bound_B}
We have 
\[
\vert \E R_N( \bbtheta^t) - R(\rho_t) \vert \le K (1 + t)^2/N.
\]
\end{lemma}
\begin{proof}[Proof of Lemma \ref{lem:error_II_bound_B}]
The bound hold simply by observing that
\[
\begin{aligned}
\vert \E R_N( \bbtheta^t) - R(\rho_t) \vert =& \frac{1}{N} \Big \vert \int a^2 u(\bw, \bw) \rho_t (\de \btheta) - \int a_1 a_2 u(\bw_1, \bw_2) \rho_t(\de \btheta_1) \rho_t(\de \btheta_2) \Big\vert  \\
\le& (K/N) \int a^2 \rho_t(\de \btheta) \le K(1 + t)^2 / N.
\end{aligned}
\]
\end{proof}

\begin{lemma}[Term ${\rm I}$ bound]\label{lem:error_I_bound_B}
There exists a constant $K$, such that 
\[
\P \Big( \sup_{t \in [0, T] } \vert R_N( \bbtheta^{t} ) - \E R_N( \bbtheta^{t} ) \vert \le K (1 + T)^4 [\sqrt{\log(NT) } + z] / \sqrt{N} \Big) \ge 1 - e^{- z^2}. 
\]
\end{lemma}

\begin{proof}[Proof of Lemma \ref{lem:error_I_bound_B}]
Let $\btheta = (\btheta_1, \ldots, \btheta_i, \ldots, \btheta_N)$ and $\btheta' = (\btheta_1, \ldots, \btheta_i', \ldots \btheta_N)$ be two configurations that differ only in the $i$'th variable. Assuming $a, a' \in [- K(1 + t), K(1 + t)]$, then 
\begin{equation}\label{eqn:RN_bounded_Lip_difference_B}
\begin{aligned}
&\vert R_N(\btheta) - R_N(\btheta') \vert\\
\le& \frac{2}{N} \vert V(\btheta_i) - V(\btheta_i') \vert  + \frac{1}{N^2} \vert U(\btheta_i, \btheta_i) - U(\btheta_i', \btheta_i') \vert + \frac{2}{N^2} \sum_{j \in [N], j \neq i}\vert U(\btheta_i, \btheta_j) - U(\btheta_i', \btheta_j) \vert \\
\le& \frac{K}{N}(1 + t)^2. 
\end{aligned}
\end{equation}
Note we have $\bara_i^t \in [- K(1 + t), K(1 + t)]$, applying McDiarmid's inequality, we have 
\[
\P\Big( \vert R_N( \bbtheta^t) - \E R_N( \bbtheta^t ) \vert \ge \delta \Big) \le \exp\{ - N \delta^2 / [K(1+t)^4] \}. 
\]
By Lemma \ref{lem:Lip_rho_B}, \ref{lem:bound_lip_traject_B} and \ref{lem:bounded_support_of_rhot_B}, for $0 \le s < t$, we have 
\[
\begin{aligned}
&\vert R_N( \bbtheta^t ) - R_N(\bbtheta^s) \vert \\
\le& K \max_{i \in [N]} (1 + \vert \bara_i^s \vert \vee \vert \bara_i^t \vert)^2 \cdot \max_{j \in [N]} \| \bbtheta_j^t - \bbtheta_j^s \|_2 \le K (1 + t)^4 \vert t - s\vert,
\end{aligned}
\]
which gives
\[
\Big \vert \vert R_N( \bbtheta^t ) - \E R_N( \bbtheta^t ) \vert - \vert R_N( \bbtheta^s ) - \E R_N( \bbtheta^s ) \vert \Big \vert \le K (1 + t)^4 \vert t - s \vert. 
\]
Hence taking union bound over $s \in \eta \{0, 1, \ldots, \lfloor T / \eta \rfloor\}$ and bounding difference between time in the interval and grid, we have
\[
\P\Big( \sup_{t \in [0, T]} \vert R_N( \bbtheta^t) - \E R_N( \bbtheta^t) \vert \ge  \delta + K (1 + T)^4 \eta \Big) \le (T / \eta) \exp\{ - N \delta^2 / [K (1 + T)^4] \}. 
\]
Now taking $\eta = 1/\sqrt N$ and $\delta = K (1 + T)^4 [\sqrt{\log(NT) } + z]/\sqrt{N}$, we get the desired inequality. 
\end{proof}

\subsection{Bound between nonlinear dynamics and particle dynamics}

\begin{proposition}[ND-PD]\label{prop:ND_PD_B}
There exists a constant $K$, such that with probability at least $1 - e^{-z^2}$, we have 
\begin{align}
\sup_{t \in [0, T]} \max_{i \in [N]} \| \ubtheta_i^t - \bbtheta_i^t \|_2 \le& K e^{K(1 + T)^3} \frac{1}{\sqrt N} [\sqrt{\log(NT)} + z], \label{eqn:particle_population_perturbation_bound_B}\\
\sup_{t \in [0, T]} \vert R_N(\ubtheta^t) - R_N(\bbtheta^t)\vert \le& K e^{K(1 + T)^3} \frac{1}{\sqrt N} [\sqrt{\log(NT)} + z].
 \label{eqn:risk_population_perturbation_bound_B}
\end{align}
\end{proposition}

\begin{proof}[Proof of Proposition \ref{prop:ND_PD_B}]

Note we have
\begin{equation}\label{eqn:difference_dynamics_B}
\begin{aligned}
\frac{1}{2}\frac{\de}{\de t} \| \ubtheta_i^t - \bbtheta_i^t \|_2^2 =& \<\ubtheta_i^t - \bbtheta_i^t, \nabla V (\bbtheta_i^t) - \nabla V(\ubtheta_i^t) \> + \Big\<\ubtheta_i^t - \bbtheta_i^t,  \frac{1}{N}\sum_{j=1}^N \nabla_1 U( \bbtheta_i^t, \bbtheta_j^t) - \nabla_1 U(\bbtheta_i^t, \ubtheta_j^t) \Big\>\\
& +  \Big\<\ubtheta_i^t - \bbtheta_i^t,  \frac{1}{N}\sum_{j=1}^N \nabla_1 U( \bbtheta_i^t, \ubtheta_j^t) - \nabla_1 U(\ubtheta_i^t, \ubtheta_j^t) \Big\> \\
& - \frac{1}{N}\<\ubtheta_i^t - \bbtheta_i^t, \nabla_1 U( \bbtheta_i^t, \bbtheta_i^t) - \int \nabla_1 U(\bbtheta_i^t, \btheta) \rho_t(\de \btheta) \>\\
&- \Big \<\ubtheta_i^t - \bbtheta_i^t, \frac{1}{N}\sum_{j \neq i}  \nabla_1 U( \bbtheta_i^t, \bbtheta_j^t) - \int \nabla_1 U(\bbtheta_i^t, \btheta) \rho_t(\de \btheta)\Big\>\\
\le& K (1 + t)^2 \| \ubtheta_i^t - \bbtheta_i^t\|_2\cdot \max_{j \in [N]} \| \ubtheta_j^t - \bbtheta_j^t\|_2 + \| \ubtheta_i^t - \bbtheta_i^t\|_2 (K (1 + t)^2 / N + I_i^t), 
\end{aligned}
\end{equation}
where
\[
I_i^t =\Big\| \frac{1}{N}\sum_{j \neq i} \Big[ \nabla_1 U( \bbtheta_i^t, \bbtheta_j^t) - \int \nabla_1 U(\bbtheta_i^t, \btheta) \rho_t(\de \btheta) \Big] \Big\|_2. 
\]
The last inequality follows by Lemma \ref{lem:bound_lip_traject_B} and \ref{lem:bounded_support_of_rhot_B}. Now we would like to prove a uniform bound for $I_i^t$ for $i \in [N]$ and $t \in [0, T]$. 

\begin{lemma}\label{lem:concentration_of_I_B}
There exists a constant $K$, such that
\[
\P\Big( \sup_{t \in [0, T]} \max_{i \in[N]} I_i^t  \le K (1 + T)^2 [\sqrt{\log (N T )} + z] / \sqrt{N} \Big)  \ge 1 - e^{-z^2}. 
\]
\end{lemma}

\begin{proof}[Proof of Lemma \ref{lem:concentration_of_I_B}]
Denote $\bX_i^t = \nabla_1 U( \bbtheta_i^t, \bbtheta_j^t) - \int \nabla_1 U(\bbtheta_i^t, \btheta) \rho_t(\de \btheta)$.  Note we have $\E[\bX_i^t \vert \bbtheta_i^t] = 0$ (where expectation is taken with respect to $\bbtheta_j^0 \sim \rho_0$ for $j \neq i$), and $\| \bX_i^t \|_2 \le 2 (1 + t)^2 K$ (by Lemma \ref{lem:bound_lip_traject_B} and \ref{lem:bounded_support_of_rhot_B}). By Lemma \ref{lem:bounded_difference_martingale}, we have for any fixed $i \in [N]$ and $t \in [0, T]$, 
\[
\P\Big( I_i^t \ge K (1 + t)^2 (\sqrt {1/N} + \delta) \Big) = \E\Big[ \P\Big( I_i^t \ge K (1 + t)^2 (\sqrt {1/N} + \delta) \vert \bbtheta_i^t \Big)\Big] \le \exp\{ - N \delta^2  \}. 
\]
By Lemma \ref{lem:Lip_rho_B}, there exists $K$ such that, for any $0\le s < t \le T$ and $i \in [N]$, we have  
\[
\begin{aligned}
\vert I_i^t - I_i^s \vert \le K (1 + t)^2 \vert t - s \vert. 
\end{aligned}
\]
Taking the union bound over $i \in [N]$ and $s \in \eta [T / \eta]$ and bounding time in the interval and the grid, we have
\[
\P\Big( \sup_{t \in [0, T]} \max_{i \in[N]} I_i^t \ge K (1 + T)^2 (\sqrt {1/N} + \delta) + K (1 + T)^2 \eta \Big)  \le (N T / \eta) \exp\{ - N \delta^2 \}. 
\]
Taking $\eta = \sqrt{1 / N}$, and $\delta = K [\sqrt{\log(NT) } + z]/ \sqrt{N}$, we get the desired result. 
\end{proof}

Denote $\delta(N, T, z) = K (1 + T)^2 [\sqrt{\log(NT) } + z]/\sqrt{N}$ and
\[
\Delta(t) = \sup_{s \in [t]} \max_{i \in [N]} \| \ubtheta_i^s - \bbtheta_i^s \|_2. 
\]
We condition on the good event in Lemma \ref{lem:concentration_of_I_B} to happen. By Eq. (\ref{eqn:difference_dynamics_B}), we have 
\[
\begin{aligned}
\Delta'(t) \le K (1 + T)^2 \cdot \Delta(t)+ \delta(N, T, z),
\end{aligned}
\]
By Gronwall's inequality, we have
\[
\Delta(T) \le K e^{K(1 + T)^3} \delta(N, T, z). 
\]
This happens with probability $1 - e^{- z^2}$. This proves Eq. (\ref{eqn:particle_population_perturbation_bound_B}). Finally, Eq. (\ref{eqn:risk_population_perturbation_bound_B}) holds by Lemma \ref{lem:bound_lip_traject_B}. 
\end{proof}

\subsection{Bound between particle dynamics and GD \label{sec:PD_GD_B}}

\begin{proposition}[PD-GD]\label{prop:PD_GD_B}
There exists constants $K$ and $K_0$ such that, letting $\eps \le 1/(K_0 e^{K_0 (1 + T)^3})$, we have for any $t \le T$, 
\[
\begin{aligned}
\sup_{k \in [0, t/\eps] \cap \N} \vert \tilda_i^k \vert \le& K (1 + t),\\
\sup_{k \in [0, t/\eps] \cap \N} \max_{i \in [N]} \| \ubtheta_i^{k \eps} - \tbtheta_i^k \|_2 \le& K e^{K (1 + T)^2 t} \eps,\\
\sup_{k \in [0, t/\eps] \cap \N} \vert R_N(\ubtheta^{k \eps}) - R_N(\tbtheta^k)\vert \le& K e^{K (1 + T)^2 t} \eps.\\
\end{aligned}
\]
\end{proposition}
\begin{proof}[Proof of Proposition \ref{prop:PD_GD_B}]
Let $\urho^{(N)}_s = (1/N) \sum_{i=1}^N \delta_{\ubtheta_i^s}$, and $\trho^{(N)}_k = (1/N) \sum_{i=1}^N \delta_{\tbtheta_i^k}$. For $k \in \N$ and $t = k \eps$, we have
\[
\begin{aligned}
\| \ubtheta_i^{t} - \tbtheta_i^k \|_2 \leq &2  \int_0^t  \| \bG ( \ubtheta^s_i ; \urho^{(N)}_s )  - \bG ( \tbtheta^{[s]/\eps}_i ; \Tilde \rho^{(N)}_{[s]/\eps} ) \|_2 \de s \\
 \leq & 2 \int_0^t \| \bG ( \ubtheta^s_i ; \urho^{(N)}_s )  - \bG ( \ubtheta^{[s]}_i ; \urho^{(N)}_{[s]} ) \|_2 \de s + 2 \int_0^t \| \bG ( \ubtheta^{[s]}_i ; \urho^{(N)}_{[s]} )  - \bG ( \tbtheta^{[s]/\eps}_i ; \trho^{(N)}_{[s]/\eps} ) \|_2 \de s. \\
\end{aligned}
\]
By Lemma \ref{lem:Lip_rho_B} and \ref{lem:bound_lip_traject_B}, for $0 \le s \le t$, we have 
\[
\begin{aligned}
&  \| \bG ( \ubtheta^s_i ; \urho^{(N)}_s )  - \bG ( \ubtheta^{[s]}_i ; \urho^{(N)}_{[s]} ) \|_2 \\
\le& \| \nabla V(\ubtheta^s_i) - \nabla V(\ubtheta^{[s]}_i) \|_2 + \sup_{j \in [N]} \| \nabla_1 U(\ubtheta_i^s, \ubtheta_j^s) - \nabla_1 U(\ubtheta_i^s, \ubtheta_j^{[s]}) \|_2 \\
&+ \sup_{j \in [N]} \| \nabla_1 U(\ubtheta_i^s, \ubtheta_j^{[s]}) - \nabla_1 U(\ubtheta_i^{[s]}, \ubtheta_j^{[s]}) \|_2\\
\le& K [1 + \vert \undera_i^s \vert ] \| \ubtheta_i^s - \ubtheta_i^{[s]} \|_2 + \sup_{j \in [N]} K (1 + \vert \undera_i^s \vert) (1 + \vert \undera_j^{[s]} \vert )[  \| \ubtheta_i^s - \ubtheta_i^{[s]} \|_2 +  \| \ubtheta_j^s - \ubtheta_j^{[s]} \|_2]\\
\le& K (1 + t)^4 (s - [s]) \le K (1 + t)^4 \eps,
\end{aligned}
\]
and for $ u = k \eps \le t$,
\[
\begin{aligned}
&  \| \bG ( \ubtheta^u_i ; \urho^{(N)}_u )  - \bG ( \tbtheta^k_i ; \trho^{(N)}_k ) \|_2 \\
\le& \| \nabla V(\ubtheta^u_i) - \nabla V(\tbtheta^k_i) \|_2 + \sup_{j \in [N]} \| \nabla_1 U(\ubtheta_i^u, \ubtheta_j^u) - \nabla_1 U(\tbtheta_i^k, \ubtheta_j^u) \|_2 \\
&+ \sup_{j \in [N]} \| \nabla_1 U(\tbtheta_i^k, \ubtheta_j^u) - \nabla_1 U(\tbtheta_i^k, \tbtheta_j^k) \|_2\\
\le& K (1 + \vert \undera_i^u \vert) \| \ubtheta_i^u - \tbtheta_i^k \|_2 + \sup_{j \in [N]} K (1 + \vert \undera_i^u\vert ) (1 + \vert \undera_j^u \vert) \| \ubtheta_i^u - \tbtheta_i^k\|_2\\
&+ \sup_{j \in [N]}K (1 + \vert \tilda_i^k \vert ) (1 + \vert \undera_j^u \vert) \| \ubtheta_j^u - \tbtheta_j^k \|_2 \\
\le& \max_{j \in [N]} K (1 + t + \vert \tilda_j^k - \undera_j^u \vert) (1 + t)  \| \ubtheta_i^u - \tbtheta_j^k \|_2 \le K (1 + t)^2 \cdot \max_{j \in [N]} \{ \| \ubtheta_j^u - \tbtheta_j^k \|_2, \| \ubtheta_j^u - \tbtheta_j^k \|_2^2 \}. 
\end{aligned}
\]

Denoting $\Delta (t) \equiv \sup_{k \in [0,t/\eps] \cap \N } \max_{i\leq N} \| \ubtheta^{k\eps}_i - \tbtheta^{k}_i \|_2$, we get the equation
\[
\begin{aligned}
\Delta (t) \le& K (1 + t)^2 \int_0^t \max\{ \Delta (s), \Delta(s)^2\} \de s + K (1 + t)^4 t \eps \\
\le&K (1 + T)^2 \int_0^t [\max\{ \Delta(s), \Delta(s)^2\} + (1 + T)^2 \eps] \de s. 
\end{aligned}
\]
Let $T_\Delta = \inf\{t: \Delta(t) \ge 1\}$. For $t \le T_\Delta$, we have $\Delta(s)^2 \le \Delta(s)$. Applying Gronwall's lemma, we get for any $t \le T_\Delta$, 
\[
 \Delta (t) \leq K e^{K(1 + T)^2 t} \eps. 
\]
Note we assumed $\eps \le 1/(K_0 e^{K_0 (1 + T)^3})$, which gives $K e^{K (1 + T)^2 T} \eps \le 1/2$. This shows that $T_\Delta \ge T$. Hence we get 
\[
\Delta(T) \le K e^{K (1 + T)^2 T} \eps. 
\]
Moreover, we immediately have, 
\[
\max_{i \in [N]}\sup_{k \in [0, T/\eps] \cap \N} \vert \tilda_i^k \vert \le \max_{i \in [N]}\sup_{t \in [0, T]} \vert \undera_i^t \vert + \max_{i \in [N]}\sup_{k \in [0, T/\eps] \cap \N} \| \tbtheta_i^t - \ubtheta_i^t \|_2 \le K(1 + t) + K e^{K (1 + T)^2 T} \eps \le 2 K(1 + t). 
\]
Finally, applying the last inequality in Lemma \ref{lem:bound_lip_traject_B} concludes the proof.
\end{proof}

\subsection{Bound between GD and SGD \label{sec:GD_SGD_B}}

\begin{proposition}[GD-SGD]\label{prop:GD_SGD_B}
There exists constants $K$ and $K_0$, such that if we take $\eps \le 1/ [K_0 (D + \log N + z^2) e^{K_0 (1 + T)^3}]$, the following holds 
 with probability at least $1 - e^{- z^2}$: for any $t \le T$, we have
\[
\begin{aligned}
\sup_{k \in [0, t/\eps] \cap \N} \max_{i \in [N]}\vert a_i^k \vert \le & K(1 + t), \\
\sup_{k \in [0, t/\eps] \cap \N} \max_{i \in [N]} \| \tbtheta_i^k - \btheta_i^k \|_2 \le & K e^{K (1 + T)^2 t} \sqrt{\eps } [\sqrt{D + \log N} + z ],\\
\sup_{k \in [0, t/\eps] \cap \N} \vert R_N(\tbtheta^k) - R_N(\btheta^k)\vert \le & K e^{K (1 + T)^2 t} \sqrt{\eps } [\sqrt{D + \log N} + z ].
\end{aligned}
\]
\end{proposition}

\begin{proof}[Proof of Proposition \ref{prop:GD_SGD_B}]
Denoting $\cF_k = \sigma ( ( \btheta^0_i )_{i \in [N]}, \bz_1 , \ldots , \bz_k )$ the $\sigma$-algebra generated by the data sample $\bz_{\ell}=(y_{\ell},\bx_{\ell})$
for $\ell\le k$, we get:
\[
\E [ \bF_i ( \btheta^k ; \bz_{k+1} ) \vert \cF_k ] = - \na V (\btheta^k_i ) - \frac{1}{N} \sum_{j = 1}^N \na_1 U ( \btheta_i^k , \btheta_j^k ) = \bG ( \btheta^k_i , \rho^{(N)}_k ),
\]
where $\rho^{(N)}_k \equiv (1/N) \sum_{i \in [N]} \delta_{\btheta^k_i}$ denotes the empirical distribution of the iterates of SGD. Hence we get:
\[
\begin{aligned}
\| \btheta^k_i - \tbtheta^k_i \|_2 =& \Big\|  \eps \sum_{l = 0}^{k-1} \bF_i ( \btheta_i^l ; \bz_{l+1} ) - \eps \sum_{l = 0}^{k-1} \bG ( \tbtheta^l_i ; {\Tilde \rho}^{(N)}_l ) \Big \|_2 \\
\leq & \Big \| \eps \sum_{l = 0}^{k-1} \bZ_i^l \Big\|_2 +   \eps \sum_{l = 0}^{k-1} \Big\| \bG ( \btheta^l_i; \rho^{(N)}_l ) -\bG ( \tbtheta^l_i; {\Tilde \rho}^{(N)}_l ) \Big\|_2,
\end{aligned}
\]
where $\bZ_i^l \equiv \bF_i ( \btheta^l ; \bz_{l+1} ) - \E [ \bF_i ( \btheta^l ; \bz_{l+1}) \vert \cF_l]$.

Denote $\bA_i^k = \sum_{l=0}^{k-1} \eps \bZ_i^l$. Hence $\{ \bA_i^k \}_{k\in \N}$ is a martingale adapted to $\{ \cF_k \}_{k \in \N}$. Note we have 
\[
\begin{aligned}
\bF_i ( \btheta^k ; \bz_{k+1}) =& ( (y_{k+1} - \hat y(\bx_{k+1}, \btheta^k)) \sigma(\bx_{k+1}; \bw_i^k), (y_{k+1} - \hat y(\bx_{k+1}, \btheta^k)) a_i^k \nabla_\bw \sigma(\bx_{k+1}; \bw_i^k) ),
\end{aligned}
\]
where $\hat y(\bx_{k+1}, \btheta^k) = (1/N) \sum_{j=1}^n a_j^k \sigma(\bx_{k+1}; \bw_j^k)$. 

The following discussion is under the conditional law $\cL(\,\cdot\, \vert \cF_k)$. Note that $\vert \sigma(\bx_{k+1}; \bw_i^k)\vert \le K$, and $\vert y_{k+1} -  \hat y_{k+1}(\btheta^k) \vert \le K (1 + \max_j\vert a_j^k \vert)$, hence $(y_{k+1} - \hat y(\bx_{k+1}, \btheta^k)) \sigma(\bx_{k+1}; \bw_i^k) $ is $K (1 + \max_i\vert a_i^k \vert)$-sub-Gaussian. Furthermore, $\nabla_\bw \sigma(\bx_{k+1}; \bw_i^k)$ is a $K$-sub-Gaussian random vector, and $\vert(y_{k+1} - \hat y(\bx_{k+1}, \btheta^k)) a_i^k \vert \le K (1 + \max_i \vert a_i^k \vert)^2$, hence $(y_{k+1} - \hat y(\bx_{k+1}, \btheta^k)) a_i^k \nabla_\bw \sigma(\bx_{k+1}; \bw_i^k)$ is a $K (1 + \max_j \vert a_j^k \vert)^2$-sub-Gaussian random vector. As a result, we have $\bF_i ( \btheta^k ; \bz_{k+1})$ under the conditional law $\cL(\,\cdot\, \vert \cF_k)$ is a $K (1 + \max_j \vert a_j^k \vert)^2$-sub-Gaussian random vector (concatenation of two possibly dependent sub-Gaussian random vectors is sub-Gaussian). 

Let $T_a = \min\{l: \max_{i \in [N]} \vert a_i^l \vert \ge M_T \}$ where $M_T \equiv 2 K (1 + T)$. Then we have 
\[
\E[e^{\< \blambda, \eps \bZ_i^k \>} \vert \cF_k] \ones\{ \max_{i \in [N]} \vert a_i^k \vert  \le M_T \} \le e^{\eps^2 K^2 M_T^4 \| \blambda \|_2^2 / 2}. 
\]
Now let $\bbA_i^k = \bA_i^{k \wedge T_a}$. Then $\bbA_i^k$ is also a martingale. Furthermore, we have 
\[
\begin{aligned}
&\E[e^{\<\blambda, \bbA_i^k - \bbA_i^{k - 1}\>} \vert \cF_{k-1}] \\
=& \E[e^{\<\blambda, \bA_i^k - \bA_i^{k - 1}\>} \ones\{ T_a \ge k \}\vert \cF_{k-1}] + \E[e^{\<\blambda, \bA_i^{T_a} - \bA_i^{T_a} \>} \ones\{ T_a \le k - 1 \}\vert \cF_{k-1}]\\
=& \E[e^{\<\blambda, \eps \bZ_i^{k-1}\>} \vert \cF_{k-1}] \ones\{ T_a \ge k \} + \ones\{ T_a \le k - 1 \}\\
= & \E[e^{\< \blambda, \eps \bZ_i^{k-1} \>} \vert \cF_{k-1}] \ones\{ \max_{i \in [N]} \vert a_i^{k-1} \vert  < M_T \}  + \ones\{ T_a \le k - 1 \} \\
\le& e^{\eps^2 K^2 M_T^4 \| \blambda \|_2^2 / 2}.\\
\end{aligned}
\]
Hence we can apply Azuma-Hoeffding's concentration bound (Lemma \ref{lem:Azuma}) to $\| \bbA_i^l \|_2$, 
\[
\P \Big( \max_{k \in [0,T/\eps ] \cap \N} \| \bbA_i^{k} \|_2 \geq K M_T^2 \sqrt{T\eps } ( \sqrt{D} + z )  \Big) \leq e^{- z^2},
\]
and taking the union bound over $i \in [N]$, we get:
\begin{equation}
\P \Big( \max_{i \in [N]} \max_{k \in [0,T/\eps ] \cap \N} \| \bA_i^{k \wedge T_a} \|_2 \le K M_T^2 \sqrt{T\eps } ( \sqrt{D + \log N} + z )  \Big) \ge 1- e^{- z^2}. 
\end{equation}

Denote the above event to be a good event $E_{\good}$, 
\[
E_{\good} = \Big\{ \max_{i \in [N]} \max_{k \in [0,T/\eps ] \cap \N} \| \bA_i^{k \wedge T_a} \|_2 \le K M_T^2 \sqrt{T\eps } ( \sqrt{D + \log N} + z ) \Big\}. 
\]
We consider the case in which $E_{\good}$ happens. We have (note we assumed $\eps \le 1 / (K_0 e^{K_0 (1 + T)^3})$, by Proposition \ref{prop:PD_GD_B}, we have $\sup_{k \in [0, t/\eps] \cap \N}\max_{i \in [N]}\vert \tilda_i^k \vert\le K(1 + t)$)
\[
\begin{aligned}
&  \| \bG ( \btheta^k_i ; \rho^{(N)}_k )  - \bG ( \tbtheta^k_i ; \trho^{(N)}_k ) \|_2 \\
\le& \| \nabla V(\btheta^k_i) - \nabla V(\tbtheta^k_i) \|_2 +  \sup_{j \in [N]} \| \nabla_1 U(\tbtheta_i^k, \btheta_j^k) - \nabla_1 U(\tbtheta_i^k, \tbtheta_j^k) \|_2 \\
&+ \sup_{j \in [N]} \| \nabla_1 U(\btheta_i^k, \btheta_j^k) - \nabla_1 U(\tbtheta_i^k, \btheta_j^k) \|_2\\
\le& K (1 + \vert \tilda_i^k \vert) \| \btheta_i^k - \tbtheta_i^k \|_2 + \sup_{j \in [N]} K (1 + \vert \tilda_i^k \vert ) (1 + \vert \tilda_j^k \vert) \| \btheta_j^k - \tbtheta_j^k\|_2\\
&+ \sup_{j \in [N]}K (1 + \vert \tilda_i^k \vert ) (1 + \vert a_j^k \vert) \| \btheta_i^k - \tbtheta_i^k \|_2 \\
\le& K (1 + T + \vert \tilda_i^k - a_i^k \vert) (1 + T) \max_{j \in [N]} \| \btheta_j^k - \tbtheta_j^k \|_2\\
 \le& K (1 + T)^2 \cdot \max_{j \in [N]} \{ \| \btheta_j^k - \tbtheta_j^k \|_2, \| \btheta_j^k - \tbtheta_j^k \|_2^2 \}. 
\end{aligned}
\]
Denoting $\Delta (t) \equiv \sup_{k \in [0,t/\eps] \cap \N } \max_{i\in [N]} \| \btheta^{k}_i - \tbtheta^{k}_i \|_2$. Denote $T_\Delta = \inf\{u: \Delta(u) \ge 1\}$. For $t \le T_a \wedge T_\Delta \wedge T$, we get the equation
\[
\begin{aligned}
\Delta (t) \le& K M_T^2 \int_0^t \Delta (s) \de s + K M_T^2 \sqrt{\eps T} ( \sqrt{D + \log N} + z ), \\
\end{aligned}
\]
which gives 
\[
\Delta(t) \le K M_T^2 \sqrt{\eps T} ( \sqrt{D  + \log N} + z ) e^{K M_T^2 t}. 
\]
Since we choose $\eps \le 1/ [K_0(D + \log N+ z^2) e^{ K_0 (1 + T)^3}]$, we have
\[
\Delta(T_a \wedge T_\Delta \wedge T) \le M_T^2 \sqrt{T\eps } ( \sqrt{D + \log N} + z ) e^{K M_T^2 T} \le 1/2.
\]
Moreover, for $t \le T_a \wedge T_\Delta \wedge T$, we have 
\[
\begin{aligned}
\sup_{k \in [0, t / \eps] \cap \N}\max_{i \in [N]} \vert a_i^k \vert   \le \sup_{k \in [0, t/\eps] \cap \N} \max_{i \in [N]} \vert \tilda_i^k \vert + \Delta(t) \le K(1 + T) + 1/2 < 2 K (1 + T). 
\end{aligned}
\]
This means that the stopping times $T_a, T_\Delta \ge T$. Hence, for any $t \le T$, we have 
\[
\begin{aligned}
\Delta(t) \le& M_T^2 \sqrt{\eps T} ( \sqrt{D  + \log N} + z ) e^{K M_T^2 t},\\
\sup_{k \in [0, t / \eps] \cap \N}\max_{i \in [N]} \vert a_i^k \vert   \le& 2 K (1 + t). 
\end{aligned}
\]
Note all these happens when event $E_{\good}$ happens. Hence, the probability such that the events above happens is at least $1 - e^{-z^2}$. Finally, by Lemma \ref{lem:bound_lip_traject_B}, we have the desired bound on $R_N$. This concludes the proof. 
\end{proof}

\section{Proof of Theorem \ref{thm:bound_approximating_noisy} part (A)}\label{sec:proof_2A}

The proof follows the same scheme as for Theorem \ref{thm:bound_approximating_noiseless} (A) and we will limit ourselves to describing the differences. 

Throughout this section, the assumptions {\rm A1}-{\rm A6} of Theorem \ref{thm:bound_approximating_noisy} are understood to hold. 
For the sake of simplicity we will write the proof under the following
restriction: 
\begin{itemize}
\item[R1.] The coefficients $a_i \equiv 1$. 
\item[R2.] The step size function $\xi(t) \equiv 1/2$. 
\end{itemize}
The proof for a general function $\xi(t)$ is obtained by a straightforward adaptation.

For the reader's convenience, we copy here the limiting PDE:
\[
\begin{aligned}
\partial_t \rho_t =& 2 \xi (t) \nabla \cdot [\rho(\btheta) \nabla \Psi_{\lambda} (\btheta; \rho_t)] + 2 \xi(t) \tau D^{-1}   \Delta_{\btheta} \rho_t, \\
\Psi_\lambda (\btheta; \rho) =& V (\btheta ) + \int U (\btheta , \btheta' ) \rho (\de \btheta' )  + \frac{\lambda }{2} \norm{\btheta}_2^2.
\end{aligned}
\]
We will consider four different coupled dynamics with same initialization $( \bbtheta^0_i )_{i \leq N}  \sim_{iid} \rho_0$ and stochastic term. We will denote $\{ \bW_i (s) \}_{ s\ge 0}$ for $i \in [N]$ independent $D$-dimensional Brownian motions. The integral equations and summation forms of the four dynamics are as follows:
\begin{itemize}
    \item The \textit{nonlinear dynamics (ND)}: 
    \begin{equation}
    \bbtheta^t_i = \bbtheta^0_i + 2 \int_0^t \xi(s) \bG ( \bbtheta^s_i ; \rho_s ) \de s + \int_0^t \sqrt{ 2 \xi (s) \tau D^{-1} } \de \bW_i (s), \label{eq:traj_PDE_C}
    \end{equation}
    where we denoted $\bG (\btheta ; \rho ) = - \na \Psi_{\lambda}  (\btheta; \rho ) = - \lambda \btheta -  \na V ( \btheta ) - \int \na_{\btheta} U (\btheta , \btheta ' ) \rho (\de \btheta')$, and $\bbtheta \sim \rho_0$ i.i.d.
    
    \item The \textit{particle dynamics (PD)}: 
    \begin{equation}
    \ubtheta^t_i = \ubtheta^0_i + 2 \int_0^t \xi(s) \bG ( \ubtheta^s_i ; \hat \rho^{(N)}_s ) \de s + \int_0^t \sqrt{2 \xi (s) \tau D^{-1} } \de \bW_i (s) ,
    \label{eq:traj_PD_C}
    \end{equation}
    where $\ubtheta^0_i = \bbtheta^0_i$.
    
    \item The \textit{gradient descent (GD)}:
    \[
    \begin{aligned}
    \tbtheta^{k}_i = \tbtheta^{0 }_i + 2\eps \sum_{l = 0}^{k-1} \xi (l \eps ) \bG ( \tbtheta^l_i ; {\Tilde \rho}^{(N)}_l ) + \int_0^{k\eps } \sqrt{2 \xi ([s]) \tau D^{-1}} \de \bW_i (s) ,
    \end{aligned}
    \]
    where $\tbtheta^0_i = \bbtheta^0_i$.
    
    \item The \textit{stochastic gradient descent (SGD)}:
    \[
    \btheta^k_i = \btheta^0_i + 2 \eps \sum_{l = 0}^{k-1}\xi (l \eps ) \bF_i ( \btheta^l ; \bz_{l+1} ) + \int_0^{k\eps } \sqrt{ 2 \xi ([s]) \tau D^{-1}} \de \bW_i (s) ,
    \]
    where we denoted $\bF_i ( \btheta^k ; \bz_{k+1} ) =  - \lambda \btheta_i^k + (y_{k+1} - \hat{y}_{k+1} ) \na_{\btheta_i } \sigma_\star ( \bx_{k+1} ; \btheta^k_i ) $, and $\btheta^0_i = \bbtheta^0_i$.
\end{itemize}

By Proposition \ref{prop:PDE_ND_C}, \ref{prop:ND_PD_C}, \ref{prop:PD_GD_C}, \ref{prop:GD_SGD_C}, there exists constants $K$ and $K_0$, such that with probability at least $1 - e^{-z^2}$, we have
\[
\begin{aligned}
\sup_{t \in [0, T]} \vert R_N(\bbtheta^t) - R(\rho_t) \vert \le&  K e^{KT} \frac{1}{\sqrt{N}}[ \sqrt{\log(NT)} + z], \\
\sup_{t \in [0, T]} \vert R_N(\ubtheta^t) - R_N(\bbtheta^t)\vert \le& K e^{KT} \frac{1}{\sqrt N} [\sqrt{\log(NT)} + z], \\
\sup_{k \in [0, T/\eps] \cap \N} \vert R_N(\ubtheta^{k\eps}) - R_N(\tbtheta^k) \vert \le& K e^{KT} [ \sqrt{\log(N (T / \eps \vee 1))} + z] \sqrt{\eps}, \\
\sup_{k \in [0, T/\eps] \cap \N} \vert R_N(\tbtheta^{k}) - R_N(\btheta^k) \vert \le& K e^{KT} \sqrt{T\eps } [\sqrt{D + \log N} + z ]. \\
\end{aligned}
\]
Combining these inequalities gives the conclusion of Theorem \ref{thm:bound_approximating_noisy} (A). In the following subsections, we prove all the above interpolation bounds, under the setting of Theorem \ref{thm:bound_approximating_noisy} (A).

\subsection{Technical lemmas \label{sec:technical_theorem_2_A}}

Define the maximum and the average of the norm of the initialization:
\[
\begin{aligned}
&\Theta_{\infty} \equiv \max_{i \leq N} \| \btheta^0_i \|_2, \qquad \Theta_1 \equiv \frac{1}{N} \sum_{i=1}^N  \| \btheta^0_i \|_2. 
\end{aligned}
\]
Similarly define the following bounds on the Brownian noise:
\begin{align*}
 \obW_i (t) & \equiv \sqrt{\frac{\tau}{D}}\bW_i (t)  =\int_0^t \sqrt{\frac{\tau}{D}} \de \bW_i (s),  & W_{\infty} & \equiv \max_{i \leq N} \sup_{t \leq T} \| \obW_{i} (t) \|_2,  & W_1 & \equiv \sup_{t \leq T} \frac{1}{N} \sum_{i=1}^N   \| \obW_{i} (t) \|_2.
 \end{align*}
\begin{lemma}\label{lem:bound_noise_initialization_D}
There exists a constant $K$ such that:
\[
\begin{aligned}
& \P \Big( \max ( \Theta_{\infty} , W_{\infty}) \leq K(1 + T) \big[ \sqrt{\log N} + z \big]  \Big) \geq 1 - e^{-z^2}, \\
& \P \Big( \max ( \Theta_1 , W_1) \leq K(1 + T) \big[ 1 + z \big] \Big) \geq 1 - e^{-z^2}.
\end{aligned}
\]
\end{lemma}

\begin{proof}[Proof of Lemma \ref{lem:bound_noise_initialization_D}]
 Let us first consider a generic $D$-dimensional $K^2$-sub-Gaussian random vector $\bX$, we have:
\[
\E_{\bX} [ \exp \{ \mu  \norm{\bX}^2_2 /2 \} ] = \E_{\bX,\bG} [ \exp \{ \sqrt{\mu} \< \bG , \bX \> \} ] \leq \E_{\bG} [ \exp \{ \mu K^2 \norm{\bG}_2^2/2 \} ] =  (1 - \mu K^2/2)^{-D/2},
\]
where $\bG \sim N (0, \bI_D)$. Recall that $\{(a_i^0, \bw_i^0)\}_{i \in [N]} \sim_{iid} \rho_0$ with $\bw_i^0$ being a $K^2/D$-sub-Gaussian vector in $\R^{D-1}$ independent of $a_i^0$. Using the above inequality, we get
\[
\P \Big( \|\bw^0_i \|_2 \ge u \Big) \le \E[ \exp \{ \mu  \|\bw^0_i \|_2^2 /2 \} ] / \exp \{ \mu u^2/2 \} \leq (1 - \mu K^2/D )^{-(D-1)/2} \exp \{ - \mu u^2/ 2 \}.
\]
Taking the union bound over $i \in [N]$, and noting that $\vert a_i^0 \vert \le K$, we get:
\[
\P \Big( \max_{i \in [N]} \|\btheta^0_i \|_2 \ge u + K \Big) \leq (1 - \mu K^2/D )^{-(D-1)/2} \exp \{ - \mu u^2/ 2+ \log N \}.
\]
Taking $ \mu = D/(2K^2)$ and $u = 2 K [ \sqrt{D + \log N}  + z ]/\sqrt{D}$, we get:
\[
\P \Big( \Theta_{\infty} \ge 2 K \lb \sqrt{D + \log N}  + z \rb/\sqrt{D} \Big) \leq e^{-z^2}.
\]
Let us now consider the average over $i\in [N]$ of the $\|\bw^0_i \|_2 $, which are independent, we get:
\[
\P \Big( N^{-1} \sum_{i=1}^N \| \bw^0_i \|_2 \geq u \Big) \leq \P \Big(  \sum_{i=1}^N \| \bw^0_i \|_2^2 \geq N u^2 \Big) \leq (1 - \mu K^2 / D )^{-N(D-1)/2} \exp \{ - \mu N u^2 /2 \}.
\]
Taking $\mu = D/(2K^2)$ and $u = 2K \lb 1 + z \rb $, noting $(1/N) \sum_{i=1}^N \vert a_i^0 \vert \le K$, we get:
\[
\P ( \Theta_1 \geq 2K \lb 1 + z \rb ) \leq e^{-z^2}.
\]
Similarly, we consider $\obW_i (t) \equiv  \sqrt{ \tau /D}  \bW_i (t)$ which is a $D$-dimensional Gaussian random variable with variance $\Var ( \oW^j_i (t) ) = \int_0^t (\tau/D) \de s =\tau t/ D$. We note that $\exp \{ \mu \|  \bW_i (t) \|_2^2  \}$ is a sub-martingale and by Doob's martingale inequality, we have:
\[
\P \Big( \sup_{t \in [0,T]}  \|  \bW_i (t) \|_2 \ge u \Big) \le \E[ \exp \{ \mu  \|  \bW_i (T) \|_2^2 /2 \} ] / \exp \{ \mu u^2/2 \} \leq (1 - 2\mu  (\tau T/D))^{-D/2} \exp \{ - \mu u^2/ 2 \}.
\]
Taking the union bound over $i \in [N]$ gives:
\[
\P \Big( \max_{i \in [N] } \sup_{t \in [ 0,T]}   \|  \bW_i (t) \|_2 \ge u \Big) \le (1 - 2\mu \tau T/D )^{-D/2} \exp \{ - \mu u^2/ 2 + \log N \}.
\]
Taking $ \mu =  D/(4 \tau T)$ and $u = 4 \sqrt{ T \tau} [ \sqrt{D + \log N}  + z ] /\sqrt{D}$, we get:
\[
\P \Big( W_{\infty} \ge 4 \sqrt{T \tau} [ \sqrt{D + \log N}  + z ] /\sqrt{D} \Big) \leq e^{-z^2}.
\]
We can consider the average over $i \in [N]$ of the preceding bound, by noticing that:
\[
\frac{1}{N } \sum_{i = 1}^N \|  \bW_i (t) \|_2 \leq \Big( \frac{1}{N } \sum_{i = 1}^N \|  \bW_i (t) \|_2^2 \Big)^{1/2} \equiv  \| \bW (t) \|_2/ \sqrt{N},
\]
where $ \bW(t) $ is a $ND$-dimensional Brownian motion. We can therefore apply Doob's martingale inequality to the sub-martingale $\exp \{ \mu \|   \bW (t) \|^2_2  \}$. We have
\[
\begin{aligned}
\P \Big( \sup_{t \in [ 0,T]}  \|   \bW (t) \|_2  \ge \sqrt{N} u \Big)  & \le \E [ \exp \{ \mu \|   \bW (t) \|_2 ^2 /2 \} ] / \exp \{  N \mu u^2/2 \} \\
& \leq (1 - 2\mu T \tau/ D )^{-ND/2} \exp \{ - N \mu u^2/ 2 \}.
\end{aligned}
\]
Taking $ \mu = D/(4 \tau T)$ and $u = 4 \sqrt{T \tau} [ 1  + z ]$, we get:
\[
\P \Big( W_1 \ge 4 \sqrt{T \tau} [ 1  + z ] \Big) \leq e^{-z^2}.
\]
This proves the lemma. 
\end{proof}

The two following lemmas are modified from \cite[Section 7.2, Lemma 7.5]{mei2018mean}.
\begin{lemma}\label{lem:continuity_noisy_PDE_C}
There exists a constant $K$, such that 
\[
\begin{aligned}
& \P \Big( \sup_{i \leq N} \sup_{k\in [0,T/\eta] \cap \N} \sup_{u \in [0,\eta]} \| \bbtheta^{k\eta +u}_i - \bbtheta^{k\eta}_i  \|_2 \leq Ke^{KT} \Big[ \sqrt{\log{ (N (T/\eta \vee 1)}) } + z \Big] \sqrt{\eta} \Big) \leq 1 - e^{-z^2},
\end{aligned}
\]
and for any $t,h \geq 0 , t+h \leq T$,
\[
\begin{aligned}
 W_2 (\rho_t , \rho_{t+h} ) \leq ( \E [ \| \bbtheta_i^{t} - \bbtheta_i^{t+h} \|_2^2 ] )^{1/2} \leq K e^{KT} \sqrt{h}.
\end{aligned}
\]
\end{lemma}
\begin{proof}[Proof of Lemma \ref{lem:continuity_noisy_PDE_C}] 
Define $\Delta_i (t) \equiv \sup_{s \leq t}  \| \bbtheta^t_i \|_2$. From Eq.~\eqref{eq:traj_PDE_C},
\begin{align*}
\| \bbtheta^t_i \|_2 &  \leq K   \int_0^t  \| \bbtheta^s_i \|_2  \de s +  \Theta_{\infty} +  W_{\infty},
\end{align*}
which gives, after applying Gronwall's inequality with the bounds of Lemma \ref{lem:bound_noise_initialization_D}:
\begin{align}
\P \Big( \Delta_i (t)  \leq K e^{KT}   \big[ \sqrt{\log N} + z   \big] \Big) \ge 1 - e^{-z^2}. \label{eq:bound_delta_i}
\end{align}
Consider $\Delta_i (h;k,\eps ) = \sup_{0 \leq u \leq \eps} \| \bbtheta_i^{k\eps + u} - \bbtheta_i^{k\eps} \|_2$. We have
\begin{align*}
\| \bbtheta_i^{k\eps + u} - \bbtheta_i^{k\eps} \|_2 & \leq \Big\| \int_{k\eps}^{k\eps +u} \xi (s) \bG(\bbtheta^s_i ; \rho_s)  \de s  \Big\|_2 + \| \obW_{i,k} (u) \|_2 \\
& \leq Kh \sup_{s \leq T} \left[ \lambda \| \bbtheta^s_i \|_2 + 1 \right]  + \| \obW_{i,k} (u) \|_2 ,
\end{align*}
where we defined $\obW_{i,k} (u)  \equiv \int_{k\eps}^{k\eps +u} \sqrt{ \tau /D } \de \bW_i (s)$. By a similar computation as in Lemma \ref{lem:bound_noise_initialization_D}, we have
\[
\P \Big( \max_{i\leq N} \sup_{k \in [0, T/\eps] \cap \N} \sup_{0 \leq u \leq \eps} \| \obW_{i,k} (u) \|_2 \geq 4\sqrt{K \eps} \lb \sqrt{\log ( N (T/\eps \vee 1))} + z \rb \Big) \leq e^{-z^2} .
\]
Combining this bound and Eq.~\eqref{eq:bound_delta_i} yields:
\begin{align}
\P \Big( \max_{i\leq N} \sup_{k \in [0, T/\eps] \cap \N}  \Delta_i (h;k,\eps )  \leq Ke^{KT} \big[ \sqrt{\log ( N (T/\eps \vee 1))} + z \big] \sqrt{\eps} \Big) \geq 1 - e^{-z^2}. \label{eq:bound_var_eps}
\end{align}

We now bound $W_2 (\rho_t , \rho_{t+h} )$:
\[
W_2 (\rho_t , \rho_{t+h} )^2 \leq \E [ \| \bbtheta^t - \bbtheta^{t+h} \|_2^2 ] = \int_0^\infty  \P ( \| \bbtheta^t - \bbtheta^{t+h} \|_2^2 \geq u ) \de u .
\]
Using Eq.~\eqref{eq:bound_var_eps}, we have (where we removed the union bound on $i \in [N]$ and $k \in [0,T/\eps]\cap \N$)
\[
\P \Big( \|\bbtheta_i^{t + h} - \bbtheta_i^{t} \|_2 \geq  Ke^{KT} [ 1 + z ] \sqrt{h} \Big) \leq e^{-z^2}.
\]
Integrating this upper bound on the probability yields the desired inequality.
\end{proof}

The exact same proof shows a similar lemma for the particle dynamics.

\begin{lemma}\label{lem:continuity_ODE_C}
There exists a constant $K$, such that 
\[
\begin{aligned}
& \P \Big( \max_{i \le N} \sup_{k\in [0,T/\eps] \cap \N} \sup_{u \in [0,\eps]}  \| \ubtheta^{k\eps +u}_i - \ubtheta^{k\eps}_i  \|_2 \leq Ke^{KT} \Big[ \sqrt{\log{ (N (T/\eps \vee 1))}}+z \Big] \sqrt{\eps} \Big) \leq 1 - e^{-z^2}.
\end{aligned}
\]
\end{lemma}

\subsection{Bound between PDE and nonlinear dynamics \label{sec:PDE_ND_C}}
\begin{proposition}[PDE-ND]\label{prop:PDE_ND_C}
There exists a constant $K$ such that with probability at least $1 - e^{-z^2}$, we have 
\[
\sup_{t \in [0, T]} \vert R_N(\bbtheta^t) - R(\rho_t) \vert \le K e^{KT} \frac{1}{\sqrt{N}}[ \sqrt{\log(NT)} + z].
\]
\end{proposition}

We will follow the same decomposition as in the proof of Proposition \ref{prop:PDE_ND_A}. The proof of term {\rm II} only depend on the upper bound on the potential $U$ and still apply. The term I bound follow from a similar proof as lemma \ref{lem:error_I_bound_A}.

\begin{lemma}[Term ${\rm I}$ bound]\label{lem:error_I_bound_C}
There exists $K$, such that 
\[
\P \Big( \sup_{t \in [0, T]} \vert R_N( \bbtheta^t ) - \E R_N( \bbtheta^t ) \vert \le K e^{KT} [\sqrt{\log(NT)} + z] /\sqrt{N}\Big) \ge 1 - e^{- z^2}. 
\]
\end{lemma}

\begin{proof}[Proof of Lemma \ref{lem:error_I_bound_C}]
Applying McDiarmid's inequality, we have 
\[
\P\Big( \vert R_N( \bar \btheta^t) - \E R_N( \bar \btheta^t ) \vert \ge \delta \Big) \le \exp\{ - N \delta^2 / K \}. 
\]
Furthermore we have the following increment bound for $t,h \geq 0$:
\[
\begin{aligned}
  &\Big \vert \vert R_N( \bar \btheta^{t +h} ) - \E R_N( \bar \btheta^{t +h} ) \vert - \vert R_N( \bar \btheta^{t} ) - \E R_N( \bar \btheta^{t} ) \vert \Big\vert \\
  \leq& \Big \vert R_N( \bar \btheta^{t+h} ) - R_N( \bar \btheta^{t} )  \Big \vert  + \Big \vert \E R_N( \bar \btheta^{t+h} ) - \E R_N( \bar \btheta^{t} ) \Big\vert \\
 \leq & K \Big[ \sup_{i \in [N]} \| \bbtheta^{t+h}_i - \bbtheta^{t}_i  \|_2 +  \E [  \| \bbtheta_j^{t+h} - \bbtheta_j^{t} \|_2 ] \Big].  
\end{aligned}
\]
Using Lemma \ref{lem:continuity_noisy_PDE_C}, we get 
\[
\sup_{k\in [0,T/\eta] \cap \N} \sup_{u \in [0,\eta]} \Big \vert \vert R_N( \bar \btheta^{k\eta +u} ) - \E R_N( \bar \btheta^{k\eta +u} ) \vert - \vert R_N( \bar \btheta^{k\eta} ) - \E R_N( \bar \btheta^{k\eta} ) \vert \Big \vert \le K e^{KT} \Big[ \sqrt{\log{ N (T/\eta \vee 1)} } +z\Big] \sqrt{\eta},
\]
with probability at least $1 - e^{-z^2}$. Hence taking an union bound over $s \in \eta \{0, 1, \ldots, \lfloor T / \eta\rfloor \}$ and bounding the variation inside the grid intervals, we have
\[
\P\Big( \sup_{t \in [0, T]} \vert R_N( \bar \btheta^t) - \E R_N( \bar \btheta^t) \vert \ge  \delta + K e^{KT} \lb \sqrt{\log{ N (T/\eta \vee 1)} } + z \rb \sqrt{\eta} \Big) \le (T / \eta) \exp\{ - N \delta^2 / K \} + e^{-z^2}.
\]
Taking $\eta = 1/N$ and $\delta = K [\sqrt{\log(NT)} + z]/\sqrt{N}$ concludes the proof. 
\end{proof}

\subsection{Bound between nonlinear dynamics and particle dynamics \label{sec:ND_PDE_C}}

\begin{proposition}[ND-PD]\label{prop:ND_PD_C}
There exists a constant $K$, such that with probability at least $1 - e^{-z^2}$, we have 
\begin{align}
\sup_{t \in [0, T]} \max_{i \in [N]} \| \ubtheta_i^t - \bbtheta_i^t \|_2 \le K e^{KT} \frac{1}{\sqrt N} [\sqrt{\log(NT)} + z],  \label{eqn:particle_population_perturbation_bound_C}\\
\sup_{t \in [0, T]} \vert R_N(\ubtheta^t) - R_N(\bbtheta^t)\vert \le K e^{KT} \frac{1}{\sqrt N} [\sqrt{\log(NT)} + z].
\end{align}
\end{proposition}

\begin{proof}[Proof of Proposition \ref{prop:ND_PD_C}]
The nonlinear dynamics and the particle dynamics are coupled by using the same Brownian motion, and the noise term cancel out. By the same calculation as in Proposition \ref{prop:ND_PD_A}, we get 
\begin{equation}\label{eqn:difference_dynamics_C}
\begin{aligned}
\frac{\de}{\de t} \| \ubtheta_i^t - \bbtheta_i^t \|_2 \le& K \cdot \max_{j \in [N]} \| \ubtheta_j^t - \bbtheta_j^t\|_2 + K / N + I_i^t,
\end{aligned}
\end{equation}
where
\[
I_i^t =\Big\| \frac{1}{N}\sum_{j \neq i} \Big[ \nabla_1 U( \bbtheta_i^t, \bbtheta_j^t) - \int \nabla_1 U(\bbtheta_i^t, \btheta) \rho_t(\de \btheta) \Big] \Big\|_2. 
\]
Now we would like to prove a uniform bound for $I^t_i$ for $i \in [N]$ and $t \in [0,T]$.
\begin{lemma}\label{lem:bound_I_C}
There exists a constant $K$, such that
\[
\P\Big( \sup_{t \in [0, T]} \max_{i \in[N]} I_i^t  \le K e^{KT} [\sqrt{\log(NT)} + z]/ \sqrt{N} \Big)  \ge 1 - e^{- z^2}. 
\]
\end{lemma}
\begin{proof}[Proof of Lemma \ref{lem:bound_I_C}]
Denoting $\bX_i^t = \nabla_1 U( \bbtheta_i^t, \bbtheta_j^t) - \int \nabla_1 U(\bbtheta_i^t, \btheta) \rho_t(\de \btheta)$, we have $\E[\bX_i^t \vert \bbtheta_i^t] = 0$, (where the expectation is taken with respect to $\bbtheta_j^0 \sim \rho_0$ and $\{ \bW_j (s) \}_{s \ge 0}$ for $j \neq i$), and $\| \bX_i^t \|_2 \le 2 K$ (by assumption that $\| \nabla U \|_2 \le K$). By Lemma \ref{lem:bounded_difference_martingale}, we have for any fixed $i \in [N]$ and $t \in [0, T]$, 
\[
\P\Big( I_i^t \ge K (\sqrt {1/N} + \delta) \Big) = \E\Big[ \P\Big( I_i^t \ge K (\sqrt {1/N} + \delta) \vert \bar \btheta_i^t \Big)\Big] \le \exp\{ - N \delta^2  \}. 
\]
We then bound the variation of $I_i^s$ over an interval $[t, t+h]$, with $t,h \geq 0$:
\[
\begin{aligned}
\vert I_i^{t+h} - I_i^{t} \vert   \leq & \frac{1}{N} \sum_{j \leq i} \Big\| \na_1 U ( \bbtheta^{t+h}_i , \bbtheta^{t+h}_j) - \na_1 U ( \bbtheta^{t} _i , \bbtheta^{t}_j) \Big\|_2 \\
& +  \Big\| \int \na_1 U ( \bbtheta^{t+h}_i , \btheta) \rho_{t+h} (\de \btheta) - \int \na_1 U ( \bbtheta^{t} _i , \btheta) \rho_{t}  (\de \btheta) \Big\|_2  \\
 \leq & K \Big[ \sup_{i \leq N} \| \bbtheta^{t+h}_i - \bbtheta^{t}_i \|_2 +  \E [ \| \bbtheta_j^{t+h} - \bbtheta_j^{t} \|_2]\Big].
\end{aligned}
\]
By Lemma \ref{lem:continuity_noisy_PDE_C}, there exists $K$ such that, we have  
\[
\begin{aligned}
\P \Big( \sup_{k\in [0,T/\eta] \cap \N} \sup_{u \in [0,\eta]}  \vert I_i^{k\eta +u} - I_i^{k\eta} \vert \le K e^{KT} \Big[ \sqrt{\log{ (N (T/\eta \vee 1))} } +z \Big] \sqrt{\eta} \Big) \geq 1 - e^{-z^2}.
\end{aligned}
\]
Taking an union bound for $i \in [N]$ and $s \in \eta \{ 0, 1, \ldots, \lfloor T / \eta \rfloor\}$ and bounding the variation inside the grid intervals, we have
\[
\P\Big( \sup_{t \in [0, T]} \max_{i \in[N]} I_i^t \ge K (\sqrt {1/N} + \delta) + K e^{KT} \lb \sqrt{\log{ N (T/\eta \vee 1)} } \rb \sqrt{\eta} \Big)  \le (N T / \eta) \exp\{ - N \delta^2 \} + e^{-z^2}. 
\]
Taking $\eta = 1/N $, and $\delta = K [\sqrt{\log(NT)} + z] / \sqrt N$, we get the desired result.
\end{proof}

Denote $\delta_N (T,z) = K e^{KT} [\sqrt{\log (N T )} + z]/ \sqrt{N} $ and
\[
\Delta(t) = \sup_{s \in [t]} \max_{i \in [N]} \| \ubtheta_i^s - \bbtheta_i^s \|_2. 
\]
With probability at least $1 - e^{-z^2}$, we have
\[
\begin{aligned}
\Delta'(t) \le K \cdot \Delta(t)+ \delta_N(T, z),
\end{aligned}
\]
which, after applying Gronwall's inequality, concludes the proof.
\end{proof}

\subsection{Bound between particle dynamic and GD \label{sec:ND_GD_C}}

\begin{proposition}[PD-GD]\label{prop:PD_GD_C}
There exists a constant $K$ such that with probability at least $1 - e^{-z^2}$, we have 
\[
\begin{aligned}
 \sup_{k \in [0,T/\eps] \cap \N} \max_{i \leq N} \| \ubtheta_i^{k\eps} - \tbtheta_i^{k} \|_2 \leq& Ke^{KT} \Big[ \sqrt{\log{ (N (T/\eps \vee 1))} } + z \Big] \sqrt{\eps},\\
\sup_{k \in [0, T/\eps] \cap \N} \vert R_N(\ubtheta^{k\eps}) - R_N(\tbtheta^k) \vert \le& K e^{KT} [ \sqrt{\log(N (T / \eps \vee 1))} + z] \sqrt{\eps}.
\end{aligned}
\]
\end{proposition}

\begin{proof}[Proof of Proposition \ref{prop:PD_GD_C}]
For $k \in \N$ and $t=k\eps$,
\[
\begin{aligned}
&\| \ubtheta_i^{t} - \tbtheta_i^{k} \|_2 \leq & \int_0^t \| \bG ( \ubtheta^s_i ; \urhoN_s )  - \bG ( \ubtheta^{[s]}_i ; \urhoN_{[s]} ) \|_2 \de s + \int_0^t \| \bG ( \ubtheta^{[s]}_i ; \urhoN_{[s]} )  - \bG ( \tbtheta^{[s]/\eps}_i ; \trhoN_{[s]/\eps} ) \|_2 \de s.
\end{aligned}
\]
We have by Lemma \ref{lem:continuity_ODE_C}
\[
\begin{aligned}
\int_0^t  \| \bG ( \ubtheta^s_i ; \urhoN_s )  - \bG ( \ubtheta^{[s]}_i ; \urhoN_{[s]} ) \|_2 \de s & \le K T \sup_{s \in [0,T]} \max_{i \in [N]} \|  \ubtheta^{s}_i - \ubtheta^{[s]}_i \|_2 \\
& \le T K e^{KT} \lb \sqrt{\log{ (N (T/\eps \vee 1))} } + z\rb \sqrt{\eps},
\end{aligned}
\]
with probability at least $1 - e^{-z^2}$. Denote $\delta (N,T,z) = T K e^{KT} \lb \sqrt{\log{ (N (T/\eps \vee 1))} } + z\rb \sqrt{\eps}$  and 
\[
\Delta (t ) \equiv \sup_{k \in [0,t/\eps] \cap \N } \max_{i\leq N} \| \ubtheta^{k\eps}_i - \tbtheta^{k}_i \|_2.
\] 
With probability at least $1-e^{-z^2}$, we get
\[
\begin{aligned}
\Delta (t) \le K \int_0^t \Delta (s) \de s + \delta (N,T,z).
\end{aligned}
\]
Applying Gronwall's inequality concludes the proof.
\end{proof}

\subsection{Bound between GD and SGD \label{sec:GD_SGD_C}}
\begin{proposition}[GD-SGD]\label{prop:GD_SGD_C}
There exists a constant $K$ such that, with probability at least $1 - e^{-z^2}$, we have 
\[
\begin{aligned}
\sup_{k \in [0, T/\eps] \cap \N} \max_{i \in [N]} \| \tbtheta_i^k - \btheta_i^k \|_2 \le& K e^{KT} \sqrt{T\eps } [\sqrt{D + \log N} + z ],\\
\sup_{k \in [0, T/\eps] \cap \N} \vert R_N(\tbtheta^{k}) - R_N(\btheta^k) \vert \le& K e^{KT} \sqrt{T\eps } [\sqrt{D + \log N} + z ].
\end{aligned}
\]
\end{proposition}

\begin{proof}[Proof of Proposition \ref{prop:GD_SGD_C}]
We coupled the noise between the GD and SGD such that the noise cancels out. Noticing furthermore that the regularization term does not depend on $\bz_k$ and vanishes in the martingale difference $\bZ_i^l \equiv \bF_i ( \btheta^l ; \bz_{l+1} ) - \E [ \bF_i ( \btheta^l ; \bz_{l+1}) \vert \cF_l]$, where $\cF_k = \sigma ( ( \btheta^0_i )_{i \in [N]}, (\bz_l)_{l=0}^k, (\bW_i(s))_{s \le k \eps} )$ . Therefore the same proof as Proposition \ref{prop:GD_SGD_A} applies here.
\end{proof}

\section{Proof of Theorem \ref{thm:bound_approximating_noisy} part (B)}

We remind the notations used in the proof of Theorem \ref{thm:bound_approximating_noiseless} (B): 
for $\btheta = (a, \bw)$ and $\btheta' = (a', \bw')$,
\[
\begin{aligned}
v(\bw) =& - \E_{y, \bx}[y \sigma(\bx; \bw)],\\
u(\bw, \bw') =& \E_{\bx} [\sigma(\bx; \bw) \sigma(\bx; \bw')],\\
V(\btheta) =& a \cdot v (\bw),\\
U (\btheta , \btheta ' ) =& a a' \cdot u ( \bw , \bw' ) \\
\na_\btheta V(\btheta) =& (v(\bw), a \na_\bw v(\bw)), \\
\na_\btheta U(\btheta, \btheta') =& (a' \cdot u(\bw, \bw'), a a' \cdot \na_\bw u(\bw, \bw')).
\end{aligned}
\]
%

For convenience, we copy here the properties of the potentials $V(\btheta)$ and $U(\btheta , \btheta')$ listed in Lemma \ref{lem:bound_lip_traject_B}. Denoting $\btheta = (a, \bw)$, $\btheta_1 = (a_1, \bw_1)$ and $\btheta_2 = (a_2, \bw_2)$. We have
\[
\begin{aligned}
\vert V(\btheta) \vert, \| \na V (\btheta ) \|_2 \leq& K( 1 + \vert a \vert),\\
\| \na V (\btheta_1) - \na V (\btheta_2) \|_2 \leq& K \cdot [ 1 + \min \{ \vert a_1\vert, \vert a_2\vert \}] \cdot  \| \btheta_1 - \btheta_2 \|_2, \\
\vert U(\btheta, \btheta') \vert, \| \na_1 U(\btheta, \btheta' ) \|_2 \le& K  (1 + \vert a \vert) (1 + \vert a' \vert), \\
\| \na_{(1, 2)} U(\btheta_1, \btheta ) - \na_{(1, 2)} U(\btheta_2, \btheta ) \|_2 \leq& K  (1 + \vert a\vert ) \cdot [ 1 +  \min \{ \vert a_1\vert,\vert a_2\vert \} ]\cdot  \| \btheta_1 - \btheta_2 \|_2. 
\end{aligned}
\]
%
%

Throughout this section, the assumptions A1 - A6 are understood to hold.
For the sake of simplicity we will write the proof under the following
restriction: 
\begin{itemize}
\item[R1.] The step size function $\xi(t) \equiv 1/2$. 
\end{itemize}
The proof for a general function $\xi(t)$ is obtained by a straightforward adaptation.

We recall the form of the limiting PDE:
\[
\begin{aligned}
\partial_t \rho_t =& 2 \xi (t) \nabla \cdot [\rho(\btheta) \nabla \Psi_{\lambda} (\btheta; \rho_t)] + 2 \xi(t) \tau D^{-1}   \Delta_{\btheta} \rho_t, \\
\Psi_\lambda (\btheta; \rho) =& V (\btheta ) + \int U (\btheta , \btheta' ) \rho (\de \btheta' )  + \frac{\lambda }{2} \norm{\btheta}_2^2.
\end{aligned}
\]
We will consider four different coupled dynamics with same initialization $( \bbtheta^0_i )_{i \leq N}  \sim_{iid} \rho_0$. The integral equations and summation form are as follows:
\begin{itemize}
    \item The \textit{nonlinear dynamics (ND)}: 
    \begin{equation}
    \bbtheta^t_i = \bbtheta^0_i + 2 \int_0^t \xi(s) \bG ( \bbtheta^s_i ; \rho_s ) \de s + \int_0^t \sqrt{ 2 \xi (s) \tau D^{-1} } \de \bW_i (s), \label{eq:traj_PDE_D}
    \end{equation}
    where we denoted $\bG (\btheta ; \rho ) = - \na \Psi_{\lambda}  (\btheta; \rho ) = - \lambda \btheta -  \na V ( \btheta ) - \int \na_{\btheta} U (\btheta , \btheta ' ) \rho (\de \btheta')$, and $\bbtheta \sim \rho_0$ iid.
    
    \item The \textit{particle dynamics (PD)}: 
    \begin{equation}
    \ubtheta^t_i = \ubtheta^0_i + 2 \int_0^t \xi(s) \bG ( \ubtheta^s_i ; \hat \rho^{(N)}_s ) \de s + \int_0^t \sqrt{2 \xi (s) \tau D^{-1} } \de \bW_i (s) ,
    \label{eq:traj_PD_C}
    \end{equation}
    where $\ubtheta^0_i = \bbtheta^0_i$.
    
    \item The \textit{gradient descent (GD)}:
    \[
    \begin{aligned}
    \tbtheta^{k}_i = \tbtheta^{0 }_i + 2\eps \sum_{l = 0}^{k-1} \xi (l \eps ) \bG ( \tbtheta^l_i ; {\Tilde \rho}^{(N)}_l ) + \int_0^{k\eps } \sqrt{2 \xi ([s]) \tau D^{-1}} \de \bW_i (s) ,
    \end{aligned}
    \]
    where $\tbtheta^0_i = \bbtheta^0_i$.
    
    \item The \textit{stochastic gradient descent (SGD)}:
    \[
    \btheta^k_i = \btheta^0_i + 2 \eps \sum_{l = 0}^{k-1}\xi (l \eps ) \bF_i ( \btheta^l ; \bz_{l+1} ) + \int_0^{k\eps } \sqrt{ 2 \xi ([s]) \tau D^{-1}} \de \bW_i (s) ,
    \]
    where we defined $\bF_i ( \btheta^k ; \bz_{k+1} ) =  - \lambda \btheta_i^k + (y_{k+1} - \hat{y}_{k+1} ) \na_{\btheta_i } \sigma_\star ( \bx_{k+1} ; \btheta^k_i ) $, and $\btheta^0_i = \bbtheta^0_i$.
\end{itemize}

By Proposition \ref{prop:PDE_ND_D}, \ref{prop:ND_PD_D}, \ref{prop:PD_GD_D}, \ref{prop:GD_SGD_D}, there exists constants $K$, such that with probability at least $1 - e^{-z^2}$, we have
\[
\begin{aligned}
\sup_{t \in [0, T]} \vert R_N(\bbtheta^t) - R(\rho_t) \vert \le&  Ke^{KT} [ \log^{3/2} (NT) + z^3 ] / \sqrt{N} , \\
\sup_{t \in [0, T]} \vert R_N(\ubtheta^t) - R_N(\bbtheta^t)\vert \le& K e^{e^{KT} [ \sqrt{\log N } + z^2 ]} [ \sqrt{D\log N} + \log^{3/2} (NT) + z^5 ]/\sqrt{N}, \\
\sup_{k \in [0, T/\eps] \cap \N} \vert R_N(\ubtheta^{k\eps}) - R_N(\tbtheta^k) \vert \le& Ke^{e^{KT}[\sqrt{\log N} +z^2]} [ \log ( N (T/\eps \vee 1))  + z^6 ] \sqrt{\eps}, \\
\sup_{k \in [0, T/\eps] \cap \N} \vert R_N(\tbtheta^{k}) - R_N(\btheta^k) \vert \le& Ke^{e^{KT}[\sqrt{\log N} +z^2]} [ \sqrt{D}\log N + \log^{3/2} N + z^5 ] \sqrt{\eps }  . \\
\end{aligned}
\]
Combining these inequalities gives the conclusion of Theorem \ref{thm:bound_approximating_noisy} (B). In the following subsections, we prove all the above interpolation bounds, under the setting of Theorem \ref{thm:bound_approximating_noisy} (B).

\subsection{Technical lemmas}

The bounds on the potentials $U, V$, and their derivatives scales with the coefficients $a$, which can be arbitrarily large with non-zero probability due to the Brownian noise. In our analysis we will need to keep track of the maximum and the first moment of $|a|$ for each of the different dynamics. In this section we will show that there exists high probability bounds along the trajectories.

We recall the following notations introduced in Appendix Section \ref{sec:technical_theorem_2_A},
\[
\begin{aligned}
\Theta_{\infty} \equiv \max_{i \leq N} \| \btheta^0_i \|_2, \qquad \Theta_1 \equiv \frac{1}{N} \sum_{i=1}^N  \| \btheta^0_i \|_2. 
\end{aligned}
\]
and on the Brownian motion,
\begin{align*}
 \obW_i (t) & \equiv \sqrt{\frac{\tau}{D}}\bW_i (t)  =\int_0^t \sqrt{\frac{\tau}{D}} \de \bW_i (s),  & W_{\infty} & \equiv \max_{i \leq N} \sup_{t \leq T} \| \obW_{i} (t) \|_2,  & W_1 & \equiv \sup_{t \leq T} \frac{1}{N} \sum_{i=1}^N   \| \obW_{i} (t) \|_2.
 \end{align*}
For convenience, we recall here the bounds derived in Lemma \ref{lem:bound_noise_initialization_D}:
\[
\begin{aligned}
& \P \Big( \max ( \Theta_{\infty} , W_{\infty}) \leq K(1 + T) \big[ \sqrt{\log N} + z \big]  \Big) \geq 1 - e^{-z^2}, \\
& \P \Big( \max ( \Theta_1 , W_1) \leq K(1 + T) \big[ 1 + z \big] \Big) \geq 1 - e^{-z^2}.
\end{aligned}
\]

In the following lemma, and throughout the proof, we will denote $\bbara^t \equiv (\bara_1^t , \ldots , \bara_N^t) \in \R^N$ the vector of the $\bara_i^t $ variables of the nonlinear dynamics. Similarly we will denote $\bundera^t$, $\btilda^k$ and $\ba^k $ the vectors of variable $a$ associated to the particle dynamics, gradient descent and stochastic gradient descent. We will furthermore use $\|\ba\|_1$,
$\|\ba\|_{\infty}$ to denote the $\ell_1$ and $\ell_{\infty}$ norms of the coefficients vector. 

\begin{lemma}\label{lem:bound_asquared_noisy_D}
There exists a constant $K$, such that denoting $M_2(t) = K e^{Kt}$, we have
\[
\sup_{s \in [0,t]} \int a^2 \rho_s (\de a) \leq M_2(t).
\]
Furthermore, letting $(\bara^t, \bbw^t) \sim \rho_t$, then $\bara^t$ is $M_2(t)$-sub-Gaussian.
\end{lemma}

\begin{proof}[Proof of Lemma \ref{lem:bound_asquared_noisy_D}]
Denote $A(t) = \int a^2 \rho_t (\de a)/2$. For simplicity, we will directly take the derivative of this function. This computation can be made rigorous by considering smooth approximation of a truncated squared function, with bounded second derivative, and using the definition of weak solution. We get:
\[
\frac{\de}{\de t} A(t) =  \tau /D - \int \Big[ \lambda a^2 + a \cdot  v ( \bw )  + a \cdot \int a'  u ( \bw , \bw' ) \rho_t ( \de \btheta' )  \Big] \rho_t ( \de \btheta ) \leq K + K A(t), 
\]
which implies by applying Gronwall's lemma we have
\[
\sup_{s \in [0,t]} \int a^2 \rho_s (\de a) \leq Ke^{Kt}.
\]

Let us consider the nonlinear dynamics for the variable $\bara^t \sim \rho_t$:
\[
\de \bara^t = - \lambda \bara^t \de t + \Big[ -  v ( \bbw^t )  - \int a' u ( \bbw^t , \bw') \rho_t ( \de \btheta' ) \Big] \de t  + \sqrt{\frac{\tau}{D}} \de W^a (t).
\]
Denote $u_{\lambda}(t) = \bara^t e^{\lambda t}$ and
\[
K(\bbw^t,\rho_t) = -  v ( \bbw^t )  - \int a' u ( \bbw^t , \bw') \rho_t ( \de \btheta' ),
\]
we get
\[
\de u_{\lambda}(t) = e^{\lambda t} K(\bbw^s ,\rho_t) \de t  + e^{\lambda t} \sqrt{\frac{\tau}{D}} \de W^a (t),
\]
and in integration form we have
\[
u_{\lambda}(t) = u_{\lambda}(0) + \int_0^t e^{\lambda s} K(\bbw^s ,\rho_s) \de s + \int_0^t   e^{\lambda s} \sqrt{\frac{\tau}{D}} \de W^a (s).
\]
We deduce that we can rewrite $\bara^t \sim \rho_t$ as the sum of three random variables:
\[
\bara^t = \underbrace{e^{-\lambda t} a^0}_{\Gamma_1} + \underbrace{\int_0^t e^{- \lambda (t-s)} K(\bbw^s ,\rho_s) \de s}_{\Gamma_2} + \underbrace{\int_0^t   e^{- \lambda (t-s) } \sqrt{\frac{\tau}{D}} \de W^a (s)}_{\Gamma_3}.
\]
By assumption $a_0$ is $K$-bounded, and thus $\Gamma_1$ is $K^2$-sub-Gaussian. By the boundedness of $u$ and $v$, Cauchy Schwartz inequality, and by $A(t) \le M_2(t)$, then for $s \leq t$, we have $\vert K(\bbw^s,\rho_s) \vert \leq Ke^{Kt}$, hence the random variable $\Gamma_2$ is $K e^{Kt}$-bounded and thus $K e^{Kt}$-sub-Gaussian. The random variable $\Gamma_3$ is a Gaussian random variable with variance
\[
\text{Var} (\Gamma_3) = \int_0^t  e^{- 2 \lambda (t-s) }\frac{\tau}{D } \de s \leq Kt.
\]
We deduce that $\bara^t$ is the sum of three (dependent) sub-Gaussian random variables with parameters $K^2, K e^{Kt}, Kt$ respectively, and therefore the sum $\bara^t$ is $Ke^{Kt}$-sub-Gaussian. 
\end{proof}

\begin{lemma}\label{lem:bound_a_trajectories_noisy_D}
There exists a constant $K$ such that with probability at least $1 - e^{-z^2}$, we have
\[
\begin{aligned}
 &\max \Big( \sup_{t \in [0,T]} \| \bbara^t \|_1 , \sup_{t \in [0,T]} \| \bundera^t\|_1 , \sup_{k \in [0,T/\eps]\cap\N} \| \btilda^k \|_1 ,  \sup_{k \in [0,T/\eps]\cap\N} \| \ba^k \|_1 \Big) \leq N \cdot Ke^{KT} [ 1 + z ] \equiv N \cdot M_1, \\
&\max \Big( \sup_{t \in [0,T]} \| \bbara^t \|_{\infty} , \sup_{t \in [0,T]} \| \bundera^t \|_{\infty}, \sup_{k \in [0,T/\eps]\cap\N} \| \btilda^k \|_{\infty} ,  \sup_{k \in [0,T/\eps]\cap\N} \| \ba^k \|_{\infty}   \Big) \leq Ke^{KT} [ \sqrt{\log N} + z ] \equiv M_{\infty}.
\end{aligned}
\]

\end{lemma}

\begin{proof}[Proof of Lemma \ref{lem:bound_a_trajectories_noisy_D}]
Let us start with the non-linear dynamics trajectories. We have in integral form:
\begin{equation}
\begin{aligned}
\abs{\bba^t_i} & = \Big \vert \bba^0_i + \int_0^t  \Big[ - \lambda \bba^s_i - v ( \bbw^s_i ) - \int a u ( \bbw^s_i , \bw) \rho_s ( \de  \btheta ) \Big]  \de s + \int_0^t \sqrt{\frac{\tau }{D} } \de W^a_i (s) \Big\vert \\
& \leq \vert \bba^0_i\vert + K \int_0^t \vert \bba^s_i \vert ds + KT \sqrt{M_2} + \vert \oW^a_i (t) \vert \label{eq:traj_a_PDE_noisy_D} \\
& \leq K \int_0^t \abs{\bba^s_i} ds + \Theta_{\infty} + KT \sqrt{M_2} + W_{\infty},
\end{aligned}
\end{equation}
where we recall that $\oW^a_i (t) = \sqrt{\tau/D }  W^a_i (t) $. Applying Gronwall's lemma to $\Delta (t) = \sup_{s \in [0,t]} \vert \bba^s_i \vert$ with Lemma \ref{lem:bound_noise_initialization_D} gives:
\[
\Delta (T) \leq K e^{KT} [\sqrt{\log N}+z],
\]
while summing (\ref{eq:traj_a_PDE_noisy_D}) over $i$ yields:
\[
(\| \bbara^t \|_1 /N ) \leq \Theta_1 + K \int_0^t (\| \bbara^s \|_1/N) \de s  + Ke^{KT} + W_1,
\]
and by Gronwall's lemma: $\sup_{t \in [0,T]} \| \bbara^s \|_1 /N \leq Ke^{KT} [ 1 + z ]$. The same proof applies to the other trajectories and we will only write down the corresponding inequality on the integral or summation form:
\[
\begin{aligned}
\vert \ua^t_i \vert & \leq \vert a^0_i \vert + KT + K \int_0^t \vert \undera^s_i\vert \de s +K \int_0^t (\| \bundera^s \|_1 /N) \de s + \vert \oW^a_i (t) \vert, \\
\vert \ta^k_i \vert & \leq \vert a^0_i\vert + KT + K \eps  \sum_{l=1}^{k-1} \vert \ta^l_i \vert +K \eps  \sum_{l=1}^{k-1} (\| \btilda^l \|_1 /N) + \vert  \oW^a_i (t) \vert, \\
\vert a^k_i \vert & \leq \vert a^0_i\vert + KT + K \eps  \sum_{l=1}^{k-1} \vert a^l_i\vert +K \eps  \sum_{l=1}^{k-1} (\| \ba^l \|_1/N ) + \vert \oW^a_i (t) \vert.
\end{aligned}
\]
\end{proof}

\begin{lemma}\label{lem:bound_increment_noisy_D}
There exists a constant $K$ such that:
\[
\begin{aligned}
& \P \Big( \sup_{i \leq N} \sup_{k\in [0,T/\eps] \cap \N} \sup_{u \in [0, \eps]} \|\bbtheta^{k\eps +u}_i - \bbtheta^{k\eps}_i  \|_2 \leq Ke^{KT} \big[ \sqrt{\log ( N (T/\eps \vee 1))} + z \big] \sqrt{\eps} \Big) \leq 1 - e^{-z^2}, \\
& \P \Big( \sup_{i \leq N} \sup_{k\in [0,T/\eps] \cap \N} \sup_{u \in [0, \eps]} \| \ubtheta^{k\eps +u}_i - \ubtheta^{k\eps}_i  \|_2 \leq Ke^{KT} \big[ \sqrt{\log ( N (T/\eps \vee 1))} + z \big] \sqrt{\eps} \Big) \leq 1 - e^{-z^2}. 
\end{aligned}
\]
Furthermore, we have for $t,h \geq 0 , t+h \leq T$,
\[
W_2 (\rho_t , \rho_{t+h} ) \leq  \Big( \E [ \| \bbtheta^{t} - \bbtheta^{t+h}\|_2^2 ] \Big)^{1/2} \leq K e^{KT} \sqrt{h}.
\]
\end{lemma}

\begin{proof}[Proof of Lemma \ref{lem:bound_increment_noisy_D}]
We will only show the result for the non-linear dynamic. The proof for the particle dynamic will be exactly the same, upon replacing $\sqrt{M_2}$ by $M_1$.

\noindent
{\bf Step 1.}
Let us consider $\Delta_i (t) \equiv \sup_{s \leq t}  \| \bbtheta^t_i \|_2$ and $\Delta_0 (t) \equiv \sup_{s \leq t} \frac{1}{N}\sum_{i \leq N} \| \bbtheta^t_i \|_2$ :
\begin{align*}
\| \bbtheta^t_i \|_2 & \leq \| \btheta^0_i\|_2 +2K \int_0^t \Big( \lambda \| \bbtheta^s_i \|_2 + K(1 + \vert \bba^s_i \vert ) + K\sqrt{M_2} (1+\vert \bba^s_i \vert ) \Big) \de s + \| \obW_i \|_2 \\
& \leq K   \int_0^t  \| \bbtheta^s_i \|_2  \de s + Ke^{KT} T \sup_{s \in [0,t] } \vert \bba^s_i\vert + \Theta_{\infty} +  W_{\infty},
\end{align*}
which gives, after applying Gronwall's inequality with the bounds of Lemma \ref{lem:bound_noise_initialization_D} and \ref{lem:bound_a_trajectories_noisy_D}:
\begin{align*}
\P \Big( \Delta_i (t)  \leq K e^{KT}   \big[ \sqrt{\log N} + z   \big] \Big) \ge 1 - e^{-z^2}.
\end{align*}
Similarly:
\[
\Delta_0 (t)  \leq K \int_0^t \Delta_0 (s) \de s + K  e^{KT} \sup_{s \in [0,t]} (\| \bbara^s \|_{1}/N)+ \Theta_1 +  W_1,
\]
and thus:
\begin{equation}
\P \Big( \Delta_0 (t) \leq K e^{KT} [ 1 + z  ]  \Big) \ge 1 - e^{-z^2}.
\label{eq:delta_0_D}
\end{equation}

\noindent
{\bf Step 2.} Let us bound $\sup_{0 \leq u \leq \eps} \| \bbtheta_i^{k\eps + u} - \bbtheta_i^{k\eps} \|_2$:
\begin{align*}
\| \bbtheta_i^{k\eps + u} - \bbtheta_i^{k\eps} \|_2 & \leq \Big\| \int_{k\eps}^{k\eps +u} \xi (s) \bG(\bbtheta^s_i ; \rho_s)  \de s  \Big\|_2 + \| \obW_{i,k} (u) \|_2 \\
& \leq Kh \sup_{s \leq T} \left[ \lambda \| \bbtheta^s_i \|_2 + (1 + \sqrt{M_2}) (1 + \abs{\bba^s_i }) \right]  + \| \obW_{i,k} (u) \|_2 ,
\end{align*}
where we defined $\obW_{i,k} (u)  \equiv \int_{k\eps}^{k\eps +u} \sqrt{ \tau /D } \de \bW_i (s)$. By a similar computation as in Lemma \ref{lem:bound_noise_initialization_D}, we have
\[
\P \Big( \max_{i\leq N} \sup_{k \in [0, T/\eps] \cap \N} \sup_{0 \leq u \leq \eps} \| \obW_{i,k} (u) \|_2 \geq 4\sqrt{K \eps} \lb \sqrt{\log ( N (T/\eps \vee 1))} + z \rb \Big) \leq e^{-z^2} .
\]
Injecting this bound in the above inequality yields:
\begin{align*}
\P \Big( \max_{i\leq N} \sup_{k \in [0, T/\eps] \cap \N} \sup_{0 \leq u \leq \eps} \| \bbtheta_i^{k\eps + u} - \bbtheta_i^{k\eps} \|_2 \leq Ke^{KT} \big[ \sqrt{\log ( N (T/\eps \vee 1))} + z \big] \sqrt{\eps} \Big) \geq 1 - e^{-z^2}.
\end{align*}
Another useful bound can be obtained by taking the average over $i \in [N]$:
\begin{align*}
\frac{1}{N} \sum_{i=1}^N \| \bbtheta_i^{k\eps + u} - \bbtheta_i^{k\eps} \|_2  \leq & K \Delta_0(t) + Ke^{KT} \sup_{s \leq T}  \| \bbara^s \|_1   +  \frac{1}{N} \sum_{i=1}^N \| \obW_{i,k} (u) \|_2.
\end{align*}
We get by a similar computation as in Lemma \ref{lem:bound_noise_initialization_D}, we have
\[
\P \Big( \sup_{k \in [0, T/\eps] \cap \N} \sup_{0 \leq u \leq \eps} \frac{1}{N} \sum_{i=1}^N  \| \obW_{i,k} (u) \|_2 \geq 4\sqrt{K \eps} [ \sqrt{\log (T/\eps \vee 1)} + z ] \Big) \leq e^{-z^2} .
\]
We get the following bound:
\begin{equation}
\P \Big(  \sup_{k \in [0, T/\eps] \cap \N} \sup_{0 \leq u \leq \eps} \frac{1}{N} \sum_{i=1}^N \| \bbtheta_i^{k\eps + u} - \bbtheta_i^{k\eps} \|_2 \leq Ke^{KT} [ \sqrt{\log ( T/\eps \vee 1)} + z ] \sqrt{\eps} \Big) \geq 1 - e^{-z^2}.
\label{eq:average_increment_noisy_D}
\end{equation}

\noindent
{\bf Step 3.} We now bound $W_2 (\rho_t , \rho_{t+h} )$:
\[
W_2 (\rho_t , \rho_{t+h} )^2 \leq \E [ \| \bbtheta^t - \bbtheta^{t+h} \|_2^2 ] = \int_0^\infty  \P ( \| \bbtheta^t - \bbtheta^{t+h} \|_2^2 \geq u ) \de u .
\]
Using step 2, we have (where we removed the union bound over $i \in [N]$ and $k \in [0,T/\eps]\cap \N$):
\[
\P \Big( \|\bbtheta_i^{t + h} - \bbtheta_i^{t} \|_2 \geq  Ke^{KT} [ 1 + z ] \sqrt{h} \Big) \leq e^{-z^2}.
\]
Integrating this upper bound on the probability yields the desired inequality.
\end{proof}

\begin{lemma}\label{lem:RN_bound}
There exists a constant K, such that for $\btheta,\btheta' \in \R^{ND}$
\[
\vert R_N(\btheta) - R_N( \btheta' ) \vert \le K (1 +\|\ba \|_{1}/N  + \|\ba'\|_{1} /N + \| \ba' \|_1^2/N^2 ) \max_{i \in [N]} \| \btheta_i - \btheta_i '\|_2.
\]
\end{lemma}

\begin{proof}[Proof of Lemma \ref{lem:RN_bound}] We have
\begin{equation*}
\begin{aligned}
& \vert R_N(\btheta) - R_N( \btheta ) \vert \\ \leq &\frac{2}{N} \sum_{i=1}^N \abs{a_i  v ( \bw_i ) - a'_i v ( \bw'_i) } +\frac{1}{N^2} \sum_{i,j =1}^N \vert a_i a_j u(\bw_i, \bw_j) - a'_i a'_j u (\bw'_i, \bw'_j) \vert \\
  \leq & \frac{2}{N}  \sum_{i=1}^N K (\vert a'_i - a_i \vert +  \vert a'_i \vert \| \bw_i - \bw'_i \|_2 ) \\
  & + \frac{1}{N^2} \sum_{i,j =1}^N K \lb \vert a_i \vert \vert a_j - a'_j \vert 
 +  \vert a'_j \vert  \vert a_i - a'_i \vert + \vert a'_i a'_j \vert ( \| \bw_i - \bw'_i \|_2 + \| \bw_j - \bw'_j \|_2 )\rb \\
 \leq & K (1 +\|\ba \|_{1}/N  + \|\ba'\|_{1} /N + \| \ba' \|_1^2/N^2 ) \max_{i \in [N]} \| \btheta_i - \btheta_i '\|_2.
\end{aligned}
\end{equation*}

\end{proof}

\subsection{Bound between PDE and nonlinear dynamics}

\begin{proposition}[PDE-ND]\label{prop:PDE_ND_D}
There exists a constant $K$ such that
\[
\P \Big( \sup_{t \in [0, T]} \vert R_N(\bbtheta^t) - R(\rho_t) \vert \le Ke^{KT} \lb \log^{3/2} (NT) + z^3 \rb / \sqrt{N}  \Big) \geq 1 - e^{- z^2}.
\]
\end{proposition}

The proof will use the same decomposition in two terms as in the proof of proposition \ref{prop:PDE_ND_A}.

\begin{lemma}[Term ${\rm II}$ bound]\label{lem:error_II_bound_noisy_D}
We have 
\[
\vert \E R_N( \bar \btheta^t) - R(\rho_t) \vert \le Ke^{KT}/N.
\]
\end{lemma}
\begin{proof}[Proof of Lemma \ref{lem:error_II_bound_noisy_D}]
The bound hold simply by observing that
\[
\begin{aligned}
\vert \E R_N( \bar \btheta^t) - R(\rho_t) \vert & = \frac{1}{N} \Big \vert \int a^2 u (\bw , \bw ) \rho_t (\de \btheta) - \int a_1 a_2 u(\bw_1 , \bw_2 ) \rho_t(\de \btheta_1) \rho_t(\de \btheta_2) \Big\vert \\
& \le K/N \int a^2 \rho_t (\de a) \leq Ke^{KT} / N
\end{aligned}
\]
where we used the upper bound on the second moment of variable $a$ in Lemma \ref{lem:bound_asquared_noisy_D}.
\end{proof}

\begin{lemma}[Term ${\rm I}$ bound]\label{lem:error_I_bound_a_noisy_D}
There exists $K$, such that 
\[
\P \Big( \sup_{t \in [0, T]} \vert R_N( \bar \btheta^t ) - \E R_N( \bar \btheta^t ) \vert \le Ke^{KT} \lb \log (NT) + z^3 \rb / \sqrt{N} \Big) \ge 1 - e^{- z^2}. 
\]
\end{lemma}

\begin{proof}[Proof of Lemma \ref{lem:error_I_bound_a_noisy_D}] 
We have:
\[
\begin{aligned}
\abs{ R_N( \bar \btheta^t ) - \E R_N( \bar \btheta^t )  } \leq & 2 \Big\vert\frac{1}{N} \sum_{i = 1}^N \lb V (\bbtheta^t_i ) - \E V (\bbtheta^t_i) \rb \Big\vert + \frac{1}{N^2} \sum_{i = 1}^N \Big\vert U (\bbtheta_i^t , \bbtheta_i^t ) - \E_{\bbtheta_i^t} U(\bbtheta_i^t, \bbtheta_i^t) \Big\vert \\
& + \frac{1}{N} \sum_{i = 1}^N \Big\vert \frac{1}{N} \sum_{j = 1, j \neq i}^N \lb U (\bbtheta_i^t , \bbtheta_j^t ) - \E_{\bbtheta_j^t} U(\bbtheta_i^t, \bbtheta_j^t) \rb \Big\vert \\
& + \frac{1}{N} \sum_{i = 1}^N \Big\vert \frac{1}{N} \sum_{j = 1, j \neq i}^N \lb \E_{\bbtheta_j^t}  U (\bbtheta_i^t , \bbtheta_j^t ) - \E_{\bbtheta_i^t,\bbtheta_j^t} U(\bbtheta_i^t, \bbtheta_j^t) \rb \Big\vert .
\end{aligned}
\]
We will bound each of these terms separately. For any fixed $t$, we have $(\bbtheta^t_i)_{i \in [N]} \sim \rho_t$ independently. Define:
\[
Q_1 (t) = \Big\vert \frac{1}{N} \sum_{i = 1}^N \lb V (\bbtheta^t_i ) - \E V (\bbtheta^t_i) \rb \Big\vert,
\]
which is the absolute value of the sum of martingale differences. Furthermore, we can rewrite $V (\bbtheta^t_i) = \bara^t_i v(\bbw^t_i)$ which is $Ke^{KT}$-sub-Gaussian (product of a sub-Gaussian random variable, by Lemma \ref{lem:bound_asquared_noisy_D}, and a bounded random variable). We can therefore apply Azuma-Hoeffding's inequality (Lemma \ref{lem:Azuma}),
\[
\P \Big( Q_1 (t) \leq Ke^{KT}\lb 1 +z\rb/\sqrt{N} \Big) \geq 1 - e^{-z^2}.
\]
The second term is bounded as follow:
\[
\begin{aligned}
E_2 (t) \equiv \frac{1}{N^2} \sum_{i = 1}^N \Big\vert U (\bbtheta_i^t , \bbtheta_i^t ) - \E_{\bbtheta_i^t} U(\bbtheta_i^t, \bbtheta_i^t) \Big\vert & \le
\frac{1}{N^2} \sum_{i = 1}^N \vert (\bara^t_i)^2 u (\bbw_i^t , \bbw_i^t )\vert  + \frac{1}{N^2} \sum_{i = 1}^N \Big\vert \E_{\bbtheta_i^t} \lb ( \bara^t_i)^2 u (\bbw_i^t , \bbw_i^t ) \rb \Big\vert  \\
& \leq \frac{K}{N^2} \cdot \| \bbara^t \|_\infty \cdot \| \bbara^t\|_1 + \frac{Ke^{KT}}{N},
\end{aligned}
\]
where we used that $\int a^2 \rho_t (\de a) \le Ke^{KT}$. Using Lemma \ref{lem:bound_a_trajectories_noisy_D}, we get:
\[
\P \Big( E_2 (t) \leq Ke^{KT} \lb \sqrt{\log N} + z^2 \rb/N \Big) \geq 1 - e^{-z^2}.
\]
Define:
\[
Q_2^i (t) = \Big\vert \frac{1}{N} \sum_{j = 1, j \neq i}^N \lb U (\bbtheta_i^t , \bbtheta_j^t ) - \E_{\bbtheta_j^t} U(\bbtheta_i^t, \bbtheta_j^t) \rb \Big\vert.
\]
Because $\bbtheta_i^t$ is independent of the $(\bbtheta^t_j)_{j \in [N], j\neq i}$, we can condition on $\bbtheta_i^t$, and restrict ourselves to the event where $\bbtheta_i^t \le M_{\infty}$. $Q_2^i (t)$ is the absolute value of a sum of martingale difference, with $U (\bbtheta_i^t , \bbtheta_j^t )= \bara^t_i \bara^t_j u ( \bbw^t_i , \bbw^t_j)$ which is $K e^{KT} \vert \bara^t_i \vert^2 $-sub-Gaussian (product of a sub-Gaussian random variable and a bounded random variable). We apply Azuma-Hoeffding's inequality (Lemma \ref{lem:Azuma}),
\[
\begin{aligned}
&\P \Big(  Q_2^i (t) \geq Ke^{KT} M_{\infty} \lb 1 + z \rb / \sqrt{N} \Big) \\
\le &  \E_{\bbtheta_i^t} \lb \P \Big( Q_2^i (t) \geq Ke^{KT} M_{\infty}\lb 1 + z \rb/\sqrt{N} \Big\vert \bbtheta_i^t \Big)  \ones (\vert \bba_i^t \vert  \le M_{\infty}) \rb + \P ( \vert \bba_i^t \vert > M_\infty) \\
\le &  2 e^{-z^2}
\end{aligned}
\]
We take the union bound over $i \in [N]$ and get:
\[
\P \Big( \max_{i \in [N]} Q_2^i (t) \geq Ke^{KT} \lb \log N + z^2 \rb /\sqrt{N} \Big) \le e^{-z^2}.
\]
Define: 
\[
Q_3^i (t) = \Big\vert \frac{1}{N} \sum_{j = 1, j \neq i}^N \lb \E_{\bbtheta_j^t} U (\bbtheta_i^t , \bbtheta_j^t ) - \E_{\bbtheta_i^t, \bbtheta_j^t} U(\bbtheta_i^t, \bbtheta_j^t) \rb \Big\vert.
\]
We have:
\[
\E_{\bbtheta_j^t}  U (\bbtheta_i^t , \bbtheta_j^t ) = \bara^t_i \cdot \int a u ( \bbw^t_i , \bw ) \rho (\de \btheta) ,
\]
with $\Big\vert \int a u ( \bbw^t_i , \bw ) \rho (\de \btheta) \Big\vert \le K \Big( \int a^2 \rho_t (\de a ) \Big)^{1/2} \le Ke^{KT}$. Thus, $\E_{\bbtheta_j^t}  U (\bbtheta_i^t , \bbtheta_j^t ) $ is $Ke^{KT}$-sub-Gaussian (product of a sub-Gaussian random variables and of a bounded random variable). Applying Azuma-Hoeffding's inequality Lemma \ref{lem:Azuma}, followed by an union bound over $i \in [N]$, we get
\[
\P \Big( \max_{i \in [N]} Q_3^i (t) \geq Ke^{KT} \lb \sqrt{\log N} + z \rb /\sqrt{N} \Big) \le e^{-z^2}.
\]
Combining the above bounds with the bound on $\sup_{s \in [0,T]}\{ \| \bbara^s \|_1, \| \bbara^s \|_\infty\}$ of Lemma \ref{lem:bound_a_trajectories_noisy_D} yields:
\begin{equation}
\P \Big( \abs{ R_N( \bar \btheta^t ) - \E R_N( \bar \btheta^t )  }  \geq Ke^{KT} \lb \log N + z^2 \rb/\sqrt{N} \Big)  \le e^{-z^2}. \label{eq:proba_RN_PDE_ND_D}
\end{equation}
In order to extend this concentration uniformly on the interval $[0,T]$, we use the following result:

\begin{lemma}\label{lem:continuity_ND_RN_D}
There exists K, such that 
\[
\begin{aligned}
\sup_{k\in [0,T/\eta] \cap \N} \sup_{u \in [0,\eta]} & \Big \vert \vert R_N( \bar \btheta^{k\eta +u} ) - \E R_N( \bar \btheta^{k\eta +u} ) \vert  - \vert R_N( \bar \btheta^{k\eta} ) - \E R_N( \bar \btheta^{k\eta} ) \vert \Big \vert \\
& \le K e^{KT} \lb \sqrt{\log{ (N (T/\eta \vee 1))}}  +z^3 \rb \sqrt{\eta},
\end{aligned}
\]
with probability at least $1 - e^{-z^2}$.
\end{lemma}

\begin{proof}[Proof of Lemma \ref{lem:continuity_ND_RN_D}]
Consider $t,h \ge 0, t+h \le T$. From Lemma \ref{lem:RN_bound},
\[
\begin{aligned}
& \vert R_N(\bbtheta^{t+h}) - R_N( \bbtheta^t ) \vert \le K (1 +\|\bbara^{t+h}\|_{1} /N  + \|\bbara^t\|_{1} /N+ \| \bbara^{t+h} \|_1^2 /N^2) \max_{i \in [N]} \| \bbtheta^{t+h}_i - \bbtheta^t_i \|_2.
 \label{eq:bound_term_I_noisy_D}
\end{aligned}
\]
Using Lemma \ref{lem:bound_increment_noisy_D} without the union bound over $s \in \eta \{ 0, 1, \ldots, \lfloor T / \eta \rfloor \}$ and the bounds on $\sup_{t \in [0,T]} \{\| \bbara^t \|_1\}$ of Lemma \ref{lem:bound_a_trajectories_noisy_D}, we get
\[
\P \Big( \vert R_N( \bar \btheta^{t +h} ) - R_N( \bar \btheta^{t} ) \vert \ge Ke^{KT} \lb \sqrt{\log N} + z^3\rb \sqrt{h} \Big) \le e^{-z^2}.
\]
The difference in expectation, where the expectation is taken over $(\bbtheta_i)_{i \in [N]}$, is therefore bounded by
\[
\begin{aligned}
\vert \E R_N( \bar \btheta^{t +h}) - \E R_N( \bar \btheta^{t} ) \vert \le \E \vert R_N( \bar \btheta^{t +h}) -R_N( \bar \btheta^{t}) \vert \le \int_0^\infty \P \Big( \vert R_N( \bar \btheta^{t +h}) -R_N( \bar \btheta^{t}) \vert \ge u \Big) \de u.
\end{aligned}
\]
Doing a change of variable, we get:
\[
\begin{aligned}
\vert \E R_N( \bar \btheta^{t +h}) - \E R_N( \bar \btheta^{t} ) \vert & \le Ke^{KT}\sqrt{h\log N}+\int_{0}^\infty e^{-z^2} Ke^{KT} \sqrt{h} z^2 \de z \\
& \le Ke^{KT} (\sqrt{\log N}+1)\sqrt{h}.
\end{aligned}
\]
Hence using that
\[
\begin{aligned}
  \Big\vert \vert R_N( \bar \btheta^{t +h} ) - \E R_N( \bar \btheta^{t +h} ) \vert - \vert R_N( \bar \btheta^{t} ) - \E R_N( \bar \btheta^{t} ) \vert \Big\vert  \leq  \vert R_N( \bar \btheta^{t +h} ) - R_N( \bar \btheta^{t} ) \vert  + \vert \E R_N( \bar \btheta^{t +h}) - \E R_N( \bar \btheta^{t} ) \vert, 
\end{aligned}
\]
with Lemma \ref{lem:bound_increment_noisy_D}, we get
\[
\begin{aligned}
\sup_{k\in [0,T/\eta] \cap \N} \sup_{u \in [0,\eta]} & \Big \vert \vert R_N( \bar \btheta^{k\eta +u} ) - \E R_N( \bar \btheta^{k\eta +u} ) \vert  - \vert R_N( \bar \btheta^{k\eta} ) - \E R_N( \bar \btheta^{k\eta} ) \vert \Big \vert \\
& \le K e^{KT} \lb \sqrt{\log{ (N (T/\eta \vee 1))}}  +z^3 \rb \sqrt{\eta},
\end{aligned}
\]
with probability at least $1 - e^{-z^2}$.
\end{proof}

Taking an union bound over $s \in \eta \{ 0, \ldots, \lfloor T / \eta \rfloor \}$ in Eq. \eqref{eq:proba_RN_PDE_ND_D} and bounding the variation inside the grid intervals, we get
\[
\begin{aligned}
\P\Big( \sup_{t \in [0, T]} \vert R_N( \bar \btheta^t) - \E R_N( \bar \btheta^t) \vert \ge  &K e^{KT} \lb \log N + z^2 \rb/\sqrt{N} + K e^{KT} \lb \sqrt{\log{ (N (T/\eta \vee 1))}}  +z^3 \rb \sqrt{\eta} \Big)\\
&\le (T / \eta) \exp\{ - z^2 \}.
\end{aligned}
\]
Taking $\eta = 1/(N\log N)$ and $z = [\sqrt{\log(NT\log N)} + z']$ concludes the proof. 
\end{proof}

\subsection{Bound between nonlinear dynamics and particle dynamics}

\begin{proposition}[ND-PD]\label{prop:ND_PD_D}
There exists a constant $K$, such that with probability at least $1 - e^{-z^2}$, we have 
\begin{align*}
 \sup_{t \in [0,T]} \max_{i \in [N]} \| \bbtheta_i^t - \ubtheta_i^t \|_2  \le& Ke^{e^{KT}[\sqrt{\log N} +z^2]} \lb \sqrt{D\log N} + \log^{3/2} (NT) + z^3 \rb/\sqrt{N}, \\
\sup_{t \in [0, T]} \vert R_N(\ubtheta^t) - R_N( \bar \btheta^t ) \vert \le& K e^{e^{KT}\lb \sqrt{\log N } + z^2 \rb} \lb \sqrt{D\log N} + \log^{3/2} (NT) + z^5 \rb/\sqrt{N} .
\end{align*}
\end{proposition}

\begin{proof}[Proof of Proposition \ref{prop:ND_PD_D}]
Define $\Delta(t) \equiv \sup_{s\le t} \max_{i \in [N]} \| \bbtheta_i^s - \ubtheta_i^s \|_2$. We have
\begin{align}
 \| \ubtheta^t_i - \bbtheta^t_i \|_2 \leq  & \int_0^t  \| \bG ( \bbtheta^s_i ; \rho_s ) - \bG ( \ubtheta^s_i ;  \urho^{(N)}_s )\|_2 \de s\nonumber\\
\le&  \int_0^t \lambda  \| \bbtheta_i^s -  \ubtheta_i^s \|_2 \de s + \int_0^t \| \nabla V (\bbtheta_i^s) - \nabla V(\ubtheta_i^s) \|_2 \de s \nonumber\\
& + \int_0^t \Big\| \frac{1}{N}\sum_{j=1}^N \nabla_1 U( \bbtheta_i^s, \bbtheta_j^s) - \nabla_1 U(\ubtheta_i^s, \ubtheta_j^s) \Big\|_2 \de s\nonumber\\
& + \int_0^t \Big\| \frac{1}{N}\sum_{j = 1}^N  \nabla_1 U( \bbtheta_i^s, \bbtheta_j^s) - \int \nabla_1 U(\bbtheta_i^s, \btheta) \rho_s(\de \btheta) \Big\|_2 \de s. \label{eq:ThetaDiff}
\end{align}
Let us bound each term separately. We have
\[
\begin{aligned}
 \| \nabla V (\bbtheta_i^s) - \nabla V(\ubtheta_i^s) \|_2
\le& \vert v(\bbw_i^s) - v(\ubw_i^s) \vert + \| \bara_i^s \nabla v(\bbw_i^s) - \undera_i^s \nabla v(\ubw_i^s) \|_2 \\
\le& K ( \| \bbw_i^s - \ubw_i^s \|_2 + \vert \bara_i^s - \undera_i^s \vert + \vert \bara_i^s \vert \|\bbw_i^s - \ubw_i^s \|_2) \\
\le& K (1 + \| \bbara^s \|_\infty ) \| \bbtheta_i^s - \ubtheta_i^s \|_2.
\end{aligned}
\]
We decompose the second term into two terms
\[
\begin{aligned}
\Big\| \frac{1}{N}\sum_{j=1}^N \nabla_1 U( \bbtheta_i^s, \bbtheta_j^s) - \nabla_1 U(\ubtheta_i^s, \ubtheta_j^s)  \Big\|_2 \le & \Big\| \frac{1}{N}\sum_{j=1}^N \nabla_1 U( \bbtheta_i^s, \bbtheta_j^s) - \nabla_1 U(\bbtheta_i^s, \ubtheta_j^s)  \Big\|_2 \\
&+ \Big\| \frac{1}{N}\sum_{j=1}^N \nabla_1 U( \bbtheta_i^s, \ubtheta_j^s) - \nabla_1 U(\ubtheta_i^s, \ubtheta_j^s)  \Big\|_2 ,
\end{aligned}
\]
where
\[
\begin{aligned}
& \Big\| \frac{1}{N}\sum_{j=1}^N \nabla_1 U( \bbtheta_i^s, \bbtheta_j^s) - \nabla_1 U(\bbtheta_i^s, \ubtheta_j^s)  \Big\|_2  \\
\le& \Big\vert   \frac{1}{N}\sum_{j=1}^N \bara_j^s  u( \bbw_i^s, \bbw_j^s) - \undera_j^s u(\bbw_i^s, \ubw_j^s) \Big\vert + \Big\|  \frac{1}{N}\sum_{j=1}^N  \bara_i^s \bara_j^s  \nabla_1 u( \bbw_i^s, \bbw_j^s) - \bara_i^s \undera_j^s \nabla_1 u(\bbw_i^s, \ubw_j^s) \Big\|_2\\
\le&  K(1 + \vert \bara_i^s \vert) \Big[ \max_{j \in [N]} \vert \bara_j^s - \undera_j^s \vert + \Big[\frac{1}{N}\sum_{j=1}^N \vert \bara_j^s \vert \Big] \max_{j \in [N]} \| \bbw_j^s - \ubw_j^s \|_2 \Big]\\
\le&  K(1 + \| \bbara^s \|_{\infty} ) \cdot (1 + \| \bbara^s \|_1/N ) \cdot  \max_{j \in [N]} \| \bbtheta_j^s - \ubtheta_j^s \|_2 , \\
\end{aligned}
\]
and
\[
\begin{aligned}
&\Big \| \frac{1}{N}\sum_{j=1}^N \nabla_1 U( \bbtheta_i^s, \ubtheta_j^s) - \nabla_1 U(\ubtheta_i^s, \ubtheta_j^s)  \Big\|_2\\
\le& \Big\vert   \frac{1}{N}\sum_{j=1}^N \undera_j^s  u( \bbw_i^s, \ubw_j^s) - \undera_j^s u(\ubw_i^s, \ubw_j^s) \Big\vert + \Big\|  \frac{1}{N}\sum_{j=1}^N  \bara_i^s \undera_j^s  \nabla_1 u( \bbw_i^s, \ubw_j^s) - \undera_i^s \undera_j^s \nabla_1 u(\ubw_i^s, \ubw_j^s) \Big\|_2\\
\le&  \Big[\frac{K}{N}\sum_{j=1}^N \vert \undera_j^s \vert \Big] \sup_{j \in [N]} \| \bbw_j^s - \ubw_j^s \|_2 + K\vert \bara_i^s - \undera_i^s \vert \Big[\frac{1}{N}\sum_{j=1}^N \vert \undera_j^s \vert \Big] + K\vert \undera_i^s \vert \Big[\frac{1}{N}\sum_{j=1}^N \vert \undera_j^s \vert \Big] \| \bbw_i^s - \ubw_i^s \|_2 \\
\le& K(1 + \| \bundera^s \|_{\infty} ) \cdot (1 + \| \bundera^s \|_1 /N ) \cdot  \max_{j \in [N]} \| \bbtheta_j^s - \ubtheta_j^s \|_2.
\end{aligned}
\]
The last term in Eq.~\eqref{eq:ThetaDiff} can be decomposed into two terms. Consider $j=i$:
\[
\begin{aligned}
&\frac{1}{N} \| \nabla_1 U( \bbtheta_i^s, \bbtheta_i^s) - \int \nabla_1 U(\bbtheta_i^s, \btheta) \rho_s(\de \btheta) \|_2\\
\le & \frac{1}{N } \| \nabla_1 U( \bbtheta_i^s, \bbtheta_i^s) \|_2 + \frac{1}{N} \int \| \nabla_1 U(\bbtheta_i^s, \btheta) \|_2 \rho_s(\de \btheta) \\
\le& \frac{1}{N} \lb \vert \bara_i^s u(\bbw_i^s, \bbw_i^s) \vert + \| (\bara_i^s)^2 \nabla_1 u(\bbw_i^s, \bbw_i^s)  \|_2 \rb + \int \lb \vert a u(\bbw_i^s, \bw)\vert + \| \bara_i^s  a \nabla_1 u(\bbw_i^s, \bw) \|_2 \rb  \rho_s (\de \btheta) \\
\le & \frac{1}{N} K \| \bbara^s \|_\infty \cdot (1+ \| \bbara^s \|_\infty ) + K e^{KT} (1+ \| \bbara^s \|_\infty ),
\end{aligned}
\]
where we used that $\int \vert a \vert \rho_s ( \de \btheta ) \le \Big( \int a^2 \rho_s ( \de \btheta ) \Big)^{1/2}$ and Lemma \ref{lem:bound_asquared_noisy_D}. We consider $j \neq i$ and denote:
\[
Q^i (s ) = \Big\| \frac{1}{N} \sum_{j=1, j \neq i}^N \lb \na_1 U (\bbtheta^s_i , \bbtheta^s_j)  - \int \na_1 U (\bbtheta^s_i , \bbtheta) \rho_s ( \de \btheta) \rb \Big\|_2,
\]
which is bounded in the following lemma:

\begin{lemma}\label{lem:bound_Q_D}
There exists a constant K, such that:
\[
\P \Big( \sup_{s \in [0,T]  } \max_{i\leq N}  Q^i (s) \ge Ke^{KT} \lb \sqrt{D\log N} + \log^{3/2} (NT) + z^3 \rb/\sqrt{N} \Big) \le e^{-z^2}.
\]
\end{lemma}
\begin{proof}[Proof of Lemma \ref{lem:bound_Q_D}]
The concentration of $Q^i (s)$ follows from a similar method as in the proof of Lemma \ref{lem:error_I_bound_a_noisy_D}. For any fixed $s$, we have $(\bbtheta^s_i)_{i \in [N]} \sim \rho_s$ independently. In particular, we have
\[
\int \na_1 U (\bbtheta^s_i , \bbtheta) \rho_s ( \de \btheta)  = \E \lb \na_1 U (\bbtheta^s_i , \bbtheta^s_j) \Big\vert \bbtheta^s_i \rb ,
\]
and $ Q^i (s) $ conditioned on $\bbtheta^s_i$ is the norm of a martingale difference sum. We furthermore restrict ourselves to the event where $\bba_i^s \le M_{\infty}$. We have $\na_1 U (\bbtheta_i^s , \bbtheta_j^s ) = \bara^t_j  \cdot ( u ( \bbw^t_i , \bbw^t_j), \bara^s_i \na_1 u ( \bbw^t_i , \bbw^t_j) )$ which is $Ke^{KT}M_{\infty}^2$-sub-Gaussian (the product of a sub-Gaussian random variable and a bounded random variable is sub-Gaussian). We can therefore apply Azuma-Hoeffding 's inequality (Lemma \ref{lem:Azuma}),
\[
\begin{aligned}
& \P \Big(  Q^i ( s) \geq Ke^{KT} M_{\infty} \lb \sqrt{D} + z \rb/\sqrt{N} \Big) \\
\le&  \E_{\bbtheta_i^t} \lb \P \Big( Q^i (s) \geq Ke^{KT} M_{\infty}\lb \sqrt{D} + z \rb/\sqrt{N} \Big\vert \bbtheta_i^s \Big)  \ones (\vert \bba_i^s \vert  \le M_{\infty}) \rb + \P (\| \bbara^s \|_{\infty}  \ge M_{\infty}) \\
 \le & 2 e^{-z^2}.
\end{aligned}
\]
Taking the union bound over the $i \in [N]$
\[
\P \lp \max_{i\leq N} Q^i ( s) \geq Ke^{KT} \lb \sqrt{D\log N} + \log (N) + z^2 \rb/\sqrt{N} \rp \leq e^{-z^2}.
\]
Furthermore, let us consider $t,h \ge 0, t+h \le T$:
\[
\begin{aligned}
& \frac{1}{N} \sum_{j=1, j \neq i}^N \| \na_1 U (\bbtheta^{t+h}_i , \bbtheta^{t+h}_j) - \na_1 U (\bbtheta^t_i , \bbtheta^t_j) \|_2 \\
\le&  \frac{1}{N} \sum_{j=1, j \neq i}^N  \lb \vert \bba^{t+h}_j u (\bbw^{t+h}_i ,\bbw^{t+h}_j ) -   \bba^{t}_j u (\bbw^{t}_i ,\bbw^{t}_j ) \vert + \|  \bba^{t+h}_i \bba^{t+h}_j \na_1 u (\bbw^{t+h}_i ,\bbw^{t+h}_j ) -   \bba^{t}_i \bba^{t}_j \na_1 u (\bbw^{t}_i ,\bbw^{t}_j ) \|_2  \rb \\
\le & K (1+\norm{\bbara^t}_{\infty}) \cdot (1+\norm{\bbara^t}_{1} /N ) \cdot \sup_{i \leq N} \|\bbtheta^{t+h}_i - \bbtheta^{t}_i  \|_2. 
\end{aligned}
\]
Considering Lemma \ref{lem:bound_increment_noisy_D} without the union bound over $s \in \eta \{ 0, 1, \ldots, \lfloor T / \eta \rfloor \}$ and the high probability bounds on $\sup_{t \in [0,T]} \{ \| \bbara^t \|_{\infty}, \| \bbara^t \|_{1} \}$ of Lemma \ref{lem:bound_a_trajectories_noisy_D}, we get:
\[
\P \Big( \frac{1}{N} \sum_{j=1, j \neq i}^N \| \na_1 U (\bbtheta^{t+h}_i , \bbtheta^{t+h}_j) - \na_1 U (\bbtheta^t_i , \bbtheta^t_j) \|_2  \ge Ke^{KT} (1+z) \lb \sqrt{\log N} + z \rb^2 \sqrt{h} \Big) \le e^{-z^2}.
\]
 The difference in expectation, where the expectation is taken over $\bbtheta_j$, is bounded by
\[
\begin{aligned}
\| \E \na_1 U (\bbtheta^{t+h}_i , \bbtheta^{t+h}_j) - \E \na_1 U (\bbtheta^t_i , \bbtheta^t_j) \|_2 & \le \E \Big[ \frac{1}{N} \sum_{j=1, j \neq i}^N \| \na_1 U (\bbtheta^{t+h}_i , \bbtheta^{t+h}_j) - \na_1 U (\bbtheta^t_i , \bbtheta^t_j) \|_2 \Big]  \\
&\le \int_0^\infty \P \Big( \frac{1}{N} \sum_{j=1, j \neq i}^N \| \na_1 U (\bbtheta^{t+h}_i , \bbtheta^{t+h}_j) - \na_1 U (\bbtheta^t_i , \bbtheta^t_j) \|_2  \ge u \Big) \de u .
\end{aligned}
\]
Noticing that $(1+z) \lb \sqrt{\log N} + z \rb^2 \le ( \sqrt{\log N} + z )^3$ and doing a change of variable, we get:
\[
\begin{aligned}
\| \E \na_1 U (\bbtheta^{t+h}_i , \bbtheta^{t+h}_j) - \E \na_1 U (\bbtheta^t_i , \bbtheta^t_j) \|_2 
& \le Ke^{KT}\log N\, \sqrt{h}+\int_{-\sqrt{\log N}}^\infty e^{-z^2} Ke^{KT}  z^2 \sqrt{h} \de z  \\
& \le Ke^{KT} (\log N + 1) \sqrt{h}.
\end{aligned}
\]
Hence using that
\[
\begin{aligned}
  \vert Q^i(t+h) - Q^i (t) \vert  \leq  & \frac{1}{N} \sum_{j=1, j \neq i}^N \| \na_1 U (\bbtheta^{t+h}_i , \bbtheta^{t+h}_j) - \na_1 U (\bbtheta^t_i , \bbtheta^t_j) \|_2   \\
  & + \| \E \na_1 U (\bbtheta^{t+h}_i , \bbtheta^{t+h}_j) - \E \na_1 U (\bbtheta^t_i , \bbtheta^t_j) \|_2, 
\end{aligned}
\]
and the bounds derived above, with an union bound over $t \in \eta \{ 0, 1, \ldots, \lfloor T / \eta \rfloor \}$, we get
\[
\P \Big( \sup_{k\in [0,T/\eta] \cap \N} \sup_{u \in [0,\eta]} \max_{i\in [N]}   \vert Q^i (k \eta + u) - Q^i (k\eta) \vert \le K e^{KT} \lb \log{ (N (T/\eta \vee 1))}  +z^3 \rb \sqrt{\eta}\Big) \ge 1 - e^{-z^2}.
\]
We can therefore take the supremum over the interval $[0,T]$ :
\[
\begin{aligned}
& \P \Big( \max_{i\leq N} \sup_{s \in [0,T]  } Q^i (s) \ge Ke^{KT} \lb \sqrt{D\log N} + \log (N) + z^2 \rb/\sqrt{N}  +  K e^{KT} \lb \log{ (N (T/\eta \vee 1))}  +z^3 \rb \sqrt{\eta} \Big) \\
&\le (T / \eta) \exp\{ - z^2 \}.
\end{aligned}
\]
Taking $\eta = 1/N$ and $z = [\sqrt{\log(NT)} + z']$:
\[
\P \Big( \max_{i\leq N} \sup_{s \in [0,T]  } Q^i (s) \ge Ke^{KT} \lb \sqrt{D\log N} + \log^{3/2} (NT) + z^3 \rb/\sqrt{N} \Big) \le e^{-z^2}.
\]
\end{proof}

Using the high probability bound on $\sup_{s \in [0,T]}\{ \| \bbara^s \|_1 / N, \| \bbara^s \|_\infty\}$ of Lemma \ref{lem:bound_a_trajectories_noisy_D}, we get with probability at least $1 - e^{-z^2}$ that for all $t \in [0,T]$
\[
\begin{aligned}
\Delta ( t ) \leq & Ke^{KT} (1+z) \lb \sqrt{\log N} + z \rb \int_0^t \Delta ( s ) \de s  + TKe^{KT} \lb \sqrt{\log N} +z \rb^2 /N \\
& +   TKe^{KT} \lb \sqrt{D\log N} + \log^{3/2} (NT) + z^3 \rb/\sqrt{N}  .
\end{aligned}
\]
Applying Gronwall's inequality, we get:
\[
\P \Big( \Delta ( T )  \leq Ke^{e^{KT}[\sqrt{\log N} +z^2]} \lb \sqrt{D\log N} + \log^{3/2} (NT) + z^3 \rb/\sqrt{N} \Big) \ge 1 - e^{-z^2}.
\]
Using Lemma \ref{lem:RN_bound} and the high probability bounds on $\sup_{t \in [0,T]} \{ \| \bbara^t \|_1 /N , \| \bbara^t \|_\infty , \| \btilda^t \|_1 /N , \|\btilda^t \|_\infty\}$ of Lemma \ref{lem:bound_a_trajectories_noisy_D} concludes the proof.
\end{proof}

\subsection{Bound between particle dynamics and GD}

\begin{proposition}[PD-GD]\label{prop:PD_GD_D}
There exists constant $K$, such that with probability at least $1 - e^{-z^2}$, we have 
\begin{align*}
\sup_{k \in [0,T/\eps] \cap \N} \max_{i \in [N]} \| \ubtheta_i^{k\eps} - \tbtheta_i^k \|_2  \leq& Ke^{e^{KT}[\sqrt{\log N} +z^2]} \lb \log ( N (T/\eps \vee 1))  + z^4 \rb \sqrt{\eps}, \\
\sup_{k \in [0,T/\eps] \cap \N}  \vert R_N (\ubtheta^{k\eps}) - R_N (\tbtheta^{k}) \vert \leq& Ke^{e^{KT}[\sqrt{\log N} +z^2]} \lb \log ( N (T/\eps \vee 1))  + z^6 \rb \sqrt{\eps}.
\end{align*}
\end{proposition}
\begin{proof}[Proof of Proposition \ref{prop:PD_GD_D}]
Denote $\Delta (t) \equiv \sup_{k \in [0,t/\eps]\cap \N} \max_{i \in [N]} \| \ubtheta_i^{k\eps} - \tbtheta_i^k \|_2 $. For $k \in \N$ and $t = k \eps$,
\[
\begin{aligned}
\| \ubtheta_i^{t} - \tbtheta_i^{k} \|_2 \leq &   \int_0^{t}  \|  \bG ( \ubtheta^s_i ; \urhoN_s )  -  \bG ( \tbtheta^{[s]/\eps}_i ; \trhoN_{[s]/\eps} ) \|_2 \de s \\
 \leq & \int_0^{t} \| \bG ( \ubtheta^s_i ; \urhoN_s )  - \bG ( \ubtheta^{[s]}_i ; \urhoN_{[s]} )\|_2 \de s \\
& + \int_0^t \|\bG ( \ubtheta^{[s]}_i ; \urhoN_{[s]} )  - \bG ( \tbtheta^{[s]/\eps}_i ; \Tilde \rho^{(N)}_{[s]/\eps} ) \|_2 \de s .
\end{aligned}
\]
Let us consider each terms separately:
\[
\begin{aligned}
& \|\bG ( \ubtheta^s_i , \urhoN_s ) -\bG ( \ubtheta^{[s]}_i ; \urhoN_{[s]} )\|_2 \\ \leq & \lambda \|\ubtheta^s_i  -  \ubtheta^{[s]}_i  \|_2 + \| \na V (\ubtheta^s_i ) - \na V (\ubtheta^{[s]}_i ) \|_2  + \Big\| \frac{1}{N } \sum_{j=1}^N \na_1 U ( \ubtheta^s_i , \ubtheta^s_j ) - \na_1 U ( \ubtheta^{[s]}_i , \ubtheta^{[s]}_j ) \Big\|_2 \\
\le & K (1 + \| \bundera^s \|_\infty ) \cdot \| \ubtheta^s_i - \ubtheta^{[s]}_i \|_2 + K (1 + \| \bundera^s \|_\infty ) \cdot (1 + \| \bundera^s \|_1 /N ) \cdot \max_{j \in [N]} \| \ubtheta^s_i - \ubtheta^{[s]}_i \|_2  \\
&+ K (1 + \| \bundera^{[s]} \|_\infty ) \cdot (1 + \| \bundera^{[s]}  \|_1 /N ) \cdot \max_{j \in [N]} \| \ubtheta^s_j - \ubtheta^{[s]}_j\|_2.
\end{aligned}
\]
From Lemma \ref{lem:bound_increment_noisy_D}, we know that 
\[
\P \Big( \sup_{i \leq N} \sup_{k\in [0,T/\eps] \cap \N} \sup_{u \in [0,\eps]} \| \ubtheta^{k\eps +u}_i - \ubtheta^{k\eps}_i  \|_2 \leq Ke^{KT} \lb \sqrt{\log ( N (T/\eps \vee 1))} + z^2 \rb \sqrt{\eps} \Big) \leq 1 - e^{-z^2},
\]
which combined with the upper bound on $\sup_{s \in [0,T]} \{ \| \bundera^s \|_1 / N, \| \bundera^s \|_\infty \}$ of Lemma \ref{lem:bound_a_trajectories_noisy_D}, shows that with probability at least $1 - e^{-z^2}$, we have
\[
\int_0^{k\eps} \| \bG ( \ubtheta^s_i ; \urhoN_s )  - \bG ( \ubtheta^{[s]}_i ; \urhoN_{[s]} )\|_2 \de s \le KT e^{KT} \lb \log ( N (T/\eps \vee 1)) + z^4 \rb \sqrt{\eps} .
\]
Consider the second term:
\[
\begin{aligned}
& \|\bG ( \tbtheta^{[s]/\eps}_i , \trhoN_{[s]/\eps} ) -\bG ( \ubtheta^{[s]}_i ; \urhoN_{[s]} )\|_2 \\ \leq & \lambda \|\tbtheta^{[s]/\eps}_i  -  \ubtheta^{[s]}_i  \|_2 + \| \na V (\tbtheta^{[s]/\eps}_i ) - \na V (\ubtheta^{[s]}_i ) \|_2  + \Big\| \frac{1}{N } \sum_{j=1}^N \na_1 U ( \tbtheta^{[s]/\eps}_i ,\tbtheta^{[s]/\eps}_j) - \na_1 U ( \ubtheta^{[s]}_i , \ubtheta^{[s]}_j ) \Big\|_2 \\
\le & K (1 + \| \btilda^{[s]} \|_\infty ) \cdot \| \tbtheta^{[s]/\eps}_i - \ubtheta^{[s]}_i \|_2 + K (1 + \| \btilda^{[s]/\eps} \|_\infty ) \cdot (1 + \| \btilda^{[s]/\eps} \|_1 /N ) \cdot \max_{j \in [N]} \| \tbtheta^{[s]/\eps}_i - \ubtheta^{[s]}_i \|_2  \\
&+ K (1 + \| \bundera^{[s]} \|_\infty ) \cdot (1 + \| \bundera^{[s]}  \|_1 /N ) \cdot \max_{j \in [N]} \| \tbtheta^{[s]/\eps}_j - \ubtheta^{[s]}_j\|_2.
\end{aligned}
\]
Using the high probability bound on $\sup_{k \in [0,T/\eps] \cap \N}\{ \| \bundera^{k\eps} \|_1 /N , \| \bundera^{k\eps} \|_\infty, \| \btilda^{k} \|_1 /N, \| \btilda^{k} \|_\infty\}$ of Lemma \ref{lem:bound_a_trajectories_noisy_D}, we get with probability at least $1 - e^{-z^2}$ that for all $t \in [0,T]$
\[
\begin{aligned}
\Delta ( t ) \leq & Ke^{KT} (1+z) \lb \sqrt{\log N} + z \rb \int_0^t \Delta ( s ) \de s  + K e^{KT} \lb \log ( N (T/\eps \vee 1)) + z^4 \rb \sqrt{\eps}.
\end{aligned}
\]
Applying Gronwall's inequality, we get with probability at least $1 - e^{-z^2}$,
\[
\P \Big( \Delta ( T )  \leq Ke^{e^{KT}[\sqrt{\log N} +z^2]} \lb \log ( N (T/\eps \vee 1)) + z^4 \rb \sqrt{\eps} \Big) \ge 1 - e^{-z^2}.
\]
This bound combined with Lemma \ref{lem:RN_bound} concludes the proof.
\end{proof}

\subsection{Bound between GD and SGD}

\begin{proposition}[GD-SGD]\label{prop:GD_SGD_D}
There exists $K$, such that with probability at least $1 - e^{-z^2}$, we have 
\begin{align*}
\sup_{k \in [0,T/\eps] \cap \N} \max_{i \in [N]} \| \tbtheta_i^k - \btheta_i^k \|_2  \leq& Ke^{e^{KT}[\sqrt{\log N} +z^2]} \lb \sqrt{D}\log N + \log^{3/2} N + z^3 \rb \sqrt{\eps }  , \\
\sup_{k \in [0,T/\eps] \cap \N}  \vert R_N (\tbtheta^{k}) - R_N (\btheta^{k}) \vert \leq& Ke^{e^{KT}[\sqrt{\log N} +z^2]} \lb \sqrt{D}\log N + \log^{3/2} N + z^5 \rb \sqrt{\eps }  .
\end{align*}
\end{proposition}

\begin{proof}[Proof of Proposition \ref{prop:GD_SGD_D}]
Define $\Delta (t) \equiv \sup_{k \in [0,t/\eps]\cap \N} \max_{i \in [N]} \| \tbtheta^k_i - \btheta^k_i \|_2 $. Denote the generated $\sigma$-algebra:
\[
\cF_k = \sigma ( ( \btheta^0_i )_{i \in [N]}, \{ \bW_i (s) \}_{i \in [N],s\le k \eps}, \bz_1 , \ldots , \bz_{k} ).
\]
We get:
\[
\E [ \bF_i ( \btheta^k ; \bz_{k+1} ) \vert \cF_k ] = - \lambda \btheta^k_i - \na V (\btheta^k_i ) - \frac{1}{N} \sum_{j = 1}^N \na_1 U ( \btheta_i^k , \btheta_j^k ) = \bG ( \btheta^k_i , \rho^{(N)}_k ),
\]
where we denoted $\rho^{(N)}_k \equiv (1/N) \sum_{i \in [N]} \delta_{\btheta^k_i}$ the particle distribution of SGD. Hence we get
\[
\begin{aligned}
\| \btheta^k_i - \tbtheta^k_i \|_2 =& \Big\|  \eps \sum_{l = 0}^{k-1} \bF_i ( \btheta_i^l ; \bz_{l+1} ) - \eps \sum_{l = 0}^{k-1} \bG ( \tbtheta^l_i ; {\Tilde \rho}^{(N)}_l ) \Big \|_2 \\
\leq & \Big \| \eps \sum_{l = 0}^{k-1} \bZ_i^l \Big\|_2 +  \eps \sum_{l = 0}^{k-1} \Big\| \bG ( \btheta^l_i; \rho^{(N)}_l ) -\bG ( \tbtheta^l_i; {\Tilde \rho}^{(N)}_l ) \Big\|_2 \\
\leq & A_i^k + B_i^k,
\end{aligned}
\]
where we denoted $\bZ_i^l \equiv \bF_i ( \btheta^l ; \bz_{l+1} ) - \E [ \bF_i ( \btheta^l ; \bz_{l+1}) \vert \cF_{l}]$ and $A_i^k = \| \eps \sum_{l=0}^{k-1} \bZ^l_i \|_2 $.

Denote $\bA_i^k = \sum_{l=0}^{k-1} \eps \bZ_i^l$. Hence $\{ \bA_i^k \}_{k \in \N}$ is a martingale adapted to $\{ \cF_k\}_{k \in \N}$. Note the regularization term cancels out. We have component-wise
\[
\begin{aligned}
\bZ_i^k  = &\Big( (y^{k+1} - \hat y(\bx^{k+1}; \btheta^k)) \sigma(\bx^{k+1}; \bw_i^k) -  \E \lb (y^{k+1} - \hat y(\bx^{k+1}; \btheta^k)) \sigma(\bx^{k+1} ; \bw_i^k) \vert \cF_k \rb , \\
&(y^{k+1} - \hat y(\bx^{k+1}; \btheta^k)) a_i^k \nabla_\bw \sigma(\bx^{k+1}; \bw_i^k) ) - \E \lb (y^{k+1} - \hat y(\bx^{k+1}; \btheta^k)) a_i^k \nabla_\bw \sigma(\bx^{k+1}; \bw_i^k) \vert \cF_k \rb \Big).
\end{aligned}
\]
The following discussion is under the conditional law $\cL(\cdot \vert \cF_k)$. Note $\vert \sigma(\bx^{k+1}; \bw_i^k) \vert \le K$, and $\vert y^{k+1} -  \hat y(\bx^{k+1}; \btheta^k) \vert \le K (1 + \| \ba^k \|_1 /N )$, hence $(y^{k+1} - \hat y(\bx^{k+1}; \btheta^k)) \sigma(\bx^{k+1}; \bw_i^k) $ is $K (1 + \| \ba^k \|_1 /N  )^2$-sub-Gaussian. Note that by assumption, $\nabla_\bw \sigma(\bx^{k+1}; \bw_i^k)$ is $K$-sub-Gaussian (random vector), and $\vert(y^{k+1} - \hat y(\bx^{k+1}; \btheta^k)) a_i^k \vert \le K (1 +  \| \ba^k \|_1 /N ) \| \ba^k \|_\infty$, hence $(y^{k+1} - \hat y(\bx^{k+1}; \btheta^k)) a_i^k \nabla_\bw \sigma(\bx^{k+1}; \bw_i^k)$ is a $K (1 +  \| \ba^k \|_1 /N )^2 \| \ba^k \|_\infty^2$-sub-Gaussian random vector. As a result, we have $\bF_k ( \btheta^k ; \bz_{k+1})$ under the conditional law $\cL(\cdot \vert \cF_k)$ is a $K (1 +  \| \ba^k \|_1 / N )^2 \| \ba^k \|_\infty^2$-sub-Gaussian random vector.. 

Let $\tau \equiv \inf \lbrace k \vert \Vert \ba^{k} \Vert_{\infty} \geq M_{\infty} \text{ or }  \Vert \ba^{k} \Vert_{1}  \geq N \cdot M_{1} \rbrace$. Notice that $
\bA^{t\wedge \tau}_i - \bA_i^{t \wedge \tau -1} = \bZ_i^{k \wedge \tau -1}$. Following the same argument as in the proof of Proposition \ref{prop:GD_SGD_B}, we deduce that for $\overline{\bA}^k_i \equiv \bA^{k \wedge \tau}_i$, the martingale difference $\overline{\bA}^k_i - \overline{\bA}^{k-1}_i$ is $\eps^2 K^2 M_1^2 M_{\infty}^2$-sub-Gaussian under the conditional law $\cL (\cdot \vert \cF_k)$. We apply Azuma-Hoeffding's inequality (Lemma \ref{lem:Azuma})
\[
\P \Big( \max_{k \in [0,T/\eps ] \cap \N} \| \overline{\bA}_i^k\|_2 \geq K M_1 M_{\infty} \sqrt{\eps } \lb \sqrt{D} + z \rb  \Big) \leq e^{- z^2}.
\]
We get:
\[
\begin{aligned}
&\P \Big( \max_{k \in [0,T/\eps ] \cap \N} \| \bA_i^k\|_2 \geq K M_1 M_{\infty} \sqrt{\eps } \lb \sqrt{D} + z \rb  \Big) \\
 \leq & \P \Big( \max_{k \in [0,T/\eps ] \cap \N} \| \overline{\bA}_i^k\|_2 \geq K M_1 M_{\infty} \sqrt{\eps } \lb \sqrt{D} + z \rb  \Big) + \P ( \tau \le T/\eps) \\
 \le & 2e^{- z^2},
\end{aligned}
\]
where we used the high probability bound of $\sup_{k \in [0,T/\eps] \cap \N} \{ \| \ba^k \|_1 , \| \ba^k \|_1 \} $ in Lemma \ref{lem:bound_a_trajectories_noisy_D}.
Taking the union bound over $i \in [N]$ yields
\[
\P \Big( \max_{i\leq N} \max_{k \in [0,T/\eps ] \cap \N} A_i^k \geq Ke^{KT} \lb \sqrt{D}\log N + \log^{3/2} N + z^3 \rb \sqrt{\eps }   \Big) \leq e^{- z^2}.
\]
For the second term, we get:
\[
\begin{aligned}
& \|\bG ( \btheta^l_i , \rho^{(N)}_l ) -\bG ( \tbtheta^l_i ; {\Tilde \rho}^{(N)}_l )\|_2 \\ \leq & \lambda \|\btheta^l_i  -  \tbtheta^l_i  \|_2 + \| \na V (\btheta^l_i ) - \na V (\tbtheta^l_i ) \|_2  + \Big\| \frac{1}{N } \sum_{j=1}^N \na_1 U ( \btheta^l_i , \btheta^l_j ) - \na_1 U ( \tbtheta^l_i , \tbtheta^l_j ) \Big\|_2 \\
\le & K (1 + \| \ba^l \|_\infty ) \cdot \| \btheta^l_i - \tbtheta^l_i \|_2 + K (1 + \| \bbara^l \|_\infty ) \cdot (1 + \| \bbara^l \|_1 /N ) \cdot \max_{j \in [N]} \| \btheta^l_j - \tbtheta^l_j \|_2  \\
&+ K (1 + \| \btilda^l \|_\infty ) \cdot (1 + \| \btilda^l \|_1 /N ) \cdot \max_{j \in [N]} \| \btheta^l_j - \tbtheta^l_j \|_2.
\end{aligned}
\]
Using the high probability bound on $\sup_{k \in [0,T/\eps] \cap \N} \{ \| \ba^k \|_1 /N , \| \ba^k \|_\infty,\| \btilda^k \|_1 /N, \| \btilda^k \|_\infty \}$ of Lemma \ref{lem:bound_a_trajectories_noisy_D}, we get with probability at least $1 - e^{-z^2}$ that for all $t \in [0,T]$
\[
\begin{aligned}
\Delta ( t ) \leq & Ke^{KT} (1+z) \lb \sqrt{\log N} + z \rb \int_0^t \Delta ( s ) \de s  + Ke^{KT} \lb \sqrt{D}\log N + \log^{3/2} N + z^3 \rb \sqrt{\eps } .
\end{aligned}
\]
Applying Gronwall's inequality, we get:
\[
\P \Big( \Delta ( T )  \leq Ke^{e^{KT}[\sqrt{\log N} +z^2]} \lb \sqrt{D}\log N + \log^{3/2} N + z^3 \rb \sqrt{\eps }  \Big) \ge 1 - e^{-z^2}.
\]
This bound combined with Lemma \ref{lem:RN_bound} concludes the proof.
\end{proof}

\section{Existence and uniqueness of PDEs solutions \label{sec:existence_uniqueness}}

\subsection{Equation \eqref{eq:PDE}  (noiseless SGD)}

For the readers convenience, we reproduce here the form of the limiting PDE
\begin{align}
\partial_t\rho_t &= 2\xi(t)\nabla \cdot \big(\rho_t \nabla\Psi(\btheta;\rho_t)\big), \label{eq:PDE_noiseless_existence} \\
\Psi(\btheta;\rho_t)& = V(\btheta)+\int U(\btheta,\tbtheta) \, \rho_t(\de\tbtheta). 
\end{align}
This PDE describes an evolution in the space of probability distribution on $\R^D$ and has to be interpreted in the weak sense. Namely $\rho_t$ is a solution of Eq. \eqref{eq:PDE_noiseless_existence}, if for any bounded function $h : \R^D \mapsto \R$ differentiable with bounded gradient:
\begin{equation}
\frac{\de }{\de t} \int h (\btheta)  \rho_t ( \de \btheta )  = - 2 \xi (t) \int \< \na h (\btheta) , \na \Psi (\btheta ; \rho_t \> \rho_t ( \de \btheta ).
\label{eq:weak_sense_noiseless}
\end{equation}
For fixed coefficient, under assumptions  {\rm A1},  {\rm A2},  {\rm A3}, {\rm A4}, we have $\na V(\btheta)$ and $\na_1 U(\btheta , \btheta')$ bounded Lipschitz. By \cite[Theorem 1.1]{sznitman1991topics}, these assumptions are sufficient to guarantee the existence and uniqueness of solution of PDE \eqref{eq:PDE_noiseless_existence}.

For general coefficients, the potentials are not bounded and Lipschitz anymore. The existence and uniqueness under assumptions  {\rm A1},  {\rm A2},  {\rm A3}, {\rm A4}, can be derived by a similar argument as in \cite[Section 4]{sirignano2018mean}, which uses an adaptation of the argument of \cite[Theorem 1.1]{sznitman1991topics}.

\subsection{Equation \eqref{eq:PDEnoisy} (noisy SGD)}

For the readers convenience, we reproduce here the form of the limiting PDE
\begin{align}
\partial_t \rho_t & = 2\xi(t) \nabla \cdot \big( \rho_t \nabla\Psi_\lambda (\btheta ; \rho_t) \big) + 2 \xi (t)/ \beta \Delta_\btheta \rho_t, \label{eq:PDE_noisy_existence} \\
\Psi_\lambda (\btheta;\rho_t) & = V(\btheta)+\int U(\btheta,\tbtheta) \, \rho_t(\de\tbtheta) + \frac{\lambda}{2} \| \btheta \|_2^2.
\end{align}
We say that $\rho_t$ is a weak solution of Eq. \eqref{eq:PDE_noisy_existence} if for any $\zeta \in C_0^{\infty} (\R \times \R^D)$ (the space of smooth functions decaying to 0 at infinity), we have for any $T > 0$
\begin{align}
& \int_{\R^D}  \zeta_0 (\btheta ) \rho_0 (\de \btheta) - \int_{\R^D}  \zeta_0 (\btheta ) \rho_T (\de \btheta) \nonumber  \\
= & - \int_{(0,T) \times \R^D} [ \partial_t \zeta_t (\btheta ) - 2 \xi (t)  \< \na_\btheta \Psi_\lambda (\btheta ; \rho_t ) , \na_\btheta \zeta_t( \btheta ) \> + 2\xi (t) \Delta_\btheta \zeta_t (\btheta)] \rho_t (\de \btheta ) \de t . \label{eq:weak_solution_noisy}
\end{align}
Note that this notion of weak solution is equivalent to the one introduced earlier in Eq. \eqref{eq:weak_sense_noiseless}, see for instance \cite[Proposition 4.2]{santambrogio2015optimal}.

For fixed coefficients, the existence and uniqueness of solution of Eq. \eqref{eq:PDE_noisy_existence} was proven in \cite[Section 10.2]{mei2018mean}, under the assumptions  {\rm A1},  {\rm A2},  {\rm A3}, {\rm A6}. The proof follows from an adaptation of the proof of \cite[Theorem 5.1]{jordan1998variational}.

For general coefficients, we can follow a similar contraction argument as in \cite[Section 4]{sirignano2018mean} and \cite[Theorem 1.1]{sznitman1991topics}, by bounding more carefully each term.

\begin{proposition}\label{prop:existence_PDE_noisy_contraction}
Assume conditions {\rm A1}-{\rm A5}. Then PDE \eqref{eq:PDE_noisy_existence} admits a weak solution $(\rho_t)_{t\ge0}$ which is unique.
\end{proposition}

\begin{proof}[Proof of Lemma \ref{prop:existence_PDE_noisy_contraction}]
Without loss of generality, we assume $\xi(t) = 1/2$, which corresponds to a reparametrization of variable time $t$. 
Denote by $\cuP (\R^D)$ the set of probability measures on $\R^D$, endowed with the topology of weak convergence.
Note that Eq.~\eqref{eq:weak_solution_noisy} immediately implies that $t\mapsto \rho_t$ is continuous in $\cuP(\R^D)$.

Denote by  $D( [0,T] ; \cuP (\R^D))$ the set of maps from $[0,T]$ into $\cuP ( \R^D )$ and by
$C( [0,T] ; \cuP (\R^D))$ the set of continuous maps in this class. We introduce the map $\Phi_T : C( [0,T] ; \cuP( \R^D)) \rightarrow D( [0,T] ; \cuP(\R^D))$, which associates $m \in D( [0,T] ; \cuP(\R^D))$ to the law of the solution 
\[
\bbtheta^t = \bbtheta^0 + \int_0^t \bG ( \bbtheta^s ; m_s ) \de s + \overline{\bW} (t), \qquad \text{for } t\leq T, \, \bbtheta_0 \sim \rho_0.
\]
Observe that if $m$ is a weak solution of PDE \eqref{eq:PDE_noisy_existence} defined on interval $[0,T]$, then $m$ is a fixed point of $\Phi_T$. Further, for any such fixed point $m$, Lemma \ref{lem:bound_asquared_noisy_D} and Lemma \ref{lem:bound_increment_noisy_D} both apply. In particular, $t \mapsto m_t$ is continuous in $\cuP(\R^D)$ and therefore $\Phi_T$ maps $C( [0,T] ; \cuP( \R^D))$ to $C( [0,T] ; \cuP(\R^D))$.
Further, again by the same derivation, there exists a constant $C$, such that
\[
\int a^2 m_t ( \de a ) \leq C e^{Ct}, \qquad \text{for all }t \in [0,T].
\]
Let us define $\cuP_{C_0 , T_0} ( \R^D )$ the space of probability measures such that $\int a^2 \mu (\de a) \le C_0 e^{C_0 T_0}$. We consider $m \in C([0,T_0];\cuP_{C_0 , T_0} ( \R^D ))$, the set of continuous mapping from $[0,T_0]$ on $\cuP_{C_0 , T_0} ( \R^D )$. Using the same computation as in the proof of Lemma \ref{lem:bound_asquared_noisy_D}, we have:
\[
\bba^t = e^{-\lambda t} \bba^0 + \int_0^t e^{-\lambda (t - s)} K (\bbw^s, m_s ) \de s + \int_0^t e^{-\lambda (t -s )} \sqrt{\tau/D} \de W^a (s),
\]
where $\vert K (\bbw^s, m_s ) \vert = \Big\vert -v (\bbw^s) - \int a u (\bbw^s , \bw) m_s (\de a , \de \bw) \Big\vert \leq K + K\sqrt{C_0}e^{C_0 s/2}$. We get:
\[
(\bba^t)^2 \leq 9 K^2 + 18K^2 t + 18 K^2 t C_0 e^{C_0 t} + B_t^2,
\]
where $B_t$ is a normal random variable with variance bounded by $9 t\tau/D$. Taking the expectation with respect to $\Phi_T(m)$, we get:
\[
\int a^2 \Phi_T (m)_t (\de a ) \leq (9K^2 + 18K^2 t + 18 K^2 t C_0 + 9 t\tau/D)e^{C_0 t}.
\]
Hence we deduce that for $C_0$ sufficiently big and $T_0$ sufficiently small, we have for every $T \in [0, T_0]$, $\Phi_T (m) \in C( [0,T] ; \cuP_{C_0, T}( \R^D))$. 
We can therefore restrict our mapping $\Phi$ to the subsets $C( [0,T] ; \cuP_{C_0, T}( \R^D))$ for $T \le T_0$, which must contains all the fixed points by the above discussion. 

We introduce the following metric on $C( [0,T] ; \cuP_{C_0, T}( \R^D))$:
\[
\cuD_T (m^1 , m^2 ) = \Big( \inf \Big\{ \int \sup_{t \le T } \Vert \btheta_1^t - \btheta_2^t \Vert_2^2  \gamma (\de \btheta_1 ,\de \btheta_2): \text{$\gamma$ is a coupling of $m^1, m^2$} \Big\} \Big)^{1/2}.
\]
We show that for $T_1 \le T_0$ sufficiently small, the mapping $\Phi_{T_1}$ is a contraction with respect to this distance.
\begin{lemma}\label{lem:contraction}
There exists a constant $K$ such that, for all $T \le T_0$, and for all $m^1,m^2 \in C( [0,T] ; \cuP_{C_0, T}( \R^D))$, we have
\[
\cuD_T ( \Phi_T (m^1) , \Phi_T (m^2) ) \le T K \cuD_T (m^1 , m^2).
\]
\end{lemma}
\begin{proof}[Proof of Lemma \ref{lem:contraction}]
Fix $T \le T_0$, and consider a coupling $\gamma$ between $m^1,m^2 \in C( [0,T] ; \cuP_{C_0, T}( \R^D))$. We consider the following coupling between $\Phi_T (m^1)$ and $\Phi_T (m^2)$:
\[
\begin{aligned}
\bbtheta^t_1 & = \bbtheta^0 + \int \bG_1 (\bbtheta^s_1 ; \gamma_s ) \de s + \overline{\bW} (t), \\
\bbtheta^t_2 & = \bbtheta^0 + \int \bG_2 (\bbtheta^s_2 ; \gamma_s ) \de s + \overline{\bW} (t), 
\end{aligned}
\]
where $\bG_1 (\bbtheta^s_1 ; \gamma_s) = -\lambda \bbtheta^s_1 - \na V(\bbtheta^s_1) - \int_{\R^D \times \R^D} \na_1 U(\bbtheta^s_1, \btheta_1) \gamma (\de \btheta_1 ,\de \btheta_2 )$ (and similarly for $\bG_2$). We have:
\[
\begin{aligned}
\| \bG_1 (\bbtheta^s_1 ; \gamma_s ) -\bG_2 (\bbtheta^s_2 ; \gamma_s ) \|_2 \le & K (1 + \vert \bba^s_1 \vert ) \| \bbtheta^s_1 - \bbtheta^s_2 \|_2 + K \int | a_1 | (1 + |\bba^s_1|) \| \bbtheta^s_1 - \bbtheta^s_2 \|_2 \gamma_s ( \de \btheta_1 ,\de \btheta_2 ) \\
& + \int (1 + |a_1|) (1 + | \bba^s_2 |) \| \btheta_1 - \btheta_2 \|_2 \gamma_s ( \de \btheta_1, \de \btheta_2 ). 
\end{aligned}
\]
Hence, we get (using that $m^1_s \in  \mathcal{P}_{C_0, T}( \R^D)$)
\[
\begin{aligned}
\| \bbtheta^t_1 - \bbtheta^t_2 \|_2 & \le K e^{K T_0} \int_0^t (1 + \vert \bba^s_1 \vert ) \| \bbtheta^s_1 - \bbtheta^s_2 \|_2 \de s + K \int_0^t (1 + | \bba^s_2 |) \int (1 + |a_1|) \| \btheta_1 - \btheta_2 \|_2 \gamma_s ( \de \btheta_1 ,\de \btheta_2 ) \de s, 
\end{aligned}
\]
where $K$ is a constant depending on the constants of the assumptions and $C_0$. Taking the square and using Cauchy-Schwartz inequality
\[
\begin{aligned}
\| \bbtheta^t_1 - \bbtheta^t_2 \|_2^2  \le & K e^{K T_0} \int_0^t (1 + \vert \bba^s_1 \vert )^2 \de s \int_0^t  \| \bbtheta^s_1 - \bbtheta^s_2 \|_2^2 \de s \\
& + K \int_0^t  (1 + | \bba^s_2 |)^2 \de s \int_0^t \Big( \int (1 + |a_1|)^2 \gamma_s ( \de \btheta_1 ,\de \btheta_2 ) \int \| \btheta_1 - \btheta_2 \|_2^2 \gamma_s ( \de \btheta_1 , \de \btheta_2 ) \Big) \de s\\
\le & Ke^{KT_0} T_0 M_{T_0} \int_0^t  \| \bbtheta^s_1 - \bbtheta^s_2 \|_2^2 \de s + Ke^{KT_0} M_{T_0} t^2 \int \sup_{t \le T} \| \btheta_1^t - \btheta_2^t \|_2^2 \gamma ( \de \btheta_1 , \de \btheta_2 ), 
\end{aligned}
\]
where $M_{T_0} = (1 + \sup_{t \le T_0} (| \bba^t_1 | \vee | \bba^t_2 | ) )^2$. Applying Gronwall's lemma, we get, for any $T<T_0$,
\[
\sup_{t \leq T} \| \bbtheta^t_1 - \bbtheta^t_2 \|_2^2 \le KT^2e^{KT_0^2e^{KT_0}M_{T_0}} \int \sup_{t \le T} \| \btheta_1^t - \btheta_2^t \|_2^2 \gamma ( \de \btheta_1 , \de \btheta_2 ).
\]
Taking the expectation:
\[
\E [ \sup_{t \leq T} \| \bbtheta^t_1 - \bbtheta^t_2 \|_2^2 ] \le K T^2  \E\lbrace \exp (KT_0^2 e^{KT_0}M_{T_0} )\rbrace \int \sup_{t \le T} \| \btheta_1^t - \btheta_2^t \|_2^2 \gamma ( \de \btheta_1 , \de \btheta_2 ).
\]
By a similar argument as in Lemma \ref{lem:bound_a_trajectories_noisy_D}, we have $\P ( M_{T_0} \ge Ke^{KT_0} (1 + z^2) ) \le e^{-z^2}$, i.e.
\[
\P ( \exp \{ KT_0^2 e^{KT_0} M_{T_0} \} \ge \exp \{ KT_0^2 e^{KT_0} (1 + z^2) \} ) \le e^{-z^2}.
\]
Doing a change of variable, we get:
\[
\begin{aligned}
 \E\lbrace \exp (KT_0^2 e^{KT_0}M_{T_0} )\rbrace = & \int \P (\exp (KT_0^2 e^{KT_0}M_{T_0} ) \ge u ) \de u \\
 \le & KT_0^2 e^{KT_0}  + KT_0^2 e^{KT_0^2 e^{KT_0}} \int_0^\infty z \exp \{ - (1 - KT_0^2 e^{KT_0} ) z^2 \} \de z < \infty, 
\end{aligned}
\]
for $T_0$ small enough. We conclude that there exists a constant $K < \infty$ such that
\[
\cuD_T ( \Phi_T (m^1) , \Phi_T (m^2) ) \le (\inf_\gamma \E [ \sup_{t \leq T} \| \bbtheta^t_1 - \bbtheta^t_2 \|_2^2 ])^{1/2} \le T K \cuD_T (m^1 , m^2),
\]
where we used that the coupling $\gamma$ was chosen arbitrarily.
\end{proof}

We can therefore consider $T_1 < 1/K$. The mapping $\Phi_{T_1} $ is a contraction on the space $C( [0,T_1] ; \cuP_{C_0, T_1}( \R^D))$. By the Banach fixed-point theorem,  there exists a fixed point for $\Phi_{T_1}$ on the interval $[0,T_1]$, which is unique. 
We can further iterate the same argument. Assume that the fixed point of $\Phi_T$ is unique, for some $T>0$. 
Then $\Phi_{[0,T+T_1]}$ has a unique fixed point, which is a map $m:[0,T+T_1]\to \cuP(\R^D)$. This suffices to conclude that PDE \eqref{eq:PDE_noisy_existence} admits a weak solution on $[0,\infty)$, and this solution is unique.
\end{proof}

Further, Duhamel's principle for PDE \eqref{eq:PDE_noisy_existence} holds. Denote $\cuG ( \btheta, \btheta' ;t)$ the heat kernel:
\[
\cuG ( \btheta, \btheta' ;t) \equiv \frac{1}{(2\pi t)^{d/2}} \exp \{ - \| \btheta - \btheta '\|^2_2 /(2t) \}.
\]

\begin{lemma} \label{lem:duhamel}
Assume conditions {\rm A1}-{\rm A5}. Let $\rho$ be a weak solution of PDE \eqref{eq:PDE_noisy_existence}. Then, for any $t >0$, $\rho_t(\de \btheta) $ has a density, denoted $\rho (t,\cdot)$, which satisfies
\[
\rho (t,\btheta ) = \int \cuG (\btheta, \btheta_1 ;\tau t/D) \rho_0 (\de \btheta_1 ) - \int_0^t \int \< \na_{\btheta_1} \cuG (\btheta, \btheta_1 ;\tau (t-s)/D) , \na_{\btheta_1} \Psi (\btheta_1 ; \rho_s ) \> \rho(s ,\btheta_1 ) \de \btheta_1 \de s.
\]
\end{lemma}

\begin{proof}[Proof of Lemma \ref{lem:duhamel}]
For ease of notation, let us set $\tau/D =1$ and $\xi(t) = 1/2$, which amounts to rescaling time. Consider $\eta \in C^\infty (\R^D)$ (space of smooth real-valued functions) with bounded support, and define:
\[
\cuG_{\eta} (\btheta ; t ) = \int \cuG (\btheta, \btheta_1 ;t ) \eta(\btheta_1) \, \de \btheta_1.
\]
By property of the heat kernel, we have
\[
(\partial_t -\Delta )\cuG_{\eta} (\btheta ; t ) = 0, \qquad \forall t > 0, \forall \btheta \in \R^D.
\]
Take $\zeta(\btheta , s ) = \cuG_{\eta} (\btheta ; t-s )$ (which indeed decays to $0$ at infinity) as a test function in Eq. \eqref{eq:weak_solution_noisy} for $T=t$. We get:
\[
\int \eta(\btheta_1) \rho_t (\de \btheta_1) = \int \cuG_{\eta} (\btheta ; t ) \rho_0 (d\btheta) -\int_{(0,t)\times\R^D} \< \cuG_{\eta} (\btheta ; t - s) , \na \Psi (\btheta ; \rho_s) \> \rho_s (\de \btheta) \de s.
\]
By applying Fubini's theorem, we get
\[
\begin{aligned}
& \int \eta(\btheta_1) \rho_t (\de \btheta_1) \\
= & \int \cuG (\btheta, \btheta_1 ; t ) \rho_0 (d\btheta) \eta(\btheta_1) \de \btheta_1  -\int_{(0,t)\times\R^D\times\R^D} \< \cuG (\btheta, \btheta_1 ; t - s) , \na \Psi (\btheta, \btheta_1 ; \rho_s) \> \rho_s (\de \btheta) \de s \; \eta (\btheta_1) \de \btheta_1,
\end{aligned}
\]
where $\eta$ is an arbitrary function with bounded support, which concludes the proof.
\end{proof}

\begin{lemma}\label{lem:regularity_PDE_noisy}
Assume conditions {\rm A1}- {\rm A6}. Assume further that $\rho_0$ has a density. Denote $(\rho_t)_{t\ge0}$ the solution of PDE \eqref{eq:PDE_noisy_existence}, with density $(\rho ( t ,\cdot ) )_{t\ge0}$. Then $( t, \btheta) \mapsto \rho ( t, \btheta)$ is in $C^{1,2}((0,\infty)\times \R^D )$, where $C^{1,2}((0,\infty)\times \R^D )$ is the function space of continuous function with continuous derivative in time, and second order continuous derivative in space.
\end{lemma}

\begin{proof}[Proof of Lemma \ref{lem:regularity_PDE_noisy}]
The proof follows exactly from the proof of Lemma \cite[Lemma 10.7]{mei2018mean}.
\end{proof}

\subsection{The noisy PDE as a gradient flow in the space of probability distributions}

We include a second independent proof of the existence of a weak solution, which is interesting in itself. It relies on a deep connection pioneered by \cite{jordan1998variational}, between Fokker-Planck PDEs and gradient flow in probability space. The proof follows closely the steps detailed in \cite{jordan1998variational}. The arguments are similar to \cite[Section 10.2]{mei2018mean}, and we will only detail the differences.

We will consider the set $\cK$ of admissible probability densities,
\[
\cK = \Big\{ \rho : \R^D \mapsto [0,+\infty ) \text{ measurable  }: \int_{\R^D } \rho (\btheta) \de \btheta = 1, M(\rho) < \infty \Big\},
\]
where 
\[
M ( \rho) \equiv \int_{\R^D} \| \btheta \|_2^2 \rho (\btheta) \de \btheta .
\]
Recall
\[
R ( \rho) = \E (y^2) + 2 \int_{\R^D } V( \btheta ) \rho (\btheta) \de \btheta + \int_{\R^D\times \R^D } U( \btheta , \btheta') \rho (\btheta) \rho (\btheta')  \de \btheta \de \btheta' .
\]
We will define
\[
\begin{aligned}
\Ent(\rho ) & = - \int_{\R^D} \rho (\btheta) \log \rho ( \btheta ) \de \btheta , \\
F (\rho) & = 1/2 \cdot [ \lambda M (\rho ) + R (\rho) ] - 1/\beta \cdot \Ent (\rho). 
\end{aligned}
\]
The PDE \eqref{eq:PDE_noisy_existence} can be interpreted as a gradient flow on the free energy functional $F(\rho )$ in the space of probability measures on $\R^D$ endowed with the $W_2(\cdot , \cdot)$ Wassertein distance \cite[Section 10.2]{mei2018mean}. Recall that for $\mu, \nu$ probability distributions over $\R^D$, we have:
\[
W_2^2 (\mu , \nu ) = \inf \Big\{ \int_{\R^D \times \R^D} \| \btheta_1 - \btheta_2 \|_2^2 \gamma (\de \btheta_1 , \de \btheta_2 ): \text{ $\gamma$ is a coupling of $\mu,\nu$} \Big\}
\]

\begin{proposition}\label{prop:existence_PDE_noisy}
Assume conditions {\rm A1},  {\rm A2},  {\rm A3}, {\rm A6}. Let initialization $\rho_0 \in \cK$ so that $F(\rho_0) < \infty$. Then the PDE \eqref{eq:PDE_noisy_existence} admits a weak solution $(\rho_t)_{t\ge0}$ which is unique. Moreover, for any fixed $t$, $\rho_t \in \cK$ is absolutely continuous with respect to the Lebesgue measure, and $M(\rho_t)$ and $\Ent ( \rho_t ) $ are uniformly bounded in $t$.
\end{proposition}

\begin{proof}[Proof of Proposition \ref{prop:existence_PDE_noisy}]
Without loss of generality, we assume $\xi(t) = 1/2$, which corresponds to a reparametrization of variable time $t$. To prove the existence of the solution, we consider the limit of the following discretized scheme when the step-size $h$ goes to zero: we define recursively a sequence of distributions $\{ \orho^h_k\}_{k \in \N}$, with $\orho^h_0 = \rho_0$ and
\begin{equation}
\orho^h_{k+1} \in \arg\min_{\rho \in \cK} \Big\{ hF(\rho) +\frac{1}{2} W_2^2 (\rho , \overline{\rho}^h_{k} ) \Big\}. \label{eq:discrete_gradient_flow}
\end{equation}
\begin{lemma}\label{lem:existence_disc_flow}
Given an initialization $\rho_0 \in \cK$, there exists a unique solution of the scheme \eqref{eq:discrete_gradient_flow}.
\end{lemma}
\begin{proof}[Proof of Lemma \ref{lem:existence_disc_flow}] Clearly it is sufficient to analyze a single step of the scheme  \eqref{eq:discrete_gradient_flow}.
The proof follows from the same arguments as in \cite[Proposition 4.1]{jordan1998variational}, which shows that there exists a sequence of measures $\{\rho_\nu\}_{\nu \in \N} \in \cK$ that converges weakly to $\rho^* \in \cK$ such that
\[
 \lim_{\nu \rightarrow \infty } \Big\{ F(\rho_\nu ) + \frac{1}{2} W_2^2 (\rho_{\nu} , \rho_0 ) \Big\} = \inf_{\rho \in \cK}  \Big\{ F (\rho )  + \frac{1}{2} W_2^2 (\rho_{\nu} , \rho_0 ) \Big\} > -\infty . 
 \]
Moreover, there exists a constant $C$ such that $M (\rho_\nu ) \leq C$ and $M (\rho^* ) \leq C$ by lower semi-continuity of $M(\rho)$. We only need to check lower semi-continuity of $R(\rho)$ to conclude that $\rho^*$ is indeed a minimizer. Uniqueness comes from convexity of the functional and strict convexity of $-\Ent (\rho)$. 

Denote for $x\in \R$, the functions $\overline{\phi}_m (x)  = \text{sign} ( x ) \cdot \max 
\{ \vert x \vert - m, 0 \}$ and $\underline{\phi}_m (x) = x - \overline{\phi}_m (x)$, and $\mathsf{B} (r) = \mathsf{B} (0, r) \subset \R^D$:
\begin{align*}
    \vert R(\rho_\nu) - R (\rho^*)\vert & \leq \Big\vert  \int \underline{\phi}_m (V (\btheta)) [ \rho_\nu (\btheta) - \rho^* (\btheta) ] \de \btheta \Big\vert + \Big\vert \int \underline{\phi}_m (U (\btheta, \btheta')) [\rho_\nu (\btheta) \rho_\nu (\btheta')  - \rho^* (\btheta ) \rho^* (\btheta') ] \de \btheta  \de \btheta' \Big\vert  \\
    & + \Big\vert  \int \overline{\phi}_m (V (\btheta)) [ \rho_\nu (\btheta) - \rho^* (\btheta) ] \de \btheta \Big\vert + \Big\vert \int \overline{\phi}_m (U (\btheta, \btheta')) [\rho_\nu (\btheta) \rho_\nu (\btheta')  - \rho^* (\btheta ) \rho^* (\btheta') ] \de \btheta  \de \btheta' \Big\vert .
\end{align*}
By weak convergence in $L^1 (\R^D)$, the first two terms converge to zero. Recalling that $V ( \btheta)  = a v (\bw )$ and $U ( \btheta , \btheta') = a a' u (\bw , \bw')$, with $\vert v ( \bw ) \vert \le K$ and $\vert u (\bw , \bw') \vert \le K$, we deduce  
\begin{align*}
 & \Big\vert  \int \overline{\phi}_m (V (\btheta)) [ \rho_\nu (\btheta) - \rho^* (\btheta) ] \de \btheta \Big\vert \le \Big\vert  \int_{\mathsf{B} (m/K) } \overline{\phi}_m (V (\btheta)) [ \rho_\nu (\btheta) - \rho^* (\btheta) ] \de \btheta \Big\vert \le 2 KC/m, 
 \end{align*}
 and
 \begin{align*}
 & \Big\vert \int \overline{\phi}_m (U (\btheta, \btheta')) [\rho_\nu (\btheta) \rho_\nu (\btheta')  - \rho^* (\btheta ) \rho^* (\btheta') ] \de \btheta  \de \btheta' \Big\vert \\
 \le & \Big\vert \int_{\mathsf{B} (\sqrt{m}/K) \times \mathsf{B} (\sqrt{m}/K) } \overline{\phi}_m (U (\btheta, \btheta')) [\rho_\nu (\btheta) \rho_\nu (\btheta')  - \rho^* (\btheta ) \rho^* (\btheta') ] \de \btheta  \de \btheta' \Big\vert \le 2KC^2/m,
\end{align*}
where we used that $ \int_{\mathsf{B}(r)} \vert a \vert \rho_\nu (\de a) \le \int a^2/r \rho_\nu (\de a) \le C/r$. Because $m$ is arbitrarily large, we conclude that
\[
\lim_{\nu \rightarrow \infty } \abs{R(\rho_\nu) - R (\rho^*)} = 0.
\]
\end{proof}

The rest of the proof follows the proof of \cite[Theorem 5.1]{jordan1998variational}, which shows that for a given $T<\infty $, there exists $C$ such that for any $h$ and $k$ with $hk \leq T$, we have $M(\overline{\rho}^h_{k})  \leq C $. If we denote $\rho^h (t,.)$ the piece wise constant distribution trajectory, we deduce that it converges weakly to $\rho$ in $L^1 ((0,T) \times \R^D )$. Furthermore, the weak convergence applies for each given time $t \in [0,+ \infty)$, i.e. $\rho^h (t) \mapsto \rho (t)$ weakly.

We still need to show that this limiting distribution is a weak solution \eqref{eq:weak_sense_noiseless} of PDE \eqref{eq:PDE_noisy_existence}. Let $\bxi \in C_0^\infty ( \R^D, \R^D )$ be a smooth vector field with bounded support, and define $\{ \Phi_\tau\}_{\tau \in \R}$ the corresponding flux:
\begin{equation}
\partial_\tau \Phi_\tau = \bxi \circ \Phi_\tau \text{ for all $\tau \in \R$ and $\Phi_0 =$id.} \label{eq:flux_forward}
\end{equation}
Further, for $\tau \in \R$, define $\nu_\tau$ to be the push forward measure of $\overline{\rho}^h_k$ under $\Phi_\tau$. Namely,
\[
\int_{\R^D} \nu_\tau (\btheta ) \zeta (\btheta) \de \btheta = \int_{\R^D} \overline{\rho}^h_k (\btheta ) \zeta (\Phi_\tau (\btheta)) \de \btheta, \qquad \forall \zeta \in C(\R^D),
\]
or equivalently $\nu_\tau = \frac{1}{\det \na \Phi_\tau} \overline{\rho}^h_{k}  \circ \Phi_\tau^{-1}$. We only need to consider the term $R(\rho )$. See the proof of \cite[Lemma 10.6]{mei2018mean} for more details.

 From the assumption of bounded support, we must have $\sup_{\btheta \in \R^D } \| \bxi (\btheta) \|_2 \leq K$. From Eq. \eqref{eq:flux_forward}, we have
 \begin{equation}
 \Phi_\tau (\btheta) =  \btheta + \int_0^\tau \Phi_s (\bxi ( \btheta) ) \de s.
 \label{eq:evol_flux_theta}
 \end{equation}
 Hence applying Gronwall's inequality to $u(\tau) = \sup_{\btheta \in \mathsf{B}(r)} \| \Phi_\tau (\btheta )\|_2$, and considering $\tau \le 1$, we get $u(\tau) \le K$. Therefore, for $\tau \le 1$, we get 
 $\vert (\partial^2/\partial \tau^2) \Phi_\tau (\btheta)  \vert = \vert  \Phi_\tau ( \bxi ( \bxi (\btheta))) \vert \le K$. We deduce that
 \begin{equation}
 \| \Phi_\tau (\btheta ) - \btheta - \tau \bxi(\btheta) \|_2 \le K \tau^2. 
 \label{eq:evol_flux_bound}
 \end{equation}
Let us consider the derivative of $R(v_\tau)$ with respect to $\tau$. Recall that $U$ is symmetric.
\[
\begin{aligned}
& \int [ U ( \Phi_\tau (\btheta_1) , \Phi_\tau (\btheta_2) ) - U (\btheta_1 , \btheta_2 ) - 2 \tau \< \na_1 U (\btheta_1 , \btheta_2 ) , \bxi (\btheta_1) \> ] \overline{\rho}^h_{k} (\btheta_1 ) \overline{\rho}^h_{k} (\btheta_2 ) \de \btheta_1 \de \btheta_2 \\
= & \int [ U ( \Phi_\tau (\btheta_1) , \Phi_\tau (\btheta_2) ) - U (\Phi_\tau (\btheta_1) , \btheta_2 ) -  \tau \< \na_2 U (\Phi_\tau ( \btheta_1) , \btheta_2 ) , \bxi (\btheta_2) \> ] \overline{\rho}^h_{k} (\btheta_1 ) \overline{\rho}^h_{k} (\btheta_2 ) \de \btheta_1 \de \btheta_2 \\
& + \int [ U ( \Phi_\tau (\btheta_1) ,  \btheta_2 ) - U (\btheta_1 , \btheta_2 ) -  \tau \< \na_1 U (\btheta_1 , \btheta_2 ) , \bxi (\btheta_1) \> ] \overline{\rho}^h_{k} (\btheta_1 ) \overline{\rho}^h_{k} (\btheta_2 ) \de \btheta_1 \de \btheta_2 \\
& + \int [  \tau \< \na_2 U (\Phi_\tau (\btheta_1) , \btheta_2 ) , \bxi (\btheta_2) \> -  \tau \< \na_2 U (\btheta_1 , \btheta_2 ) , \bxi (\btheta_2) \> ] \overline{\rho}^h_{k} (\btheta_1 ) \overline{\rho}^h_{k} (\btheta_2 ) \de \btheta_1 \de \btheta_2.
\end{aligned}
\]
Denote $(a_1^\tau , \bw_1^\tau ) = \Phi_\tau (\btheta_1)$ and $(a_2^\tau , \bw_2^\tau ) = \Phi_\tau (\btheta_2)$, and $\bxi (\btheta) = (\xi_a (\btheta) , \bxi_w (\btheta))$. Consider the first term
\[
\begin{aligned}
& U ( \Phi_\tau (\btheta_1) , \Phi_\tau (\btheta_2) ) - U (\Phi_\tau (\btheta_1) , \btheta_2 ) -  \tau \< \na_2 U (\Phi_\tau ( \btheta_1) , \btheta_2 ) , \bxi (\btheta_2) \> \\
= & a_1^\tau \{ [a_2^\tau - a_2] u( \bw_1^\tau , \bw_2^\tau ) +a_2 [u( \bw_1^\tau , \bw_2 ) - u( \bw_1^\tau , \bw_2^\tau ) ]  - \tau \xi_a (\btheta_2) u (\bw_1^\tau , \bw_2) - \tau a_2 \< \na_{\bw_2} u (\bw_1^\tau , \bw_2), \bxi_w (\btheta_2 ) \> \} \\
= & a_1^\tau \{ [a_2^\tau - a_2 - \tau \xi_a (\btheta_2) ] u( \bw_1^\tau , \bw_2^\tau ) + a_2 [u( \bw_1^\tau , \bw_2 ) - u( \bw_1^\tau , \bw_2^\tau ) - \tau  \< \na_{\bw_2} u (\bw_1^\tau , \bw_2), \bxi_w (\btheta_2 ) \> ]\}\\
& + \tau a_1^\tau \xi_a [u( \bw_1^\tau , \bw_2^\tau ) - u( \bw_1^\tau , \bw_2 )]  .
\end{aligned}
\]
Using that $\| \na u \|_{\text{op}} , \| \na^2 u \|_{\text{op}} \leq K$, and Eq. \eqref{eq:evol_flux_theta} and Eq. \eqref{eq:evol_flux_bound}, we get for $\tau \le 1$
\[
\begin{aligned}
& \Big\vert \int [ U ( \Phi_\tau (\btheta_1) , \Phi_\tau (\btheta_2) ) - U (\Phi_\tau (\btheta_1) , \btheta_2 ) -  \tau \< \na_2 U (\Phi_\tau ( \btheta_1) , \btheta_2 ) , \bxi (\btheta_2) \> ] \overline{\rho}^h_{k} (\btheta_1 ) \overline{\rho}^h_{k} (\btheta_2 ) \de \btheta_1 \de \btheta_2 \Big\vert \\
\le & K \tau^2 \int \vert a^\tau_1 (K + a_2) \vert \overline{\rho}^h_{k} (\btheta_1 ) \overline{\rho}^h_{k} (\btheta_2 ) \de \btheta_1 \de \btheta_2  \leq K \tau^2 C (K + C),
\end{aligned}
\]
where we used that $\vert a^\tau \vert \le \vert a \vert +K\tau$ from Eq. \eqref{eq:evol_flux_theta}, and $M(\overline{\rho}^h_{k}) \le C$. The same computation shows that the second and third terms, as well as the term depending on $V(\btheta)$ are $O(\tau^2)$.

Taking $\tau \rightarrow 0$, we conclude that:
\[
\frac{\de}{\de\tau} [ R (\nu_\tau )]_{\tau = 0} = \int_{\R^d} \< \na \Psi (\btheta, \overline{\rho}^h_{k} ) , \bxi (\btheta ) \> \overline{\rho}^h_{k} (\btheta) \de \btheta .
\]
This equality combined with the analysis of \cite[Theorem 5.1]{jordan1998variational} shows that $\rho(t)$ is indeed a weak solution of PDE \eqref{eq:PDE_noisy_existence}. The proof of uniqueness follows from the regularity Lemma \ref{lem:regularity_PDE_noisy} and a standard method from elliptic-parabolic equations (see \cite[Theorem 5.1]{jordan1998variational} for details).
\end{proof}

\section{Proof of Theorem \ref{thm:Kernel}}

\begin{proof}[Proof of Theorem \ref{thm:Kernel}]
Let $\Ltwo(\reals^d, \prob)$ be the space of functions on $\R^d$ that is square integrable with respect to the measure $\prob$. For any functions $u, v \in \Ltwo(\reals^d, \prob)$, we denote by $\<u,v\>_{\Ltwo} = \int_{\R^d} u(\bx) v(\bx) \prob(\de \bx)$ the scalar product of $u, v$ and  $\|u\|_{\Ltwo} = (\< u, u \>_{\Ltwo})^{1/2}$ the norm of $u$ in $\Ltwo(\reals^d,\prob)$. 

We prove the case for general coefficients. The proof of fixed coefficient is the same but simpler. 

\noindent
{\bf Step 1.} Bound the support of $\bara^{t, \alpha}$. 

Let $\bbtheta^{t, \alpha} = (\bara^{t, \alpha}, \bbw^{t, \alpha})$ satisfying the non-linear dynamics
\[
\frac{\de}{\de t}\bbtheta^{t, \alpha} = - \frac{1}{\alpha} \nabla_\btheta \Psi_\alpha(\bbtheta^{t, \alpha}; \rho_t^\alpha)
\]
with initialization $\bbtheta^{0, \alpha} \sim \rho_0$, and $\rho_{t}^\alpha$ given by Eq. \eqref{eq:RescaledPDE}. Then we have 
\[
\begin{aligned}
\Big \vert \frac{\de}{\de t} \bara^{t, \alpha} \Big \vert =& \Big \vert (1/ \alpha) \E[(f(\bx) - \hat f(\bx; \rho_t^\alpha)) \sigma(\bx; \bbw^{t, \alpha})] \Big \vert \\
\le& (1/\alpha) \E[(f(\bx) - \hat f(\bx; \rho_t^\alpha))^2]^{1/2} \E[\sigma(\bx; \bbw^{t, \alpha})^2]^{1/2}\\
\le& (1/\alpha) K R_\alpha(\rho_t^\alpha)^{1/2}.  
\end{aligned}
\]
The last inequality follows from the assumption that $\| \sigma \|_\infty \le K$. Note $R_\alpha(\rho_t^\alpha)$ will always decrease along the trajectory, i.e., we have $R_\alpha(\rho_t^\alpha) \le R_\alpha(\rho_0) \le B$. As a result, we have $\vert \de \bara^{t, \alpha} / \de t \vert \le K B^{1/2}/ \alpha$, so that 
\[
\vert \bara^{t, \alpha} \vert \le K(1 + B^{1/2} t / \alpha) \equiv M_{t, \alpha}. 
\]
Denoting $A(\rho) = \sup_{(a, \bw) \in \supp(\rho) }\vert a \vert$. Since $(\bara^{t, \alpha}, \bbw^{t, \alpha}) \sim \rho_t^\alpha$, we have 
\[
A(\rho_t^\alpha) \le M_{t, \alpha} = K(1 + B^{1/2} t / \alpha).
\]

\noindent
{\bf Step 2. } Bound $W_2(\rho_t^\alpha, \rho_0)$.

For $\btheta = (a, \bw)$, we have
\begin{align}
\|\nabla_\btheta \Psi_\alpha(\btheta,\rho^{\alpha}_t)\| &= \|\E\{\nabla_\btheta \sigma_{\star}(\bx;\btheta) [f(\bx)- \hf_\alpha(\bx;\rho^{\alpha}_t)] \} \|\\
& \le \E\{ \|\nabla_\btheta \sigma_{\star}(\bx;\btheta)\|^2 \}^{1/2}\E\{[f(\bx)-\hf_\alpha(\bx;\rho^{\alpha}_t)]^2\}^{1/2}\\
& = \{\E\{ \sigma(\bx;\bw)^2\} + a^2 \E\{\|\nabla_{\bw}\sigma(\bx;\bw)\|_2^2\} \}^{1/2}R_\alpha(\rho^{\alpha}_t)^{1/2}\\
& \le K(1+ |a|\sqrt{D})  B^{1/2}\,.
\end{align}
The last inequality follows from $\| \sigma \|_\infty \le K$ and 
\[
\E\{\|\nabla_{\bw}\sigma(\bx;\bw)\|_2^2\} = \tr(\nabla_1 \nabla_2 u(\bw, \bw)) \le D \| \nabla_1 \nabla_2 u(\bw, \bw) \|_\op \le KD. 
\]
Hence, for $s \le t$, 
\[
\| \bbtheta^{t, \alpha} - \bbtheta^{s, \alpha} \|_2 = \frac{1}{\alpha} \Big \| \int_s^t \nabla_\btheta \Psi_\alpha(\bbtheta^{u, \alpha}; \rho_u^{\alpha}) \de u \Big \|_2\le \frac{K}{\alpha} \vert t - s \vert M_{t, \alpha} B^{1/2} \sqrt{D}.
\]
Note that, by the coupling in terms of nonlinear dynamics, for any $s \le t$, we have
\begin{align}
W_2(\rho^{\alpha}_s, \rho^{\alpha}_t)&\le \E \{ \| \bbtheta^{s, \alpha} - \bbtheta^{t, \alpha} \|^2 \}^{1/2} \le \frac{K}{\alpha} \vert t - s \vert M_{t, \alpha} B^{1/2} \sqrt{D}. 
\end{align}

\noindent
{\bf Step 2. } Bound $\|\cH_{\rho_0}-\cH_{\rho_t^\alpha}\|_{\sop}$.

Note that, for $v\in \Ltwo(\reals^d,\prob)$, 
\begin{align}
\<v,\cH_{\rho}v\>_{\Ltwo} = \int \left\|\E_{\bx}\{\nabla_\btheta \sigma_\star (\bx;\btheta)v(\bx)\} \right\|_2^2\rho(\de\btheta)\, .
\end{align}
Letting $\gamma$ denote the coupling that achieves the $W_2$ distance between $\rho_1$ and $\rho_2$, we have
\[
\begin{aligned}
\<v,[\cH_{\rho_1}-\cH_{\rho_2}]v\>_{\Ltwo} =& \int\left\{ \left\|\E_{\bx}\{\nabla_\btheta \sigma_\star(\bx;\btheta_1)v(\bx)\}\right\|_2^2 -
\left\|\E_{\bx}\{\nabla_\btheta \sigma_\star(\bx;\btheta_2)v(\bx)\}\right\|_2^2\right\}\gamma(\de\btheta_1,\de\btheta_2)\\
\le& \Big[ \int A_-(\btheta_1, \btheta_2) \gamma(\de \btheta_1, \de \btheta_2 ) \cdot \int A_+(\btheta_1, \btheta_2) \gamma(\de \btheta_1, \de \btheta_2)\Big]^{1/2}.
\end{aligned}
\]
where
\[
\begin{aligned}
A_- (\btheta_1, \btheta_2)&\equiv   \left\|\E_{\bx}\{[\nabla_\btheta \sigma_\star(\bx;\btheta_1) - \nabla_\btheta \sigma_\star(\bx;\btheta_2)] v(\bx)\}\right\|_2^2 \\
&  \le \E_{\bx}\{\|\nabla_\btheta \sigma_\star(\bx;\btheta_1) - \nabla_\btheta \sigma_\star(\bx;\btheta_2)\|_2^2\} \| v \|^2_{\Ltwo} \, ,\\
A_+ (\btheta_1, \btheta_2)&\equiv  (\E_{\bx}\{[\| \nabla_\btheta \sigma_\star(\bx;\btheta_1)\|_2 + \| \nabla_\btheta \sigma_\star(\bx;\btheta_2)\|_2] \| v(\bx) \|_2 \})^2 \\& \le  \E_{\bx}\{( \|\nabla_\btheta \sigma_\star(\bx;\btheta_1) \|_2+ \|\nabla_\btheta \sigma_\star(\bx;\btheta_2)\|_2)^2\} \| v \|^2_{\Ltwo}\, . 
\end{aligned}
\]
Note we have 
\[
\begin{aligned}
& \E_{\bx}\{( \|\nabla_\btheta \sigma_\star(\bx;\btheta_1) \|_2+ \|\nabla_\btheta \sigma_\star(\bx;\btheta_2)\|_2)^2\} \\
 =& \tr[\nabla_1 \nabla_2 U(\btheta_1, \btheta_1)] +\tr [ \nabla_1 \nabla_2 U(\btheta_2, \btheta_2)] + 2 \{ \tr[\nabla_1 \nabla_2 U(\btheta_1, \btheta_1) ]\cdot \tr[\nabla_1 \nabla_2 U(\btheta_2, \btheta_2)] \}^{1/2}  \\
\le& D ( \| \nabla_1 \nabla_2 U(\btheta_1, \btheta_1) \|_\op  + \| \nabla_1 \nabla_2 U(\btheta_2, \btheta_2)\|_\op  + 2 \{ \| \nabla_1 \nabla_2 U(\btheta_1, \btheta_1) \|_\op   \| \nabla_1 \nabla_2 U(\btheta_2, \btheta_2)\|_\op\}^{1/2} ) \\
\le& K D (1 + \vert a_1 \vert \vee \vert a _2 \vert)^2,
\end{aligned}
\]
where the last inequality is by 
\[
\nabla_1 \nabla_2 U(\btheta, \btheta') = \begin{bmatrix}
u(\bw, \bw') & a' \nabla_1 u(\bw, \bw')\\
a \nabla_2 u(\bw, \bw') & a a' \nabla_1 \nabla_2 u(\bw, \bw')
\end{bmatrix},
\]
and the assumption that $\vert u \vert, \| \nabla u \|_2, \| \nabla^2 u \|_\op \le K$. This gives
\[
A_+(\btheta_1, \btheta_2) \le K D (1 + \vert a_1 \vert \vee \vert a _2 \vert)^2  \| v \|_{\Ltwo}^2. 
\] 
Moreover, we have
\[
\begin{aligned}
&\E_\bx[\| \nabla_\btheta \sigma_\star(\bx; \btheta_1) - \nabla_\btheta \sigma_\star(\bx; \btheta_2) \|_2^2] = \tr[\nabla_1 \nabla_2 U(\btheta_1, \btheta_1) + \nabla_1 \nabla_2 U(\btheta_2, \btheta_2) - 2 \nabla_1 \nabla_2 U(\btheta_1, \btheta_2)] \\
\le& D \| \nabla_1 \nabla_2 U(\btheta_1, \btheta_1) + \nabla_1 \nabla_2 U(\btheta_2, \btheta_2) - 2\nabla_1 \nabla_2 U(\btheta_1, \btheta_2) \|_\op \le K \kappa D (1 + \vert a_1 \vert \vee \vert a _2 \vert)^2 \| \btheta_1 - \btheta_2 \|_2^2, 
\end{aligned}
\]
where the last inequality follows from 
\[
\| \nabla_1 \nabla_2 U(\btheta_1, \btheta_1) + \nabla_1 \nabla_2 U(\btheta_2, \btheta_2) - 2\nabla_1 \nabla_2 U(\btheta_1, \btheta_2)\|_\op \le \| \nabla_1^2 \nabla_2^2 U(\tilde \btheta_1, \tilde \btheta_2) \|_\op \| \btheta_1 - \btheta_2 \|_2^2,
\]
and $\| \nabla^3 u \|_{\op}, \| \nabla^4 u \|_{\op} \le \kappa$. This gives 
\[
\begin{aligned}
A_-(\btheta_1, \btheta_2) \le& K \kappa D (1 + \vert a_1 \vert \vee \vert a_2 \vert)^2  \| \btheta_1 - \btheta_2 \|_2^2 \| v \|_{\Ltwo}^2.
\end{aligned}
\]
Remember the notation $A(\rho) = \sup_{(a, \bw) \in \supp(\rho) }\vert a \vert$ and we have shown $A(\rho_t^\alpha) \le M_{t, \alpha} = K(1 + B^{1/2} t / \alpha)$ in step 1, we have 
\[
\begin{aligned}
&\<v,[\cH_{\rho_1}-\cH_{\rho_2}]v\>_{\Ltwo} \\
=& \Big[K D [1 + A(\rho_1) \vee A(\rho_2)]^2 \cdot \| v \|_{\Ltwo}^2 \cdot K \kappa D [1 +A(\rho_1) \vee A(\rho_2)]^2  \cdot \int \| \btheta_1 - \btheta_2 \|_2^2 \gamma(\de \btheta_1, \de \btheta_2) \| v \|_{\Ltwo}^2 \Big]^{1/2}\\
\le& K \kappa^{1/2} D [1 + A(\rho_1) \vee A(\rho_2)]^2 W_2(\rho_1, \rho_2) \cdot \| v \|_{\Ltwo}^2. 
\end{aligned}
\]
Substituting above, we get
\begin{align}
\|\cH_{\rho_0}-\cH_{\rho_t^\alpha}\|_{\sop}\le K \kappa^{1/2} D W_2(\rho_0,\rho_t^\alpha) (1+ M_{t, \alpha})^2 \le K \kappa^{1/2} D^{3/2} (1+ B^{1/2} t / \alpha)^3    B^{1/2} t /\alpha .\label{eq:OpNorm}
\end{align}

\noindent
{\bf Step 3.} Bound the difference of mean field and linearized residue dynamics $v_t = u^{\alpha}_t-u^*_t$. 
 
We now consider the mean field residual dynamics \eqref{eq:Residual} and the linearized residual dynamics \eqref{eq:Linearized}. Defining $v_t = u^{\alpha}_t-u^*_t$, we have
\begin{align}
\partial_t v_t = -\cH_{\rho_t^\alpha}v_t +(\cH_{\rho_0}-\cH_{\rho_t^\alpha})u^*_t\, .
\end{align}
Since $\cH_{\rho_t^\alpha}\succeq \bfzero$, this implies
\begin{align}
\frac{\de\phantom{t}}{\de t} \|v_t\|_{\Ltwo}^2 & \le 2\<v_t,(\cH_{\rho_0}-\cH_{\rho_t^\alpha})u^*_t\>_{\Ltwo} \le 2\|v_t\|_{\Ltwo}\|\cH_{\rho_0}-\cH_{\rho_t^\alpha}\|_{\sop}\|u^*_t\|_{\Ltwo}.
\end{align}
Using the bound \eqref{eq:OpNorm}, and $\|u^*_t\|_{\Ltwo}^2 \le \|u^*_0\|_{\Ltwo}^2=R_\alpha(\rho_0) \le B_\alpha$, we obtain 
\begin{align}
\frac{\de\phantom{t}}{\de t} \|v_t\|_{\Ltwo} \le & \|\cH_{\rho_0}-\cH_{\rho_t^\alpha}\|_{\sop}\|u^*_t\|_{\Ltwo} \le K \kappa^{1/2} D^{3/2} (1+ B^{1/2} t / \alpha)^3    B t /\alpha \, ,
\end{align}
Integrating this inequality yields Eq.~\eqref{eq:residual_difference_general}. Eq. (\ref{eq:risk_bound_general}) follows by triangle inequality.

\noindent
{\bf Step 4. } Proving Eq. (\ref{eqn:risk_goes_to_0}). 

For $\rho_0 = \rho_0^a \times \rho_0^w$ with $\vert \E (a) \vert \le K/\alpha$, we have
\[
\| \hf(\bx; \rho_0) \| = \alpha \Big\| \int a \rho_0^a (\de a ) \cdot \int \sigma (\bx ; \bw ) \rho_0^w (\de \bw) \Big\| \leq K
\]
Then we have 
\[
R_\alpha(\rho_0) = 2 \E[f(\bx)^2] + 2 \E [\hf(\bx; \rho_0)^2 ] \le K,
\]
which is independent of $\alpha$. Hence we have in both cases
\[
\lim_{\alpha \to \infty} R_\alpha(\rho_t^\alpha) \le \| u^*_t \|^2_{\Ltwo}. 
\]
Equation (\ref{eqn:risk_goes_to_0}) holds by Lemma \ref{lem:convergence_linear_residual_dynamics}. 
\end{proof}

\section{The mean field limit and kernel limit}\label{sec:mean_field_kernel_appendix}

This section is a self-contained note comparing the \emph{mean field limit} and \emph{kernel limit}. We introduce the \emph{distributional dynamics} and \emph{residual dynamics}, which we consider in the pre-limit and in the limit of infinite number of neurons. 

Let us emphasize that the material presented here is not new and appears in the literature, possibly in a slightly different formulations.

\subsection{Two layers neural networks with a scale parameter $\speed$}

Let $f: \R^d \to \R$. We use a two layer's neural network to fit this function $f$ over data $\bx \sim \P_\bx$. We denote $\hf_{\speed, N}(\bx; \btheta)$ the $N$-neurons prediction function  at point $\bx \in \R^d$
with weights $\btheta \in \R^{D \times N}$,
\[
\hf_{\speed, N}(\bx; \btheta) = \frac{\speed}{N}\sum_{j = 1}^N \sigma_\star(\bx; \btheta_j)\, .
\]
Here $\speed$ serves as a scale parameter, which can be used to explore different regimes of the learning dynamics. We minimize the population risk over $\btheta = (\btheta_1, \ldots, \btheta_N)$:
\[
R_{\speed, N}(\btheta) =  \E_{\bx} \Big[\Big( f(\bx) - \hf_{\speed, N}(\bx; \btheta) \Big)^2\Big]. 
\]

In the rest of this appendix, we will first consider the gradient flow dynamics of the finite neuron risk function. This can be described via a \emph{distributional dynamics}, which is a flow in the space of probability measures. The distributional dynamics induces an evolution of the residuals at the data points, which we call \emph{residual dynamics}. We then consider the limit  
$N \to \infty$, which we refer to as the 
\emph{mean field limit}. 

Finally, we consider the limit of both $\speed \to \infty$ \emph{after} $N\to\infty$, that we call the \emph{kernel limit}. 
Of course, it is also possible
(and interesting) to study joint limits $\alpha, N\to \infty$ \cite{jacot2018neural}. 
Our rationale for the focusing on $\alpha\to\infty$ \emph{after} $N\to\infty$ (following \cite{chizat2018note}) is that it allows to explore the crossover between mean field and kernel behaviors.

\subsection{The residual dynamics in the pre-limit}

Calculating the gradient $\nabla_{\btheta_j} R_{\speed, N}(\btheta)$ using chain rule, we get
\[
\begin{aligned}
\nabla_{\btheta_j} R_{\speed, N}(\btheta) = - \frac{\speed}{N} \hat \E_\bx [ (f(\bx) - \hf_{\speed, N}(\bx; \btheta)) \nabla_\btheta \sigma_\star(\bx; \btheta_j)]. 
\end{aligned}
\]
We consider the gradient flow ODE with time reparameterization given by $N/(2\speed^2)$, 
\[
\frac{\de \btheta_j^t}{\de t} = - \frac{N}{2\speed^2}  \nabla_{\btheta_j} R_{\speed, N}(\btheta^t) = \frac{1}{\speed} \E_\bx [ (f(\bx) - \hf_{\speed, N}(\bx; \btheta^t)) \nabla_\btheta \sigma_\star(\bx; \btheta_j^t)].
\]
The time derivative of $\hf_{\speed, N}(\bz; \btheta^t)$ can be calculated using the chain rule. We have
\[
\begin{aligned}
\partial_t \hf_{\speed, N}(\bz; \btheta^t) =& \frac{\speed}{N}\sum_{j = 1}^N \< \nabla_\btheta \sigma_\star(\bz; \btheta_j^t), \frac{\de \btheta_j^t}{\de t} \>\\
=&   \E_\bx \Big[ \Big( \frac{1}{N}\sum_{j=1}^N \< \nabla_\btheta \sigma_\star( \bx; \btheta_j^t), \nabla_\btheta \sigma_\star(\bz; \btheta_j^t)\>  \Big) \Big(f(\bx) - \hf_{\speed, N}(\bx; \btheta^t)\Big) \Big].
\end{aligned}
\]

Define the kernel function $\cH(\bx, \bz; \btheta)$ with weights $\btheta \in \R^{D \times N}$ to be 
\[
\cH(\bx, \bz; \btheta) = \frac{1}{N}\sum_{j=1}^N \<\nabla_\btheta \sigma_\star(\bx; \btheta_j), \nabla_\btheta \sigma_\star(\bz; \btheta_j) \> , 
\]
then we have
\[
\begin{aligned}
\partial_t \hf_{\speed, N}(\bz; \btheta^t) =&  \E_\bx \Big[\Big(f(\bx) - \hf_{\speed, N}(\bx; \btheta^t) \Big) \cH(\bx, \bz; \btheta^t)  \Big]. 
\end{aligned}
\]
Taking the residue function to be $u_t^{\speed, N}(\bz) = f(\bz) - \hf_{\speed, N}(\bz; \btheta^t)$, we have
\begin{equation}\label{eq:appendix_RD}
\begin{aligned}
\partial_t u_t^{\speed, N}(\bz) =& - \E_\bx[ \cH(\bx, \bz; \btheta^t) u_t^{\speed, N}(\bx) ],\\
\end{aligned}
\end{equation}
with initialization $u_0^{\speed, N}(\bz) = f(\bz) - f_{\speed, N}(\bz; \btheta^0)$ and $\btheta_i^0 \sim \rho_0$ independently. We call Eq. (\ref{eq:appendix_RD}) the \emph{residual dynamics}. The residual dynamics is not a self-contained equation and depends on $\btheta^t$. 

\subsection{The distributional dynamics in the pre-limit}

Define 
\[
\begin{aligned}
\rho_t^{\speed, N}(\de \btheta) =& \frac{1}{N} \sum_{j=1}^N \delta_{\btheta_j^t}.
\end{aligned}
\]
Define the prediction function with distribution $\rho$ and scaled parameter $\speed$ to be
\[
\hf_\speed(\bx; \rho) = \speed \int \sigma_\star(\bx; \btheta) \rho(\de \btheta). 
\]

Consider again the gradient flow dynamics
\[
\frac{\de \btheta_j^t}{\de t} = \frac{1}{\speed} \E_\bx [ (f(\bx) - \hf_{\speed, N}(\bx; \btheta^t)) \nabla_\btheta \sigma_\star(\bx; \btheta_j)] = - \frac{1}{\speed} \nabla_\btheta \Psi_\speed(\btheta_j^t; \rho_t^{\speed, N}).
\]
where we defined
\[
\Psi_\speed(\btheta; \rho) = - \E_\bx [ ( f(\bx) - \hf_\speed(\bx; \rho) ) \sigma_\star(\bx; \btheta) ]. 
\]
Then we have 
\begin{equation}\label{eq:appendix_DD}
\begin{aligned}
\partial_t \rho_t^{\speed, N} =&(1/\speed) \nabla_\btheta \cdot ( \rho_t^{\speed, N} [\nabla_\btheta \Psi(\btheta; \rho_t^{\speed, N})] ),\\
\rho_0^{\speed, N} =& \frac{1}{N} \sum_{j=1}^N \delta_{\btheta_i^0},\\
\end{aligned}
\end{equation}
with $\btheta_i^0 \sim \rho_0$ independently. We call dynamics (\ref{eq:appendix_DD}) the \emph{distributional dynamics}. The distributional dynamics is equivalent to the gradient flow. 

\subsection{The coupled dynamics}

Writing the \emph{distributional dynamics} and \emph{residual dynamics} together (in the pre-limit), we have
\[
\begin{aligned}
\partial_t \rho_t^{\speed, N} =& (1/\speed) \nabla_\btheta \cdot ( \rho_t^{\speed, N} [\nabla_\btheta \Psi_\speed(\btheta; \rho_t^{\speed, N})] ),\\
\partial_t u_t^{\speed, N}(\bz) =& - \E_\bx[u_t^{\speed, N}(\bx) \cH_{\rho_t^{\speed, N}}(\bx, \bz)],\\
\end{aligned}
\]
where
\[
\begin{aligned}
\cH_\rho(\bx, \bz) \equiv &  \int  \< \nabla_\btheta \sigma_\star(\bx; \btheta), \nabla_\btheta \sigma_\star(\bz; \btheta)\> \rho(\de \btheta) , \\
\Psi_\speed(\btheta; \rho^{\speed, N}) = & - \E_\bx [ ( f(\bx) - \hf_\speed(\bx; \rho^{\speed, N}) ) \sigma_\star(\bx; \btheta) ] = - \E_\bx[u_t^{\speed, N}(\bx) \sigma_\star(\bx; \btheta)],
\end{aligned}
\]
with initialization conditions $\rho_0^N = (1/N) \sum_{i=1}^N \delta_{\btheta_i^0}$, $u_N(0, \bx) = f(\bx) - \hf_{\speed, N}(\bx; \btheta^0)$, and $(\btheta_i^0)_{i\le N} \sim_{i.i.d.} \rho_0$. 

Note these coupled dynamics are random, where the randomness comes from the random initialization
$(\btheta_i^0)_{i\le N} \sim_{i.i.d.} \rho_0$.

\subsection{The mean field limit}

In the mean field limit, we fix $\speed$ and take $N \to \infty$. Under some conditions, it can be shown that there exists $(\rho_t)_{t \ge 0}$ satisfying the \emph{mean field distributional dynamics}
\begin{equation}\label{eq:appendix_MF_DD}
\begin{aligned}
\partial_t \rho_t^\speed =&(1/\speed) \nabla \cdot ( \rho_t^\speed [\nabla_\btheta \Psi_\speed(\btheta; \rho_t^\speed)] ),
\end{aligned}
\end{equation}
with initialization condition $\rho_0^\speed = \rho_0$. Moreover, we have almost surely (over $\btheta_i^0 \sim \rho_0$ independently)
\[
\lim_{N\to \infty} W_2(\rho_t^{\speed, N}, \rho_t^\speed) \to 0. 
\]
The mean field distributional dynamics was proposed and studied in \cite{mei2018mean,sirignano2018mean,rotskoff2018neural,chizat2018global} under various conditions.

Now define the \emph{mean field residual function} $u_t^\speed(\bz)$ to be
\[
u_t^\speed(\bz) \equiv f(\bz) - \hf_{\speed}(\bz; \rho_t^\speed). 
\]
For any fixed $\bz$, we have almost surely
\[
\lim_{N \to \infty} u_t^{\speed, N}(\bz) = u_t^\speed(\bz). 
\]
Under some regularity conditions, it is not hard to show that this mean field residual function satisfies \emph{mean field residual dynamics}
\[
\partial_t u_t^\speed(\bz) = - \E_\bx[u_t^\speed(\bx) \cH_{ \rho_t^\speed}(\bx, \bz)]. 
\]
The \emph{mean field residual dynamics} is not a self-contained equation. It depends on the distribution through the kernel $\cH_{\rho_t^{\speed}}$. The mean field residual dynamics was first explicitly given in \cite[Proposition 2.5]{rotskoff2018neural}. 

\subsection{The kernel limit}

Theorem \ref{thm:Kernel} shows that, as $\speed$ becomes large, for any fixed $t$, we have
\[
\lim_{\speed \to \infty} W_2(\rho_t^\speed, \rho_0) = 0,
\]
and hence 
\[
\lim_{\speed \to \infty} \| \cH_{\rho_t^\speed} - \cH_{\rho_0} \|_{\op} = 0.
\]
In this limit, the mean field residual dynamics converges to the \emph{linearized residual dynamics}, 
\begin{equation}\label{eq:appendix_LRD}
\partial_t u_t^*(\bz) = - \E_\bx[u_t^*(\bx) \cH_{\rho_0}(\bx, \bz)]. 
\end{equation}
The linearized residual dynamics is exactly the same as the continuous time kernel boosting dynamics with kernel $\cH_{\rho_0}$, whose solution can be written down explicitly
\begin{equation}\label{eq:appendix_RF_solution}
u_t^* = e^{- \cH_{\rho_0} t} u_0^*. 
\end{equation}
When the kernel is strictly positive definite, one can show that the $\Ltwo$-norm of the residual function converges to $0$ as time goes to infinity. 

The kernel limit is studied in \cite{jacot2018neural, geiger2019scaling} in the joint limit $\alpha = N^{1/2} \to \infty$, and in a multi-layer neural network settings. The specific limit considered here ($N \to \infty$ followed by $\alpha \to \infty$) is discussed in \cite{chizat2018note}. 

An interesting line of research \cite{li2018learning, du2018gradient, du2018gradient2, allen2018convergence} also studies the kernel limit, but focusing on dynamics on empirical risk. Note that all the equations discussed above also holds for $\P_\bx = (1/n) \sum_{k=1}^n \delta_{\bx_k}$. The benefit of working with the empirical risk  is that, under mild assumptions, the kernel matrix $\{ \cH_{\rho_0}(\bx_i, \bx_j)\}_{i, j \in [n]}$ is strictly positive definite with least eigenvalue $\lambda_{\min}>0$. As a result, it is possible to upper bound the convergence time of the empirical risk to $0$ using Eq. (\ref{eq:appendix_LRD}). Hence, it is possible to choose the number of neurons large enough
that the residual dynamics (\ref{eq:appendix_RD}) is well approximated by the linearized residual dynamics (\ref{eq:appendix_LRD}) along the whole trajectory.

\subsection{Kernel limit as kernel ridge regression}
\label{sec:KernelRidgApp}

Consider the case when $\P_\bx = (1/n) \sum_{k=1}^n \delta_{\bx_k}$ is the empirical data distribution. We make an additional assumption on the initialization weight distribution $\rho_0$:
\begin{itemize}
\item[(I)] The initialization distribution $(a,\bw) \sim \rho_0$ verifies: $a$ is independent of $\bw$ and $\E(a) = 0$. In other words, $\rho_0 = \rho_0^a \times \rho_0^{\bw}$ with $\int a \rho_0^a(\de a) = 0$. 
\end{itemize}  
Under this assumption, we have $\hf_\alpha(\bz; \rho_t^\alpha) \equiv 0$ for any $\bz \in \R^d$, so that $u_0^\alpha(\bx_k) = f(\bx_k)$ for $k \in [n]$. 

Denote 
\[
\begin{aligned}
\bu_t^\alpha =& [u_t^\alpha(\bx_1), \ldots, u_t^\alpha(\bx_n)]^\sT,\\
\bu_t^* =& [u_t^*(\bx_1), \ldots, u_t^*(\bx_n)]^\sT,\\
\by =& [f(\bx_1), \ldots, f(\bx_n)]^\sT.
\end{aligned}
\]
Further we denote the data kernel matrix $\bH \in \R^{n \times n}$ with $\bH_{ij} =  \cH_{\rho_0}(\bx_i, \bx_j)$. Then Eq. (\ref{eq:Linearized}) can be rewritten as
\[
\bu_t^* = e^{- \bH t / n} \bu_0^* = e^{- \bH t / n} \by.
\]
Note Theorem \ref{thm:Kernel} holds also in the case when $\P_\bx$ is an empirical data distribution. Hence we have 
\[
\lim_{\alpha \to \infty} \sup_{t \in [0, T]} \frac{1}{\sqrt{n}}\| \bu_t^\alpha - \bu_t^* \|_2 = \lim_{\alpha \to \infty} \sup_{t \in [0, T]} \| u_t^\alpha - u_t^* \|_{\Ltwo} = 0.
\]

The following proposition considers the scaling limit (kernel limit) of the prediction function at time $t$, 
\[
\hf_\alpha(\bz; \rho_t^\alpha) = \alpha \int \sigma_\star(\bx; \btheta) \rho_t^\alpha(\de \btheta),
\]
where $\rho_t^\alpha$ is the solution of the rescaled distributional dynamics (\ref{eq:RescaledPDE}). 

This fact already appears (implicitly or explicitly) in several of the papers mentioned above. We state and prove it here for the sake of completeness. 
\begin{proposition}\label{thm:kernel_ridge_regression}
Assume conditions A1 - A4 hold, and $\P_\bx = (1/n) \sum_{k=1}^n \delta_{\bx_k}$ to be the empirical data distribution. Additionally assume the finite data kernel matrix $\bH \in \R^{n \times n}$ is invertible, and $\rho_0$ verifies property {\rm (I)}. Then for any fixed $\bz \in \R^d$, we have 
\[
\lim_{t \to \infty} \lim_{\alpha \to \infty} \hf_\alpha(\bz; \rho_t^\alpha) = \bh(\bz)^\sT \bH^{-1} \by,
\]
where 
\[
\begin{aligned}
\bh(\bz) =& [\cH_{\rho_0}(\bz, \bx_1), \ldots, \cH_{\rho_0}(\bz, \bx_n)]^\sT.
\end{aligned}
\] 
\end{proposition}

\begin{remark}\label{rmk:kernel_ridge_regression}
Given a data set $\{ (\bx_i, y_i)\}_{i \in [n]}$, kernel ridge regression is a function estimator $\hf_\lambda$ that solves the following minimization problem
\[
\begin{aligned}
\min_{f}& ~~ \frac{1}{n}\sum_{i=1}^n ( y_i - f(\bx_i))^2 + \lambda \| f \|_{\cH_{\rho_0}}.\\
\end{aligned}
\]
The norm $\| f \|_{\cH_{\rho_0}}$ is the \emph{reproducible kernel Hilbert space} (RKHS) norm of function $f$, where the RKHS is associated to the kernel $\cH_{\rho_0}$. The solution of the minimization problem above gives
\[
\hf_\lambda(\bz) = \bh(\bz)^\sT (\bH + \lambda \id)^{-1} \by. 
\]
Proposition \ref{thm:kernel_ridge_regression} shows that, the mean field prediction function in the kernel limit is performing a kernel ridge regression with regularization parameter $\lambda = 0$. 
\end{remark}

\begin{proof}[Proof of Proposition \ref{thm:kernel_ridge_regression}]~Recall that
\[
\begin{aligned}
\bu_t^\alpha =& [u_t^\alpha(\bx_1), \ldots, u_t^\alpha(\bx_n)]^\sT,\\
\bu_t^* =& [u_t^*(\bx_1), \ldots, u_t^*(\bx_n)]^\sT,\\
\by =& [f(\bx_1), \ldots, f(\bx_n)]^\sT.
\end{aligned}
\]
The data kernel matrix $\bH \in \R^{n \times n}$ is given by $\bH_{in} = \cH_{\rho_0}(\bx_i, \bx_j) $. By Eq. (\ref{eq:Linearized}) and the assumption on $\rho_0$, we have
\[
\bu_t^* = e^{- \bH t / n} \bu_0^* = e^{- \bH t / n} \by.
\]

For any fixed $\bz \in \R^d$, denote 
\[
\begin{aligned}
\bh_t^\alpha(\bz) =& [\cH_{\rho_t^\alpha}(\bz, \bx_1), \ldots, \cH_{\rho_t^\alpha}(\bz, \bx_n)]^\sT,\\
\bh(\bz) =& [\cH_{\rho_0}(\bz, \bx_1), \ldots, \cH_{\rho_0}(\bz, \bx_n)]^\sT.
\end{aligned}
\] 
Using chain rule, the time derivative of the prediction function $\hf_\alpha(\bz; \rho_t^\alpha) = \alpha \int \sigma_\star(\bx; \btheta) \rho_t^\alpha(\de \btheta)$ gives
\begin{equation}\label{eqn:thm_kernel_ridge_proof_1}
\begin{aligned}
\partial_t \hf_\alpha(\bz; \rho_t^\alpha) =& \alpha \partial_t \int \sigma_\star(\bz; \btheta) \rho_t^\alpha(\de \btheta) = \int \< \nabla_\btheta \sigma_\star(\bz; \btheta), \nabla_\btheta \Psi_\alpha(\btheta; \rho_t^\alpha) \> \rho_t^\alpha(\de \btheta) \\
=& \E_\bx\Big[u_t^\alpha(\bx) \int \< \nabla_\btheta \sigma_\star(\bz; \btheta), \nabla_\btheta \sigma_\star(\bx; \btheta) \> \rho_t^\alpha(\de \btheta) \Big] = \bh_t^\alpha(\bz) \bu_t^\alpha / n. 
\end{aligned}
\end{equation}
By the same argument as Step 2 of Theorem \ref{thm:Kernel}, we have
\begin{equation}\label{eqn:thm_kernel_ridge_proof_2}
\sup_ {t \in [0, T]}\| \bh(\bz) - \bh_t^\alpha(\bz) \|_2 = O(1/\alpha).
\end{equation}
By Theorem \ref{thm:Kernel}, we have 
\begin{equation}\label{eqn:thm_kernel_ridge_proof_3}
\sup_{t \in [0, T]} \| \bu_t^\alpha - \bu_t^* \|_2 = \sup_{t \in [0, T]} \| u_t^\alpha - u_t^* \|_{\Ltwo} = O(1/ \alpha).
\end{equation}

Now, we denote $\hf_t(\bz)$ be the solution of the following \emph{linearized prediction dynamics}, 
\begin{equation}\label{eqn:linear_prediction_dynamics}
\begin{aligned}
\partial_t \hf_t (\bz) =& \bh(\bz)^\sT \bu_t^* / n, \\
\hf_0(\bz) =& 0.
\end{aligned}
\end{equation}
By Eq (\ref{eqn:thm_kernel_ridge_proof_1}), (\ref{eqn:thm_kernel_ridge_proof_2}), (\ref{eqn:thm_kernel_ridge_proof_3}) and (\ref{eqn:linear_prediction_dynamics}), we have
\[
\sup_{t \in [0, T]} \vert \partial_t \hf_t(\bz) - \partial_t \hf_\alpha(\bz; \rho_t^\alpha) \vert = O(1/ \alpha), 
\]
together with $\hf_0(\bz) = \hf_\alpha(\bz; \rho_0^\alpha) = 0$ we get 
\[
\hf_t(\bz) = \lim_{\alpha \to \infty} \hf_\alpha(\bz; \rho_t^\alpha).
\]
 
Note the solution of Eq. (\ref{eqn:linear_prediction_dynamics}) gives 
\[
\begin{aligned}
\hf_t(\bz) =&n^{-1} \int_0^t \bh(\bz)^\sT \bu_s^* \de s = n^{-1}\int_0^t \bh(\bz)^\sT e^{- \bH s / n} \by \de s = \bh(\bz)^\sT \bH^{-1} (\id - e^{- \bH t / n}) \by ,
\end{aligned}
\]
so that 
\[
\hf_\infty(\bz) = \lim_{t \to \infty} \hf_t (\bz) = \lim_{t \to \infty } \bh(\bz)^\sT \bH^{-1} (\id - e^{- \bH t / n}) \by = \bh(\bz)^\sT \bH^{-1} \by. 
\]
This proves the proposition. 
%
%
%
\end{proof}

%
%
%
%
%
%
%

\section{Technical lemmas}

\begin{lemma}\label{lem:bounded_difference_martingale}
Let $\bX_i \in \R^D$ with $\{ \bX_i \}_{i \in [N]}$ to be i.i.d. random variables, with $\| \bX_i \|_2 \le K$ and $\E[\bX_i] = \bzero$. Then we have (the constant $K$ in the result is up to some universal constant)
\[
\P\Big(  \Big \| \frac{1}{N} \sum_{i=1}^N \bX_i \Big \|_2 \ge K (\sqrt {1/N} + \delta) \Big) \le e^{- N \delta^2}. 
\]
\end{lemma}

\begin{proof}

Denote $f(\bX_1, \ldots, \bX_N) = \| (1/N)\sum_{i=1}^N \bX_i \|_2$. Then we have 
\[
\begin{aligned}
\vert \E[f(\bX_1, \ldots, \bX_N)] \vert \le& \E[f(\bX_1, \ldots, \bX_N)^2]^{1/2} = \E \Big[\Big\< \frac{1}{N} \sum_{i=1}^N \bX_i, \frac{1}{N} \sum_{j=1}^N \bX_j \Big\> \Big]^{1/2}\\
=&\Big\{\frac{1}{N^2} \sum_{i=1}^N \E[\| \bX_i \|_2^2] \Big\}^{1/2} \le K \sqrt \frac{1}{N}. 
\end{aligned}
\]
Note by triangle inequality, we have 
\[
\vert f(\bX_1, \ldots, \bX_i, \ldots, \bX_N) - f(\bX_1, \ldots, \bX_i', \ldots, \bX_N) \vert \le \frac{1}{N} \| \bX_i - \bX_i' \|_2 \le \frac{2 K}{N}. 
\]
By McDiarmid's inequality, we have 
\[
\P\Big( \vert f(\bX_1, \ldots, \bX_N) - \E[f(\bX_1, \ldots, \bX_N)] \vert \ge \delta  \Big) \le \exp\{ - N \delta^2 / K \},
\]
which gives the desired result. 
\end{proof}

\begin{lemma}[Azuma-Hoeffding bound]\label{lem:Azuma}
Let $(\bX_k)_{k \geq 0} $ be a martingale taking values in $\R^D$ with respect to the filtration $(\cF_k)_{k\geq0}$, with $\bX_0 = 0$. Assume that the following holds almost surely for all $k \geq 1$:
\[
\E \lbrace e^{\< \lambda , \bX_k - \bX_{k-1} \> } \vert \cF_{k-1} \rbrace \leq e^{L^2 \norm{\lambda}^2_2/2}
\]
Then we have
\[
\P\Big( \max_{k\leq n} \norm{\bX_k }_2 \geq 2L\sqrt{n} \lb \sqrt{D} + \delta \rb \Big) \le e^{- \delta^2}. 
\]
\end{lemma}

\begin{proof}
This lemma is proven in \cite[Section A, Lemma A.1]{mei2018mean}.
\end{proof}

\end{document}